\documentclass{article}

\usepackage{graphicx}
\usepackage{enumitem}
\usepackage[table]{xcolor}
\usepackage{wrapfig}
\usepackage{amsmath}
\usepackage{subfigure}
\usepackage{caption}
\usepackage{titletoc}

\usepackage[final]{neurips_2025}

\usepackage{hyperref}
\usepackage{tcolorbox}
\usepackage{tabularx}
\usepackage{pifont}
\usepackage{algorithm}
\usepackage{algorithmic}
\colorlet{ours}{blue!10}

\usepackage[utf8]{inputenc} 
\usepackage[T1]{fontenc}    
\usepackage{hyperref}       
\usepackage{url}            
\usepackage{booktabs}       
\usepackage{amsfonts}       
\usepackage{nicefrac}       
\usepackage{microtype}      
\usepackage{xcolor}         
\usepackage{amssymb}
\usepackage{mathtools}
\usepackage{amsthm}
\usepackage{multirow}
\usepackage{amsmath}
\DeclareMathOperator*{\argmin}{arg\,min}
\newcommand{\com}[1]{\tiny$\pm$#1}

\newcommand{\appendixtoc}{%
    \startcontents[appendix]
    \printcontents[appendix]{l}{1}{\setcounter{tocdepth}{2}}
}

\theoremstyle{plain}
\newtheorem{theorem}{Theorem}[section]
\newtheorem{proposition}[theorem]{Proposition}
\newtheorem{lemma}[theorem]{Lemma}
\newtheorem{corollary}[theorem]{Corollary}
\theoremstyle{definition}
\newtheorem{definition}[theorem]{Definition}
\newtheorem{assumption}[theorem]{Assumption}
\theoremstyle{remark}
\newtheorem{remark}[theorem]{Remark}
\title{Robust Federated Finetuning of LLMs via Alternating Optimization of LoRA}

\author{%
  Shuangyi Chen$^{1}$\thanks{Equal contribution.}\\
  University of Toronto\\
  {\small\texttt{shuangyi.chen@mail.utoronto.ca}} \\
  \And
  Yuanxin Guo$^{1}$\footnotemark[1] \\
  University of Toronto \\
  {\small\texttt{yuanxin.guo@mail.utoronto.ca}}\\
  \AND
  Yue Ju \\
  Ericsson-GAIA Montréal \\
  {\small\texttt{yue.ju@ericsson.com}} \\
  \And
  Hardik Dalal  \\
  Ericsson-GAIA Montréal \\
  {\small\texttt{hardik.dalal@ericsson.com}} \\
  \And
  Zhongwen Zhu \\
  Ericsson-GAIA Montréal \\
  {\small\texttt{zhongwen.zhu@ericsson.com}} \\
  \And
  Ashish Khisti \\
  University of Toronto \\
  {\small\texttt{akhisti@ece.utoronto.ca}} \\
}

\begin{document}

\maketitle

\begin{abstract}
  Parameter-Efficient Fine-Tuning (PEFT) methods like Low-Rank Adaptation (LoRA) optimize federated training by reducing computational and communication costs.  We propose RoLoRA, a federated framework using alternating optimization to fine-tune LoRA adapters. Our approach emphasizes the importance of learning up and down projection matrices to enhance expressiveness and robustness. We use both theoretical analysis and extensive experiments to demonstrate the advantages of RoLoRA over prior approaches that either generate imperfect model updates or limit expressiveness of the model. We provide a theoretical analysis on a linear model to highlight the importance of learning both the down-projection and up-projection matrices in LoRA. We validate the insights on a non-linear model and separately provide a convergence proof under general conditions. To bridge theory and practice, we conducted extensive experimental evaluations on language models including RoBERTa-Large, Llama-2-7B on diverse tasks and FL settings to demonstrate the advantages of RoLoRA over other methods.
\end{abstract}

\section{Introduction}\label{intro}
The remarkable performance of large language models (LLMs) stems from their ability to learn at scale. With their broad adaptability and extensive scope, LLMs depend on vast and diverse datasets to effectively generalize across a wide range of tasks and domains. Federated learning \cite{mcmahan2017communication} offers a promising solution for leveraging data from multiple sources, which could be  particularly advantageous for LLMs.

Recently, Parameter-Efficient Fine-Tuning (PEFT) has emerged as an innovative training strategy that updates only a small subset of model parameters, substantially reducing computational and memory demands. A notable method in this category is LoRA \cite{hu2021lora}, which utilizes low-rank matrices to approximate weight changes during fine-tuning. These matrices are integrated with pre-trained weights for inference, facilitating reduced data transfer in scenarios such as federated learning, where update size directly impacts communication efficiency. Many works integrate LoRA into federated setting \cite{zhang-etal-2023-fedpetuning,babakniya2023slora, kuang2023federatedscopellm, chen2024robust, sun2024improving}.  FedPETuning \cite{zhang-etal-2023-fedpetuning} compares various PEFT methods in a federated setting. SLoRA   \cite{babakniya2023slora} presents a hybrid approach that combines sparse fine-tuning with LoRA to address data heterogeneity in federated settings. Furthermore, FS-LLM \cite{kuang2023federatedscopellm} presents  a framework for fine-tuning LLMs in federated environments. However, these studies typically apply the FedAVG algorithm directly to LoRA modules, resulting in in-exact model updates: the server average of LoRA-$A$ and LoRA-$B$ does not equal the effective adapter.

\begin{figure}[tp] 
    \centering
    \includegraphics[width=0.95\textwidth]{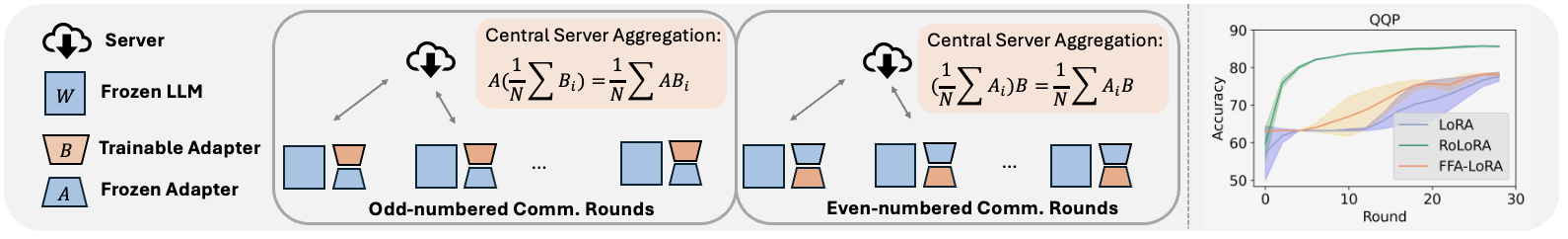} 
    \caption{(Left) Overview of the RoLoRA framework. (Right) Performance comparison with baselines on QQP in a 50-client setting, showing RoLoRA’s superior convergence speed and final accuracy.}
    \label{fig:overview} 
\end{figure}
To address the issue of in-exact model updates, a few recent works have proposed modifications to the down-projection and up-projection components in LoRA. In FlexLoRA \cite{bai2024federatedfinetuninglargelanguage}, the authors propose updating these projections with matrix multiplication followed by truncated SVD. A related method is also considered in FLoRA \cite{wang2024florafederatedfinetuninglarge}. Another approach, by Sun et al., introduces a federated finetuning framework named FFA-LoRA \cite{sun2024improving}, which builds on LoRA by freezing the down-projection matrices across all clients and updating only the up-projection matrices. They apply differential privacy \cite{dwork2006calibrating} to provide privacy guarantees for clients' data. With a sufficient number of finetuning parameters, FFA-LoRA, using a larger learning rate, can achieve performance comparable to FedAVG of LoRA while reducing communication costs by half. However, we observe that with fewer finetuning parameters, FFA-LoRA is less robust than FedAVG of LoRA, primarily due to its reduced expressiveness from freezing down-projections. In this work, we explore the necessity of learning down-projection matrices and propose a federated fine-tuning framework with computational and communication advantages.

We connect the objective of learning down-projection matrices in a federated setting to multitask linear representation learning (MLRL), an approach in which a shared low-rank representation is jointly learned across multiple tasks. While, to the best of our knowledge, the alternating optimization of down- and up-projection matrices has not been explored within the context of LoRA, prior works on MLRL \cite{pmlr-v139-collins21a, thekumparampil2021sampleefficientlinearmetalearning} have demonstrated the importance of alternately updating low-rank representations and task-specific heads, demonstrating the necessity of learning a shared representation. Inspired by MLRL, we tackle this challenge by employing alternating optimization for LoRA adapters. We theoretically establish that alternating updates to the two components of LoRA, while maintaining a common global model, enable effective optimization of down-projections and ensure convergence to the global minimizer in a tractable setting.
\vspace{-0.3cm}
\subsection{Main Contributions}
\begin{itemize}[noitemsep]
    \item \textbf{RoLoRA framework.} We propose RoLoRA, a robust federated fine-tuning framework based on the alternating optimization of LoRA as shown in Figure~\ref{fig:overview}. RoLoRA fully leverages the expressiveness of LoRA adapters while keeping the computational and communication advantages. 
    \item \textbf{Theoretical Insights.} We prove that RoLoRA achieves exponential convergence to the global optimum for federated linear regression, reaching arbitrarily small errors, whereas FFA-LoRA’s fixed down-projections result in suboptimality proportional to the initialization error. Simple non-linear model experiments further validate the theoretical importance of updating the down-projections in practice. For the case of smooth non-convex loss functions we also provide a proof of convergence of our proposed method.
    \item \textbf{Empirical results.} Through evaluations on language models (RoBERTa-Large, Llama-2-7B) across various tasks (GLUE, HumanEval, MMLU, Commonsense reasoning tasks), we demonstrate that RoLoRA maintains robustness against reductions in fine-tuning parameters and increases in client numbers compared to prior approaches.
\end{itemize}
\vspace{-0.3cm}
{
In summary, RoLoRA outperforms baselines with strong theoretical guarantees and robust performance across tasks and scales. It provides new insights into the role of down-projections, advancing the frontier of parameter-efficient federated learning, while its simple yet effective design ensures scalability and communication efficiency.}

\vspace{-0.3cm}
\subsection{Notations}
\label{subsec:notation}
\vspace{-0.3cm}
We adopt the notation that lower-case letters represent scalar variables, lower-case bold-face letters denote column vectors, and upper-case bold-face letters denote matrices. The $d \times d$ identity matrix is represented by $\mathbf{I}_d$. Depending on the context, $\lVert .\rVert$ denotes the $l_2$ norm of a vector or the Frobenius norm of a matrix, $\lVert .\rVert_{op}$ denotes the operator norm of a matrix, $|. |$ denotes the absolute value of a scalar, ${}^\top$ denotes matrix or vector transpose. For a number $N$, $[N] = \{1,\dots, N\}$.

\vspace{-0.3cm}
\section{Preliminaries and Related Works}
\vspace{-0.3cm}
\subsection{LoRA and Its Variants} \label{lora-federated}
\vspace{-0.3cm}
Low-Rank Adaptation (LoRA)~\cite{hu2021lora} fine-tunes large language models efficiently by keeping the original weights fixed and adding small trainable matrices that apply low-rank updates. Specifically, for a pre-trained weight matrix $\mathbf{W}_0 \in \mathbb{R}^{d\times d}$, the update is expressed as $\mathbf{W} = \mathbf{W}_0 + \alpha \mathbf{A}\mathbf{B}$, where $\mathbf{A} \in \mathbb{R}^{d \times r}$ and $\mathbf{B} \in \mathbb{R}^{r \times d}$ with $r \ll d$, and only $\mathbf{A}$ and $\mathbf{B}$ are trained. Many variants of LoRA have been developed to further improve efficiency and performance. In Zhang et al. \cite{zhang2023lorafa}, the authors propose a memory-efficient fine-tuning method, LoRA-FA, which keeps the projection-down weight fixed and updates the projection-up weight during fine-tuning. In Zhu et al. \cite{zhu2024asymmetrylowrankadaptersfoundation}, the authors highlight the asymmetry between the projection-up and projection-down matrices and focus solely on comparing the effects of freezing either the projection-up or projection-down matrices. Hao et al. \cite{hao2024floralowrankadapterssecretly} introduce the idea of resampling the projection-down matrices, aligning with our observation that freezing projection-down matrices negatively impacts a model's expressiveness. Furthermore, LoRA+ \cite{hayou2024lora} explore the distinct roles of projection-up and projection-down matrices, enhancing performance by assigning different learning rates to each. 
\vspace{-0.3cm}
\subsection{LoRA in Federated Setting}
\vspace{-0.3cm}
In federated settings, LoRA is practical: clients fine-tune models efficiently with minimal resources, and only the small adapter matrices need to be communicated, sharply reducing transmission costs compared to full model finetuning. Zhang et al. consider FedAVG of LoRA, named FedIT\cite{zhang2023towards}. Zhang et al. \cite{zhang-etal-2023-fedpetuning} compare multiple PEFT methods in the federated setting, including Adapter\cite{houlsby2019parameterefficient}, LoRA\cite{hu2021lora}, Prompt tuning\cite{liu2022ptuning} and Bit-Fit\cite{zaken2022bitfit}. SLoRA\cite{babakniya2023slora}, which combines sparse finetuning and LoRA, is proposed to address the data heterogeneity in federated setting. 
As discussed before, Sun et al. \cite{sun2024improving} design a federated finetuning framework FFA-LoRA by freezing projection-down matrices for all the clients and only updating projection-up matrices. 
FlexLoRA\cite{bai2024federatedfinetuninglargelanguage} and FLoRA \cite{wang2024florafederatedfinetuninglarge} consider clients with heterogeneous-rank LoRA adapters and proposes federated finetuning approaches. After we completed our work, we noticed a concurrent study, LoRA-$A^2$ \cite{koo2024robustefficientfederatedlowrank}, which combines alternating optimization with adaptive rank selection for federated finetuning. While their focus is mainly empirical, the use of alternating optimization in their algorithm is similar to ours.

\vspace{-0.3cm}

\subsection{FedAVG of LoRA Introduces Interference}\label{interference}
\vspace{-0.3cm}
Integrating LoRA within a federated setting presents challenges. In such a setup, each of the $N$ clients is provided with the pretrained model weights $\mathbf{W}_0$, which remain fixed during finetuning. Clients are required only to send the updated matrices $\mathbf{B}_i$ and $\mathbf{A}_i$ to a central server for aggregation. While most current studies, such as FedIT\cite{zhang2023towards}, SLoRA \cite{babakniya2023slora} and FedPETuning \cite{zhang-etal-2023-fedpetuning}, commonly apply FedAVG directly to these matrices as shown in \eqref{fedavg-lora}, this approach might not be optimal. The precise update for each client’s model, $\Delta \mathbf{W}_i$, should be calculated as the product of the low-rank matrices $\mathbf{A}_i$ and $\mathbf{B}_i$. Consequently, aggregation on the individual matrices leads to inaccurate model aggregation. 

\begin{align}
    \frac{1}{N} \sum_{i=1}^N  \Delta \mathbf{W}_i &= \frac{1}{N} (\mathbf{A}_1\mathbf{B}_1+\mathbf{A}_2\mathbf{B}_2+...+ \mathbf{A}_N\mathbf{B}_N) \label{fedavg-mul}\\
    &\neq  \frac{1}{N} (\mathbf{A_1}+\mathbf{A_2}+...+\mathbf{A_N}) \frac{1}{N} (\mathbf{B_1}+\mathbf{B_2}+...+\mathbf{B_N}) \label{fedavg-lora}
\end{align}
There are a few options to avoid it.

    \paragraph{Updating B and A by matrix multiplication and truncated-SVD.} One approach \cite{wang2024florafederatedfinetuninglarge, bai2024federatedfinetuninglargelanguage} involves first computing the product of local matrices $\mathbf{B}_i$ and $\mathbf{A}_i$ to accurately recover $\Delta \mathbf{W}_i$. Then, the global $\mathbf{B}$ and $\mathbf{A}$ of next iteration are obtained by performing truncated SVD on the averaged set of ${\Delta \mathbf{W}_i}$. However, this method introduces computational overhead due to the matrix multiplication and SVD operations.
    \vspace{-0.2cm}
    \paragraph{Freezing A (B) during finetuning.} Another method is to make clients freeze $\mathbf{B}$ or $\mathbf{A}$ as in Sun et al. \cite{sun2024improving}, leading to precise computation of ${\Delta \mathbf{W}}$. However, this method limits the expressiveness of the adapter.

With these considerations, we propose a federated finetuning framework, named RoLoRA, based on alternating optimization of LoRA.
\vspace{-0.3cm}
\section{RoLoRA Framework} 
\label{rolora-framework}
\vspace{-0.3cm}
In this section, we describe the framework design of RoLoRA and discuss its practical advantages. 

\paragraph{Alternating Optimization and Corresponding Aggregation}
Motivated by the observations discussed in Section~\ref{interference}, we propose applying alternating optimization to the local LoRA adapters of each client in a setting with $N$ clients. Unlike the approach in FFA-LoRA, where $\mathbf{A}$ is consistently frozen, we suggest an alternating update strategy. There are alternating odd and even communication rounds designated for updating, aggregating $\mathbf{A}$ and $\mathbf{B}$, respectively.
\begin{align}
 &\text{In the odd-numbered comm. round:}\nonumber \\ &\qquad  \frac{1}{N} \sum_{i=1}^N  \Delta \mathbf{W}_{i}^{2t+1}= \frac{1}{N} (\mathbf{A}_1^{t}\mathbf{B}_1^{t+1}+...+ \mathbf{A}_N^{t}\mathbf{B}_N^{t+1}) 
  = \frac{1}{N} \mathbf{A}^t(\mathbf{B}_1^{t+1}+...+ \mathbf{B}_N^{t+1}) \label{roundt} \\
  & \text{In the even-numbered comm. round:} \nonumber\\&\quad  \frac{1}{N} \sum_{i=1}^N  \Delta \mathbf{W}_{i}^{2t+2} = \frac{1}{N} (\mathbf{A}_1^{t+1}\mathbf{B}_1^{t+1}+...+ \mathbf{A}_N^{t+1}\mathbf{B}_N^{t+1}) 
= \frac{1}{N} (\mathbf{A}_1^{t+1}+...+ \mathbf{A}_N^{t+1})\mathbf{B}^{t+1} \label{roundt1}
\end{align}

\vspace{-0.2cm}
As in Algorithm~\ref{alg:rolora-llm} in Appendix, all clients freeze $\mathbf{A}^t$ and update $\mathbf{B}^{t}$ in the odd communication round. The central server then aggregates these updates to compute $\mathbf{B}^{t+1} = \frac{1}{N}\sum_{i=1}^{N}\mathbf{B}^{t+1}_i$ and distributes $\mathbf{B}^{t+1}$ back to the clients. In the subsequent communication round, clients freeze $\mathbf{B}^{t+1}$ and update $\mathbf{A}^{t}$. The server aggregates these to obtain $\mathbf{A}^{t+1} = \frac{1}{N}\sum_{i=1}^{N}\mathbf{A}^{t+1}_i$ and returns $\mathbf{A}^{t+1}$ to the clients. 
It is important to note that in round $2t+1$, the frozen $\mathbf{A}_i^t$ are identical across all clients, as they are synchronized with $\mathbf{A}^t$ from the central server at the beginning of the round. This strategy ensures that the update and aggregation method introduces no interference, as demonstrated in \eqref{roundt} and \eqref{roundt1}.

\vspace{-0.2cm}
\paragraph{Computation and Communication Cost.}
The parameter-freezing nature of RoLoRA enhances computational and communication efficiency. In each communication round, the number of trainable parameters in the model is effectively halved compared to FedAVG of LoRA. The only additional cost for RoLoRA compared to FFA-LoRA is the alternating freezing of the corresponding parameters. We remark this additional cost is negligible because it is applied to the clients' models and can be executed concurrently during the server's aggregation.

\vspace{-0.3cm}
\section{Analysis} \label{sec:analysis}
\vspace{-0.2cm}
In this section, we provide an intuitive analysis of why training the down-projection in the LoRA module is essential in a federated setting. We first present a theoretical comparison between RoLoRA and FFA-LoRA on a linear model. While simplified, this comparison offers a direct and rigorous examination of their solutions, clearly highlighting the limitations of FFA-LoRA in this fundamental case. We then empirically validate the theoretical findings using a two-layer non-linear neural network. Finally, we provide a convergence analysis for RoLoRA in non-convex loss landscapes, establishing standard federated learning guarantees.

\subsection{Theoretical Insights into Down-Projection Learning: Linear Model Analysis} 
\label{setup-analysis-homo}
We start by analyzing RoLoRA and FFA-LoRA within a simplified linear model, offering a clear and rigorous comparison that reveals the inherent limitations of FFA-LoRA’s approach. Consider a federated setting with $N$ clients, each with the following local linear model $f_i(\mathbf{X}_i)=\mathbf{X}_i\mathbf{a}\mathbf{b}^\top$
where $\mathbf{Y}_i \in \mathbb{R}^{m \times d}$, $ \mathbf{X}_i \in \mathbb{R}^{m \times d}$ with the sample size $m$, $\mathbf{a} \in \mathbb{R}^{d}$ (a unit vector) and $\mathbf{b} \in \mathbb{R}^{d}$ are the LoRA weights corresponding to rank $r=1$. In this setting, we model the local data of $i$-th client such that 
\begin{align}
\mathbf{Y}_i=\mathbf{X}_i\mathbf{a}^*\mathbf{b}^{*^\top} \label{data-gen}
\end{align}
for some ground truth LoRA weights $\mathbf{a}^* \in \mathbb{R}^{d}$ (a unit vector) and $\mathbf{b}^* \in \mathbb{R}^d$. We consider the following objective
\begin{equation}
    \min_{\mathbf{a}\in \mathbb{R}^{d},\mathbf{b}\in \mathbb{R}^d} \frac{1}{N}\sum_{i=1}^N l_i(\mathbf{a},\mathbf{b}) \label{obj-xi}
\end{equation}
where the local loss is $l_i(\mathbf{a},\mathbf{b}) = \frac{1}{m}\lVert  \mathbf{X}_i\mathbf{a}^*\mathbf{b}^{*^\top}- \mathbf{X}_i\mathbf{a}\mathbf{b}^\top\rVert^2$. Each $\mathbf{X}_i$ is assumed to be a Gaussian random matrix, where each entry is independently and identically distributed according to a standard Gaussian distribution.

 We remind the reader that Section~\ref{subsec:notation} provides a summary of mathematical notations and also point to
 Table~\ref{tab:notation} in Appendix~\ref{sec:notation} for a summary of the symbols used throughout the theoretical analysis.
 
\paragraph{Results.}
In this section, we assume homogeneous clients where there is a single target model as in \eqref{data-gen}. In the linear model case, we modify RoLoRA from Algorithm~\ref{alg:rolora-llm} to Algorithm~\ref{alg:rolora-linear}, employing alternating minimization for $\mathbf{b}$ (line 5) and gradient descent for $\mathbf{a}$ (line 9). Details are described in Algorithm~\ref{alg:rolora-linear} in Appendix.  We note that the analysis of the alternating minimization-gradient descent algorithm is inspired by \cite{pmlr-v139-collins21a,seyedehsara2022fastsampleefficientfederatedlow,Vaswani_2024} for a different setting of MLRL. See further discussion in Appendix~\ref{app:discussion}.

We aim to show that the training procedure described in Algorithm~\ref{alg:rolora-linear} learns the target model $(\mathbf{a}^*,\mathbf{b}^*)$ by showing the angle distance (Definition~\ref{def:angle-distance}) between $\mathbf{a}$ and $\mathbf{a}^*$ is decreasing in each iteration. Since $\mathbf{b}$ is solved exactly at each iteration via local minimization, it is always optimal with respect to the current $\mathbf{a}$. This allows us to isolate and analyze the convergence behavior of $\mathbf{a}$ using the angle distance, eliminating the potential impact of insufficient local updates of $\mathbf{b}$ on the convergence of $\mathbf{a}$.
 First, we make typical assumptions on the ground truth $\mathbf{b}^*$ and formally define the angle distance.

\begin{assumption}  \label{client-norm-1}
     There exists $L_{max} < \infty$ (known a priori),  s.t. $\lVert \mathbf{b}^*\rVert \leq  L_{max}$.
\end{assumption} 
\vspace{-0.2cm}
\begin{definition} (Angle Distance) \label{def:angle-distance}
    For two unit vectors $\mathbf{a}, \mathbf{a}^* \in \mathbb{R}^{d}$, the angle distance between $\mathbf{a}$ and $\mathbf{a}^*$ is defined as 
    \begin{equation}
        |\sin \theta ({\mathbf{a}, \mathbf{a}^*})| = \lVert (\mathbf{I}_d- \mathbf{a} \mathbf{a}^\top)\mathbf{a}^*\rVert
    \end{equation}
    where $\mathbf{I}_d- \mathbf{a} \mathbf{a}^\top$ is the projection operator to the direction orthogonal to $\mathbf{a}$.
\end{definition}
Let $\delta^t = \lVert (\mathbf{I}_d-\mathbf{a}^*\mathbf{a}^{*^\top})\mathbf{a}^t\rVert = \lVert (\mathbf{I}_d-\mathbf{a}^t\mathbf{a}^{t^\top})\mathbf{a}^*\rVert$ denote the angle distance between $\mathbf{a}^{*}$ and $\mathbf{a}^t$ of $t$-th iteration. We initialize $\mathbf{a}^0$ such that $|\sin{\theta}(\mathbf{a}^*,\mathbf{a}^0)|= \delta_0$, where $0<\delta_0<  1$, and $\mathbf{b}^0$ is zero. All clients obtain the same initialization for parameters. 
Now we are ready to state our main results. 


\begin{lemma}\label{lemma:delta-t}
    Let $\delta^t = \lVert (\mathbf{I}_d-\mathbf{a}^*\mathbf{a}^{*^\top})\mathbf{a}^t\rVert $ be the angle distance between $\mathbf{a}^{*}$ and $\mathbf{a}^{t}$ of $t$-th iteration. Assume that Assumption~\ref{client-norm-1} holds and $\delta^t \leq \delta^{t-1} \leq \dots \leq \delta^0$. Let $m$ be the number of samples for each updating step, let auxiliary error thresholds $\epsilon'=\frac{\epsilon_2}{(1-\epsilon_0)(1-\epsilon_1)},\tilde{\epsilon}=\frac{\epsilon_3}{1-\epsilon_0}$ for $\epsilon_0, \epsilon_1, \epsilon_2, \epsilon_3 \in (0,1)$, if $
    m= \Omega(q)$ for $q=\max\left(\frac{\log(N)}{[\min(\epsilon_1,\epsilon_2)]^2},\frac{d\log(\frac{2}{\epsilon_0})}{\epsilon_2^2}\right)$,
    and auxiliary error thresholds are small such that $\epsilon',\tilde{\epsilon}<\frac{1-(\delta^0)^2}{16}$, for any $t$ and $\eta \leq \frac{1}{L_{max}^2}$, then we have,
    \begin{align}
        \delta^{t+1} \leq \delta^{t} \sqrt{1-\eta (1-\delta^{0^2})\|\mathbf b^*\|^2} 
    \end{align}
    with probability at least $1-2q^{-10}$.
\end{lemma}

Theorem~\ref{convergence-1} follows by recursively applying Lemma~\ref{lemma:delta-t} and taking a union bound over all $t\in[T]$.
\begin{remark}
The decreasing angle assumption in Lemma~\ref{lemma:delta-t} is a technical tool to simplify the proof. In Theorem~\ref{convergence-1}, this condition is not required: the inductive hypothesis inherently enforces the necessary bounds on angles, bypassing the need for explicit monotonicity.
\end{remark}

\begin{theorem}(Convergence of RoLoRA for linear regressor) \label{convergence-1}
    Suppose we are in the setting described in Section~\ref{setup-analysis-homo} and apply Algorithm~\ref{alg:rolora-linear} for optimization. Given a random initial $\mathbf{a}^0$, an initial angle distance $\delta_0 \in (0,1)$, we set step size $\eta  \leq \frac{1}{L_{max}^2}$ and the number of iterations $T \geq \frac{2}{c(1-(\delta^{0})^2)}\log (\frac{\delta^0}{\epsilon}$), for $c \in (0,1)$. Under these conditions, if with sufficient number of samples $m = \Omega(q)$ and small auxiliary error thresholds $\epsilon'=\frac{\epsilon_2}{(1-\epsilon_0)(1-\epsilon_1)},\tilde{\epsilon}=\frac{\epsilon_3}{1-\epsilon_0},$ such that $ \epsilon', \tilde{\epsilon} < \frac{1-(\delta^0)^2}{16}$ , we achieve that with probability at least $1-2Tq^{-10}$ for $q=\max\left(\frac{\log(N)}{[\min(\epsilon_1,\epsilon_2)]^2},\frac{d\log(\frac{2}{\epsilon_0})}{\epsilon_2^2}\right)$,
\begin{align}
\sin \theta (\mathbf{a}^T, \mathbf{a}^*)\leq \epsilon \nonumber
\end{align}
which we refer to as $\epsilon$-accurate recovery. In addition, 
\begin{align}
    \lVert \mathbf{a}^T(\mathbf{b}^{T+1})^\top- \mathbf{a}^* ({\mathbf{b}}^*)^\top\rVert \leq (1+\epsilon')\epsilon \lVert  \mathbf a^* {\mathbf{b}}^{*^\top}\rVert. \nonumber 
\end{align}
\end{theorem}
\vspace{-0.2cm}
Theorem~\ref{convergence-1} and Lemma~\ref{lemma:delta-t} show that with a random initialization for the unit vector $\mathbf{a}$ $(\delta^0\in (0,1))$, RoLoRA makes the global model converge to the target model exponentially fast with large $q$. The requirement for sample complexity is well-supported, as demonstrated in \cite{collins2022fedavgfinetuninglocal, du2021fewshot}. While the proof of the above results are relegated to the Appendix, we provide a brief outline of the proof. In Appendices~\ref{sec:homo-analysis}, we first analyze the minimization step for updating ${\mathbf{b}_i^t}$ (Lemma~\ref{lemma:b-bar-g-bound}), then establish a bound on the deviation of the gradient from its expectation with respect to $\mathbf{a}$ (Lemma~\ref{lemma:grad-egrad-bound}), and finally derive a bound for $|\sin \theta (\mathbf{a}^{t+1},\mathbf{a}^*)|$ based on the gradient descent update rule for $\mathbf{a}$ (Lemma~\ref{lemma:delta-t}). The proof of Theorem~\ref{convergence-1} is in Section~\ref{subsec:mainProof}.

\paragraph{Intuition on Freezing-A Scheme (FFA-LoRA) can Saturate.}
We begin by applying the FFA-LoRA scheme to a centralized setting, aiming to solve the following optimization problem:
\[
    \min_{\mathbf{b} \in \mathbb{R}^d} \lVert \mathbf{X}\mathbf{a}^*\mathbf{b}^{*^\top}-\mathbf{X}\mathbf{a}^{0} \mathbf{b}^\top \rVert^2 
\]
where $\mathbf{a}^* \in \mathbb{R}^d$ and $\mathbf{b}^* \in \mathbb{R}^d$ represent the ground truth parameters, and $\mathbf{a}^0 \in \mathbb{R}^d$ is the random initialization. The objective can be transformed to $\sum_{p=1}^{d}(\mathbf{a}^*{b}^{*}_p-\mathbf{a}^{0} {b}_p)^\top \mathbf{X}^\top \mathbf{X}(\mathbf{a}^*{b}^{*}_p-\mathbf{a}^{0} {b}_p)$, with ${b}_p$ as the $p$-th entry of $\mathbf{b}$, ${b}_p^*$ as the $p$-th entry of $\mathbf{b}^*$. In FFA-LoRA scheme, $\mathbf{a}^0$ remains fixed during training. If $\mathbf{a}^0$ is not initialized to be parallel to $\mathbf{a}^*$, the objective can never be reduced to zero. This is because optimizing $\mathbf{b}$ only scales the vector $\mathbf{a}^{0} b_p$ along the direction of $\mathbf{a}^{0}$, without altering the angular distance between $\mathbf{a}^{0}$ and $\mathbf{a}^*$.

Suppose we are in the federated setting described in Section~\ref{setup-analysis-homo}, we apply FFA-LoRA, to optimize the objective in \eqref{obj-xi}. In FFA-LoRA scheme, we fix $\mathbf{a}$ of all clients to a random unit vector $\mathbf{a}^0$, where the initial angle distance $\delta^0 = |\sin \theta ( \mathbf{a}^*, \mathbf{a}^0)|, \delta^0\in (0,1)$. And we only update $\mathbf{b}_i$ by minimizing $l_i$ and aggregate them.  
\begin{proposition} \label{fa-saturate-homo}(FFA-LoRA lower bound)
Suppose we are in the setting described in Section~\ref{setup-analysis-homo}. For any set of ground truth parameters ($\mathbf{a}^*,\mathbf{b}^*$), the initialization $\mathbf{a}^0$, initial angle distance $\delta^0\in (0,1)$, we apply FFA-LoRA scheme to obtain a shared global model ($\mathbf{a}^0,{\mathbf{b}}^{FFA}$), yielding an expected global loss of 
\begin{align}
    \mathbb{E}[\frac{1}{Nm}\sum_{i=1}^N\lVert  \mathbf{X}_i\mathbf{a}^*\mathbf{b}^{*^\top}- \mathbf{X}_i\mathbf{a}^0({{\mathbf{b}}^{FFA})^\top}\rVert^2] = (1+\tilde{c})\| \mathbf{b}^*\|^2(\delta^0)^2 \label{loss-ffa}
\end{align}
where the expectation is over all the randomness in the $\mathbf{X}_i$, and $\tilde{c}=O(\frac{1}{Nm})$.
\end{proposition}

See Appendix~\ref{fa-heter-saturate} for the proof. 
\begin{remark}
    Proposition~\ref{fa-saturate-homo} holds for any unit vector $\mathbf{a}$ and corresponding $\mathbf{b}$ obtained by fully minimizing the local loss. The same expected loss applies to RoLoRA by substituting RoLoRA's reduced angle into Eq.~\ref{loss-ffa}.
\end{remark}
\vspace{-0.2cm}
\paragraph{Comparison of RoLoRA and FFA-LoRA.} Proposition~\ref{fa-saturate-homo} shows that for any choice of $\delta^0 \in (0,1)$, the global objective reached by FFA-LoRA is shown as in \eqref{loss-ffa}. The global objective of FFA-LoRA is dominated by $\|\mathbf b^*\|^2 (\delta^0)^2$  which is due to the angular distance between $\mathbf{a}^0$ and $\mathbf{a}^*$. 

By Theorem~\ref{convergence-1}, we demonstrate that RoLoRA achieves $\epsilon$-accurate recovery of the global minimizer. Specifically, the expected global loss of RoLoRA can be upper bounded by $(1+\tilde{c})\| \mathbf{b}^*\|^2\epsilon^2$. Under the same initialization and ground truth parameters for both FFA-LoRA and RoLoRA, RoLoRA's ability to update $\mathbf{a}$ reduces the global loss caused by the angle distance between $\mathbf{a}$ and $\mathbf{a}^*$ from $\| \mathbf{b}^*\|^2(\delta^0)^2$ to $\| \mathbf{b}^*\|^2\epsilon^2$. By increasing the number of iterations, $\epsilon$ can be made arbitrarily small. 

\paragraph{Heterogeneous Case.} In Appendix~\ref{rank-1-heter-clients}, we analyze the convergence of RoLoRA with single LoRA structure in a federated setting with \textit{heterogeneous} clients. By showing the decreasing of the angle distance between the ground truth $\mathbf{a}^*$ and the shared down-projection $\mathbf{a}$, we demonstrate that RoLoRA allows the global model to converge to global minimum while the global loss of FFA-LoRA can be dominated by the term caused by the angle distance between the random initialization $\mathbf{a}^0$ and $\mathbf{a}^*$.
\vspace{-0.3cm}


\subsection{Verifying Insights On a Non-Linear Model}
\begin{wrapfigure}[13]{R}{0.5\textwidth}
\includegraphics[width=0.5\textwidth]{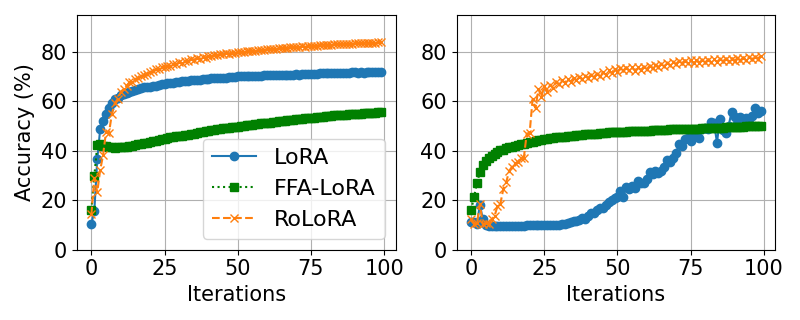} 
\caption{(Left) Comparison of three methods on a toy model with 5 clients. (Right) Comparison of three methods on a toy model with 10 clients.}
\label{fig:mnist-results}
\end{wrapfigure}

The previous analysis considers a linear model for each client. To assess the validity of the theorem in a non-linear model, we consider a two-layer neural network model on each client given by
\begin{equation}
    f_i(x_i) = \textsf{ReLU}( x_i\mathbf{A}\mathbf{B})\mathbf{W}_{out}
\end{equation}
where $ \mathbf{W}_{out} \in \mathbb{R}^{d\times c}$,  $\mathbf{A}\in \mathbb{R}^{d\times r}$ and $\mathbf{B}\in\mathbb{R}^{r\times d}$ are weights. We train the model on MNIST \cite{deng2012mnist}. We consider two different ways to distribute training images to clients. The first is to distribute the images to 5 clients and each client gets access to training images of two specific labels, while the second is to distribute the images to 10 clients and each client only has training images of one specific label. There is no overlap in the training samples each client can access. Only weights matrices $\mathbf{B}$ and $\mathbf{A}$ are tunable, while $\mathbf{W}_{out}$ are fixed. We use $c=10, d=784, r=16$ and make each client train 5 epochs locally with batch-size 64 and aggregate clients' update following three methods: FedAVG of LoRA, referred as LoRA; FFA-LoRA \cite{sun2024improving}, which freezes $\mathbf{A}$ during training, and RoLoRA, which alternately updates $\mathbf{B}$ and $\mathbf{A}$. 

 As shown in Figure~\ref{fig:mnist-results}, we evaluate the performance of the model in each iteration on the test set. 
We observe that the accuracy of FFA-LoRA plateaus around 55\% in both settings, which aligns with our theoretical analysis. The decline in LoRA’s performance with an increasing number of clients is most likely due to less accurate model aggregation, as demonstrated in \eqref{fedavg-mul} and \eqref{fedavg-lora}. Notably, RoLoRA demonstrates greater robustness in these settings. \textit{ To study the impact of non-linearity on RoLoRA}, we repeat the two-layer experiment on a linear network without ReLU. As shown in Figure~\ref{fig:non-linear-add-to-rolora} in the Appendix, RoLoRA benefits more from the added expressiveness of ReLU.

\vspace{-0.3cm}
\subsection{Convergence in Smooth Non-Convex Settings}
\vspace{-0.2cm}
We follow the approach of Li et al.\cite{Li2020On} to analyze the convergence behavior of RoLoRA in smooth, non-convex settings, and derive Theorem\ref{convergence-non-convex} in Appendix. As shown in Theorem\ref{convergence-non-convex}, RoLoRA achieves an $O(1/\sqrt{T})$ convergence rate toward a stationary point under smooth, non-convex conditions, matching the convergence rate established for FedAVG in the same regime.

\vspace{-0.2cm}
\section{Experiments on Language Models}\label{exp}
\vspace{-0.2cm}
In this section, we evaluate the performance of RoLoRA in various federated settings. 
Considering all clients will participate in each round, we will explore the following methods: FedAVG of LoRA (referred as LoRA) \cite{zhang2023towards}, FFA-LoRA \cite{sun2024improving}, FlexLoRA\cite{bai2024federatedfinetuninglargelanguage}, FloRA\cite{wang2024florafederatedfinetuninglarge}, and RoLoRA (ours).

\vspace{-0.3cm}
\paragraph{Implementation \& Configurations.}We implement all the methods based on FederatedScope-LLM \cite{kuang2023federatedscopellm}. We use NVIDIA GeForce RTX 4090 or NVIDIA A40 for all the experiments. To make a fair comparison, for each dataset, we obtain the best performance on test set and report the average over multiple seeds. Specifically, the learning rate is chosen from the set $\{5e-4, 1e-3, 2e-3, 5e-3, 1e-2, 2e-2, 5e-2, 1e-1\}$. Other hyper-parameters for experiments are specified in Table~\ref{tab:exp-set} in Appendix~\ref{exp-setup}. Please note that in all tasks, we compare the performance of the three methods \textit{under the same number of communication rounds}.
\vspace{-0.2cm}
\subsection{Language Understanding Tasks} \label{exp-LU}
\vspace{-0.2cm}
\paragraph{Model and Datasets.} We take the pre-trained RoBERTa-Large (355M) \cite{liu2019roberta} models from the HuggingFace Transformers library, and evaluate the performance of federated finetuning methods on 5 datasets (SST-2, QNLI, MNLI, QQP, RTE) from the GLUE \cite{wang2019glue}. 

\begin{table}[h!]
\centering
{\scriptsize
\begin{tabular}{ccccccccc}
    \toprule
    Rank & Clients Num & Method & SST-2 & QNLI & MNLI & QQP & RTE & Avg. \\
    \midrule

    \multirow{4}{*}{4} & \multirow{4}{*}{3}
    & LoRA       & $\text{\textbf{95.62}}_{\pm \text{0.17}}$ & $\text{91.59}_{\pm \text{0.21}}$ & $\text{\textbf{86.20}}_{\pm \text{0.05}}$ & $\text{86.13}_{\pm \text{0.10}}$ & $\text{81.46}_{\pm \text{1.22}}$ & 88.20 \\
    &             & FFA-LoRA   & $\text{95.18}_{\pm \text{0.09}}$ & $\text{91.35}_{\pm \text{0.32}}$ & $\text{84.58}_{\pm \text{0.21}}$ & $\text{85.50}_{\pm \text{0.25}}$ & $\text{81.10}_{\pm \text{0.33}}$ & 87.48 \\
    &             & FlexLoRA   & $\text{94.91}_{\pm \text{0.18}}$ & $\text{90.16}_{\pm \text{0.49}}$ & $\text{85.16}_{\pm \text{0.69}}$ & $\text{85.69}_{\pm \text{0.17}}$ & $\text{79.3}_{\pm \text{1.05}}$  & 87.04 \\
    &             & \cellcolor{ours} RoLoRA & \cellcolor{ours} $\text{95.49}_{\pm \text{0.16}}$ & \cellcolor{ours} $\text{\textbf{91.64}}_{\pm \text{0.30}}$ & \cellcolor{ours}  $\text{85.70}_{\pm \text{0.04}}$ & \cellcolor{ours}  $\text{\textbf{86.14}}_{\pm \text{0.06}}$ & \cellcolor{ours} $\text{\textbf{82.43}}_{\pm \text{0.84}}$ & \cellcolor{ours} \textbf{88.28} \\
    \midrule

    \multirow{4}{*}{4} & \multirow{4}{*}{20}
    & LoRA       & $\text{94.3}_{\pm \text{0.27}}$  & $\text{86.67}_{\pm \text{2.02}}$ & $\text{78.55}_{\pm \text{7.31}}$ & $\text{83.1}_{\pm \text{0.04}}$  & $\text{51.87}_{\pm \text{3.24}}$ & 78.90 \\
    &             & FFA-LoRA   & $\text{93.88}_{\pm \text{0.06}}$ & $\text{89.11}_{\pm \text{0.19}}$ & $\text{80.99}_{\pm \text{1.74}}$ & $\text{83.92}_{\pm \text{0.2}}$  & $\text{57.16}_{\pm \text{1.46}}$ & 80.01 \\
    &             & FlexLoRA   & $\text{90.97}_{\pm \text{1.78}}$ & $\text{54.36}_{\pm \text{0.36}}$ & $\text{53.30}_{\pm \text{14.59}}$ & $\text{69.18}_{\pm \text{10.39}}$ & $\text{53.19}_{\pm \text{1.45}}$ & 64.20 \\
    &             & \cellcolor{ours}  RoLoRA & \cellcolor{ours} $\text{\textbf{94.88}}_{\pm \text{0.18}}$ & \cellcolor{ours}  $\text{\textbf{90.35}}_{\pm \text{0.37}}$ & \cellcolor{ours} $\text{\textbf{85.28}}_{\pm \text{1.04}}$ & \cellcolor{ours}  $\text{\textbf{85.83}}_{\pm \text{0.1}}$ & \cellcolor{ours} $\text{\textbf{78.82}}_{\pm \text{1.7}}$ & \cellcolor{ours}  \textbf{87.03} \\
    \midrule

    \multirow{4}{*}{4} & \multirow{4}{*}{50}
    & LoRA       & $\text{93.00}_{\pm \text{0.35}}$ & $\text{78.13}_{\pm \text{5.13}}$ & $\text{52.64}_{\pm \text{15.07}}$ & $\text{77.60}_{\pm \text{1.47}}$ & $\text{52.23}_{\pm \text{1.1}}$  & 70.72 \\
    &             & FFA-LoRA   & $\text{93.23}_{\pm \text{0.12}}$ & $\text{85.05}_{\pm \text{0.34}}$ & $\text{69.97}_{\pm \text{5.57}}$  & $\text{78.44}_{\pm \text{0.41}}$ & $\text{55.72}_{\pm \text{1.99}}$ & 76.48 \\
    &             & FlexLoRA   & $\text{54.08}_{\pm \text{5.5}}$  & $\text{55.4}_{\pm \text{2.03}}$  & $\text{39.14}_{\pm \text{2.35}}$  & $\text{72.00}_{\pm \text{7.64}}$ & $\text{52.71}_{\pm \text{0.00}}$ & 54.67 \\
    &             & \cellcolor{ours}  RoLoRA & \cellcolor{ours} $\text{\textbf{94.80}}_{\pm \text{0.17}}$ & \cellcolor{ours}  $\text{\textbf{90.00}}_{\pm \text{0.63}}$ & \cellcolor{ours} $\text{\textbf{82.98}}_{\pm \text{3.36}}$ & \cellcolor{ours}  $\text{\textbf{85.71}}_{\pm \text{0.18}}$ & \cellcolor{ours} $\text{\textbf{75.57}}_{\pm \text{2.88}}$ & \cellcolor{ours}  \textbf{85.81} \\
    \midrule

    \multirow{4}{*}{8} & \multirow{4}{*}{50}
    & LoRA       & $\text{93.00}_{\pm \text{0.23}}$ & $\text{79.87}_{\pm \text{1.52}}$ & $\text{56.96}_{\pm \text{2.02}}$ & $\text{77.45}_{\pm \text{1.97}}$ & $\text{53.79}_{\pm \text{6.57}}$ & 64.03 \\
    &             & FFA-LoRA   & $\text{92.74}_{\pm \text{0.13}}$ & $\text{83.69}_{\pm \text{0.75}}$ & $\text{64.51}_{\pm \text{1.92}}$ & $\text{79.71}_{\pm \text{2.04}}$ & $\text{53.07}_{\pm \text{1.3}}$  & 72.46 \\
    &             & FlexLoRA   & $\text{50.92}_{\pm \text{0.00}}$ & $\text{56.92}_{\pm \text{1.04}}$ & $\text{37.43}_{\pm \text{2.80}}$ & $\text{66.40}_{\pm \text{4.74}}$ & $\text{52.59}_{\pm \text{0.21}}$ & 52.85 \\
    &             & \cellcolor{ours}  RoLoRA & \cellcolor{ours} $\text{\textbf{94.53}}_{\pm \text{0.17}}$ & \cellcolor{ours}  $\text{\textbf{90.1}}_{\pm \text{0.45}}$ & \cellcolor{ours} $\text{\textbf{85.17}}_{\pm \text{0.41}}$ & \cellcolor{ours}  $\text{\textbf{85.25}}_{\pm \text{0.13}}$ & \cellcolor{ours} $\text{\textbf{76.3}}_{\pm \text{4.9}}$  & \cellcolor{ours}  \textbf{86.27} \\
    \bottomrule
\end{tabular}
\vspace{3pt}
\caption{Results for four methods with RoBERTa-Large models across varying client numbers (3, 20, 50), maintaining a constant sample count during fine-tuning.}
\label{tab:Clients-num-flex-appendix}
}
\end{table}

\newpage
\begin{wrapfigure}[19]{R}{0.5\textwidth}
\begin{center}
\begin{subfigure}
  \centering
  \includegraphics[width=0.425\linewidth]{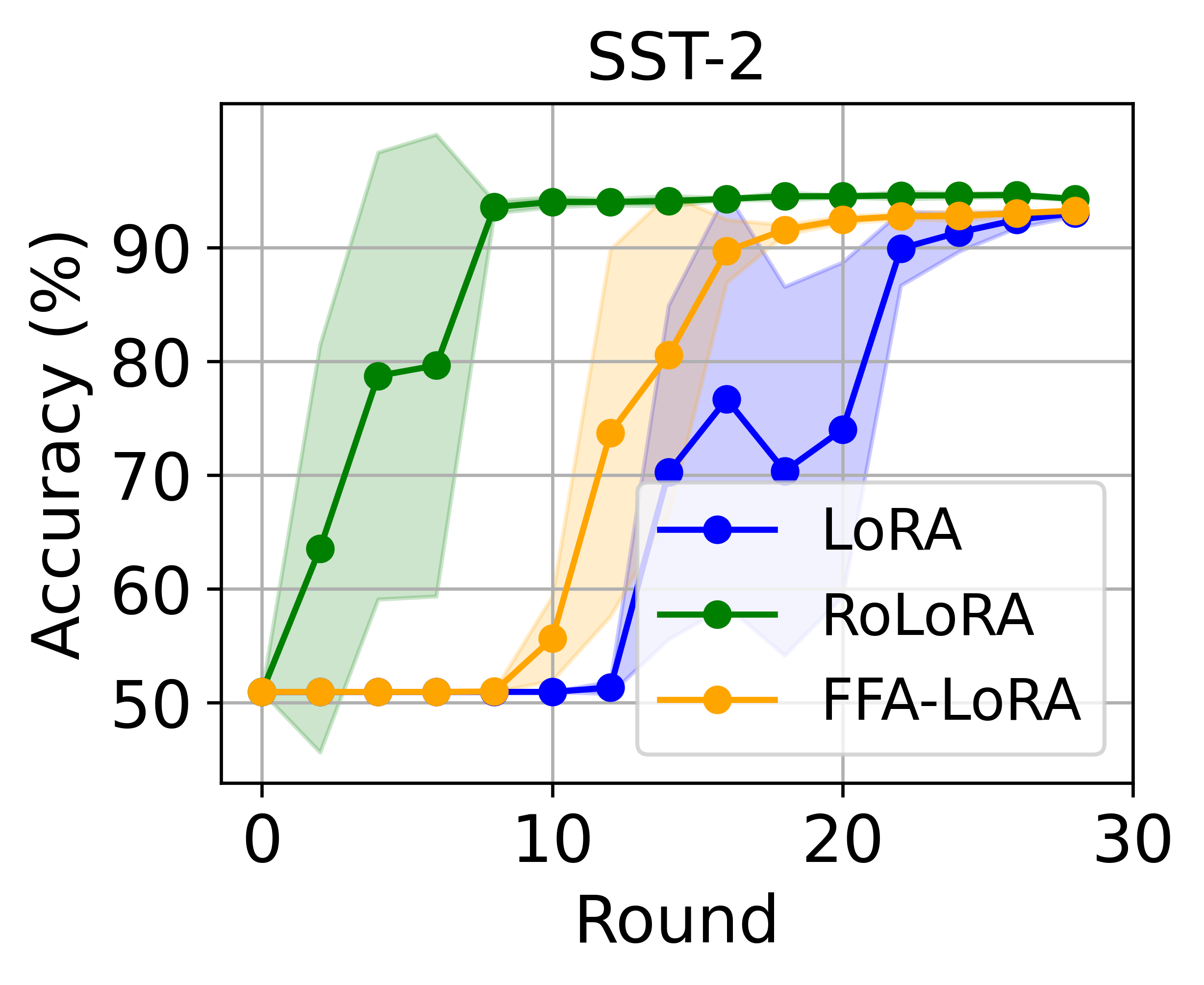}
  \label{fig:sub1-}
\end{subfigure}
\begin{subfigure}
  \centering
  \includegraphics[width=0.4\linewidth]{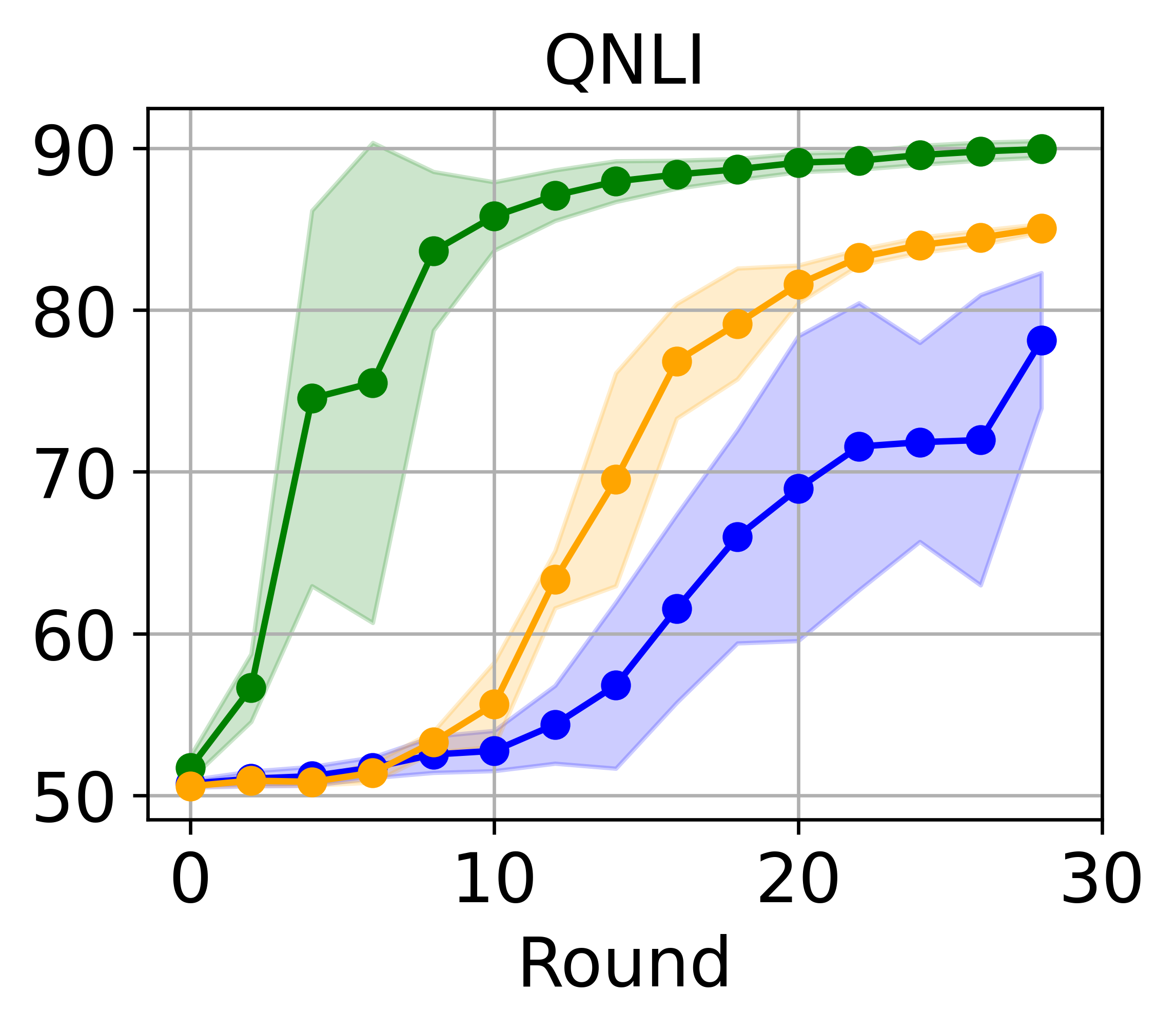}
  \label{fig:sub2-}
\end{subfigure}
\begin{subfigure}
  \centering
  \includegraphics[width=0.425\linewidth]{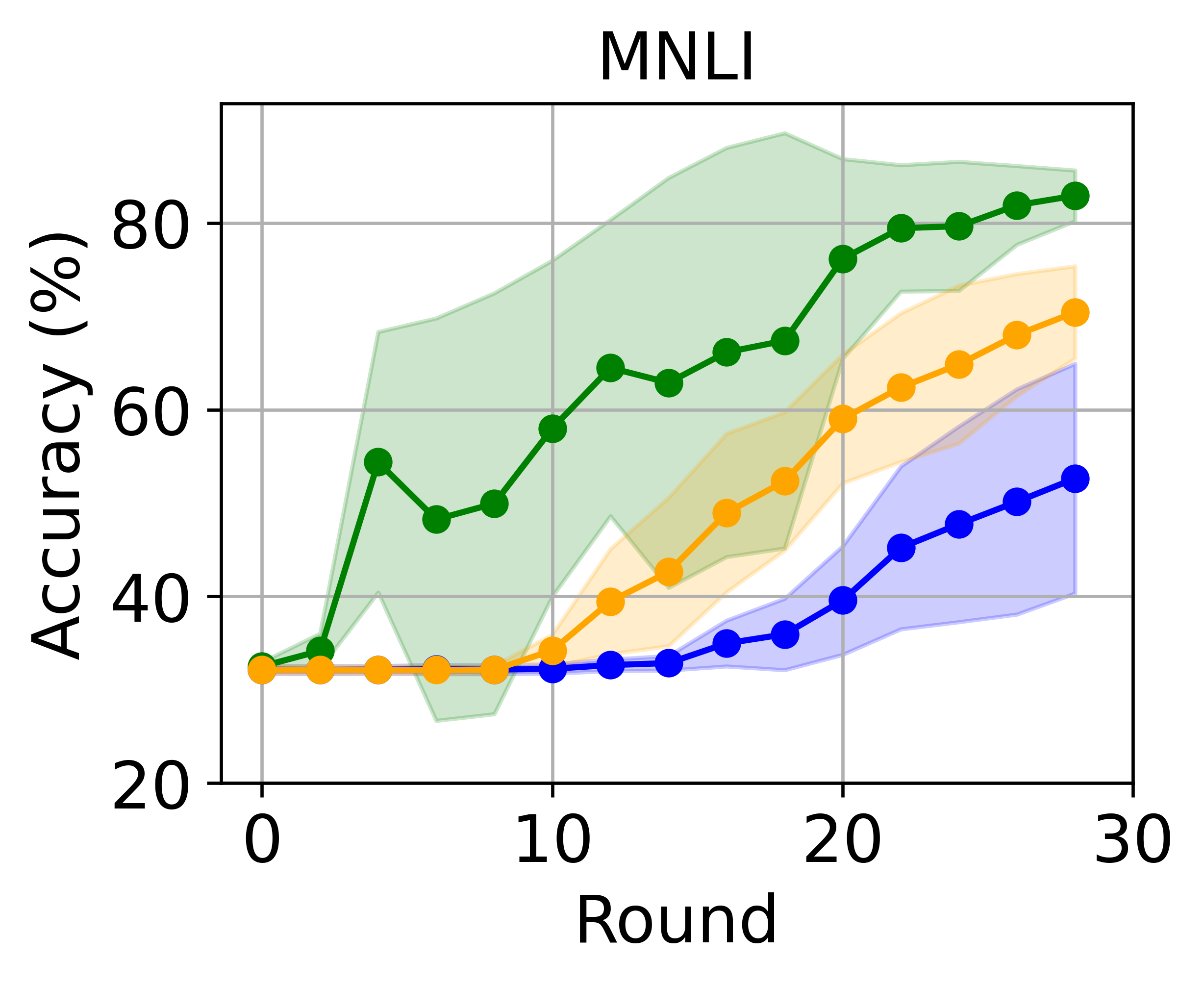}
  \label{fig:sub3-}
\end{subfigure}
\begin{subfigure}
  \centering
  \includegraphics[width=0.4\linewidth]{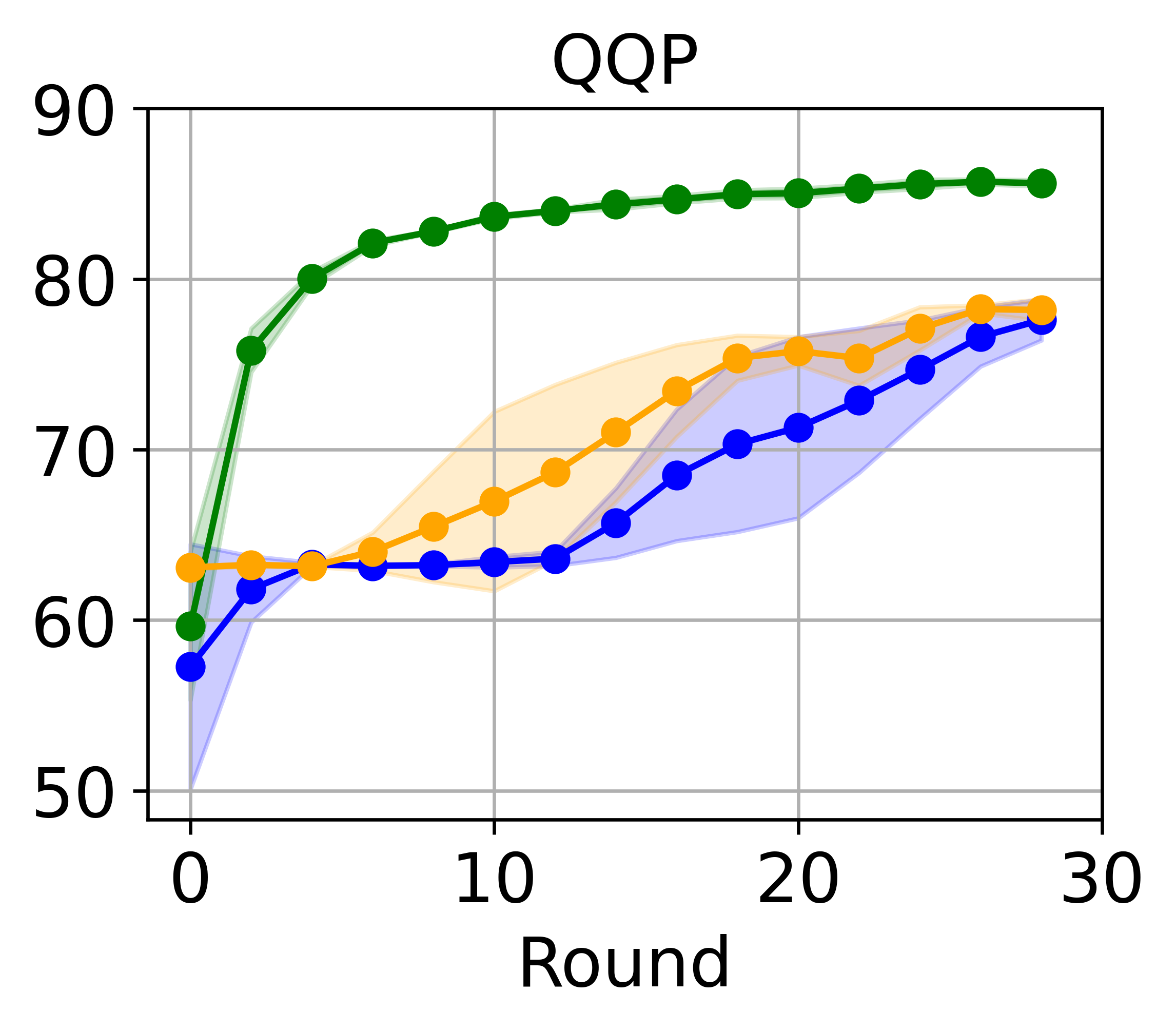}
  \label{fig:sub4-}
\end{subfigure}
 \caption{Accuracies over rounds with RoBERTa-Large models on SST-2, QNLI, MNLI, and QQP. It involves 50 clients using rank 4.}
 \label{fig:convergence-speed}
 \end{center}
 \end{wrapfigure}

\vspace{-0.2cm}
\paragraph{Effect of Number of Clients.}
In Table~\ref{tab:Clients-num-flex-appendix}, we increased the number of clients from 3 to 20, and then to 50, ensuring that there is no overlap in the training samples each client can access. Consequently, each client receives a smaller fraction of the total dataset. The configurations are presented in Table~\ref{tab:layer_type_index-client-num} in Appendix. We observe that as the number of clients increases, while maintaining the same number of fine-tuning samples, the performance of the LoRA method significantly deteriorates for most datasets. In contrast, RoLoRA maintains its accuracy levels. The performance of FFA-LoRA also declines, attributed to the limited expressiveness of the random initialization of $\mathbf{A}$ for clients' heterogeneous data. FlexLoRA shows significant degradation especially under high client counts. Notably, RoLoRA achieves this accuracy while incurring only half the communication costs associated with LoRA and FlexLoRA. A comparison when using rank-2 LoRA adapter is shown in Table~\ref{tab:Clients-num-appendix-rank-2} in Appendix.

Furthermore, we have provided performance comparison of FLoRA and RoLoRA in the same settings in Table~\ref{tab:flora_vs_rolora} in Appendix. RoLoRA consistently outperforms FLoRA across tasks and client counts. 
Additionally, Figure~\ref{fig:convergence-speed} illustrates the finetuning dynamics, highlighting that RoLoRA converges substantially faster than the other methods. An extended 100-round version is provided in Figure~\ref{fig:convergence-client50-100round} in the Appendix, further demonstrating RoLoRA’s superior accuracy when all baselines have converged.

\begin{wraptable}[12]{R}{0.6\textwidth}
{\scriptsize
\centering

\begin{tabular}{ll|cc|cc}
\toprule
\multirow{2}{*}{} & \multirow{2}{*}{} & \multicolumn{2}{c|}{Dir(0.5), \#Clients = 10} & \multicolumn{2}{c}{Dir(1.0), \#Clients = 15} \\
\cmidrule(lr){3-4} \cmidrule(lr){5-6}
 & & {MNLI} & {QQP} & {MNLI} & {QQP} \\
\midrule
{LoRA}     &   & 81.19 \com{0.23}  & 82.60 \com 0.41 & 74.54 \com 1.19 & 81.49 \com 0.60 \\
{FFA-LoRA} &   & 75.60 \com{0.21}  & 81.47 \com 0.87 & 74.83 \com 0.59 & 78.62 \com 1.67 \\
{FlexLoRA} &   & 35.45 \com 0.00   & 63.24 \com 0.09 & 35.45 \com 0.00 &  66.56 \com 4.48 \\
\rowcolor{ours} {RoLoRA}   &   & \textbf{82.60} \com 0.69   & \textbf{ 84.16 }\com 0.65 & \textbf{81.55} \com 0.44 & \textbf{84.26} \com 0.26\\
\bottomrule
\end{tabular}
\vspace{1pt} 
\caption{Performance comparison of methods on MNLI and QQP under different dirichlet distributions and client settings. We report the average and std. over three seeds.}
\label{tab:noniid-llm}}
\end{wraptable}
\vspace{-0.3cm}
\paragraph{Effect of Non-IID Data Distribution.}
Table~\ref{tab:noniid-llm} presents a performance comparison of LoRA, FFA-LoRA, FlexLoRA, and RoLoRA on the MNLI and QQP tasks under two federated settings: Dirichlet(0.5) with 10 clients and Dirichlet(1.0) with 15 clients. Here, the Dirichlet($\alpha$) parameter controls how non-iid the data is across clients: a smaller $\alpha$ produces highly skewed and heterogeneous client data distributions, while a larger $\alpha$ yields more balanced, moderately non-iid splits. Across both tasks and settings, RoLoRA consistently achieves the highest performance. The configurations are presented in Table~\ref{tab:layer_type_index-client-num} in Appendix.

\paragraph{Effect of Number of Finetuning Parameters.}
In Figure~\ref{fig:five_subfigures}, we compare three methods across five GLUE datasets. We apply LoRA module to every weight matrix of the selected layers, given different budgets of LoRA parameters. For each dataset, we experiment with three budgets (${\mathcal{P}_1, \mathcal{P}_2, \mathcal{P}_3}$) ranging from low to high. The corresponding layer sets that are attached with LoRA adapters, ${\mathcal{P}_1, \mathcal{P}_2, \mathcal{P}_3}$, are detailed in Table~\ref{tab:layer_type_index} in Appendix~\ref{exp-setup}. The figures indicates that with sufficient number of finetuning parameters, FFA-LoRA can achieve comparable best accuracy as LoRA and RoLoRA, aligning with the results in \cite{sun2024improving}; as the number of LoRA parameters is reduced, the performance of the three methods deteriorates to varying degrees. However, RoLoRA, which achieves performance comparable to LoRA, demonstrates greater robustness compared to FFA-LoRA, especially under conditions of limited fine-tuning parameters. It is important to note that with the same finetuning parameters, the communication cost of RoLoRA and FFA-LoRA is always half of that of LoRA due to their parameter freezing nature. This implies that RoLoRA not only sustains its performance but also enhances communication efficiency. Additional results of \textit{varying ranks} are provided in Figure~\ref{fig:five_subfigures-rank-8}, \ref{fig:five_subfigures-rank-2}, and \ref{fig:five_subfigures-rank-1} in Appendix~\ref{app:num-lora-appendix}.
\begin{figure}
\begin{center}
\begin{subfigure}
  \centering
  \includegraphics[width=0.161\linewidth]{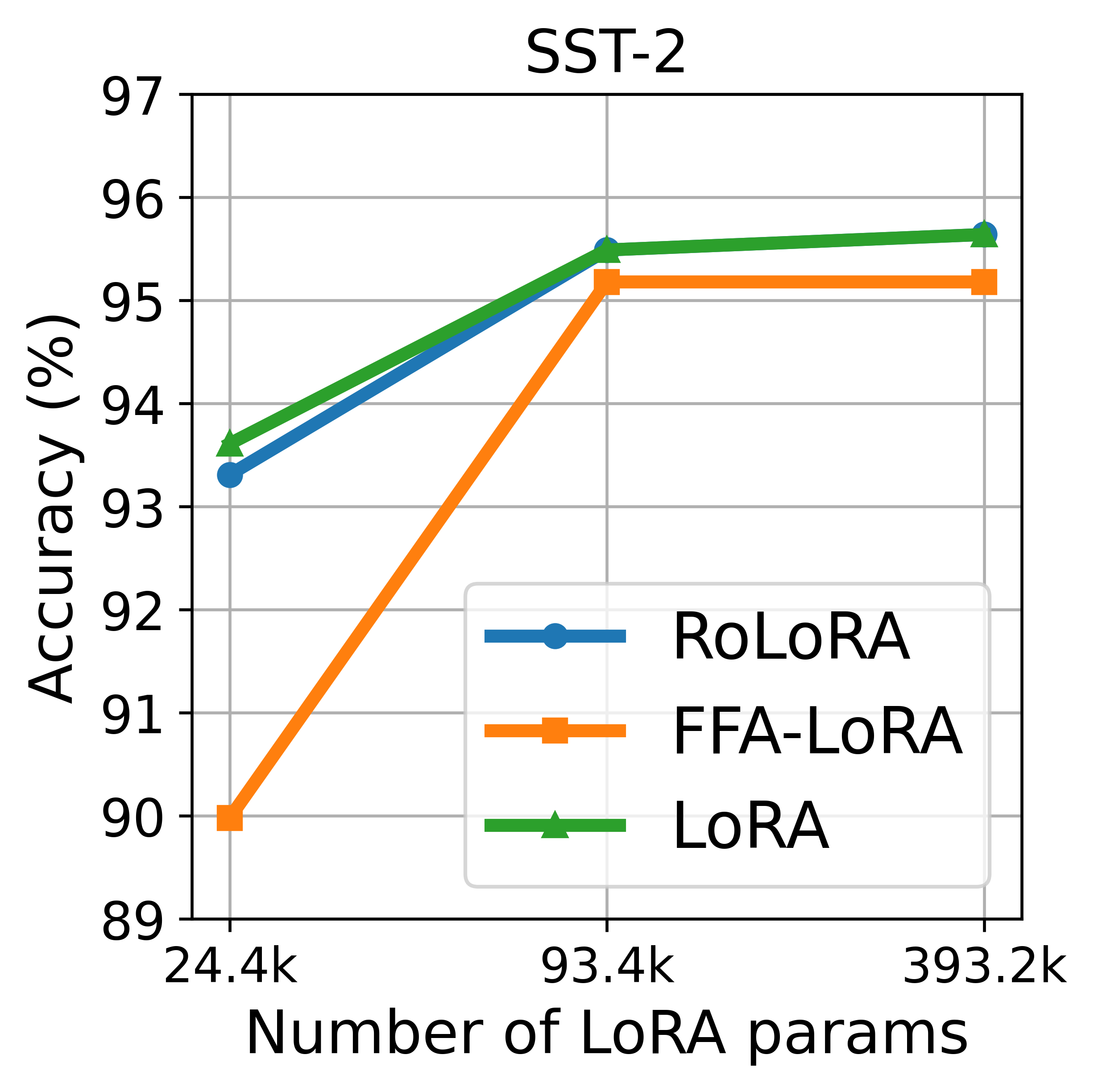}
  \label{fig:sub1}
\end{subfigure}
\hfill
\begin{subfigure}
  \centering
  \includegraphics[width=0.154\linewidth]{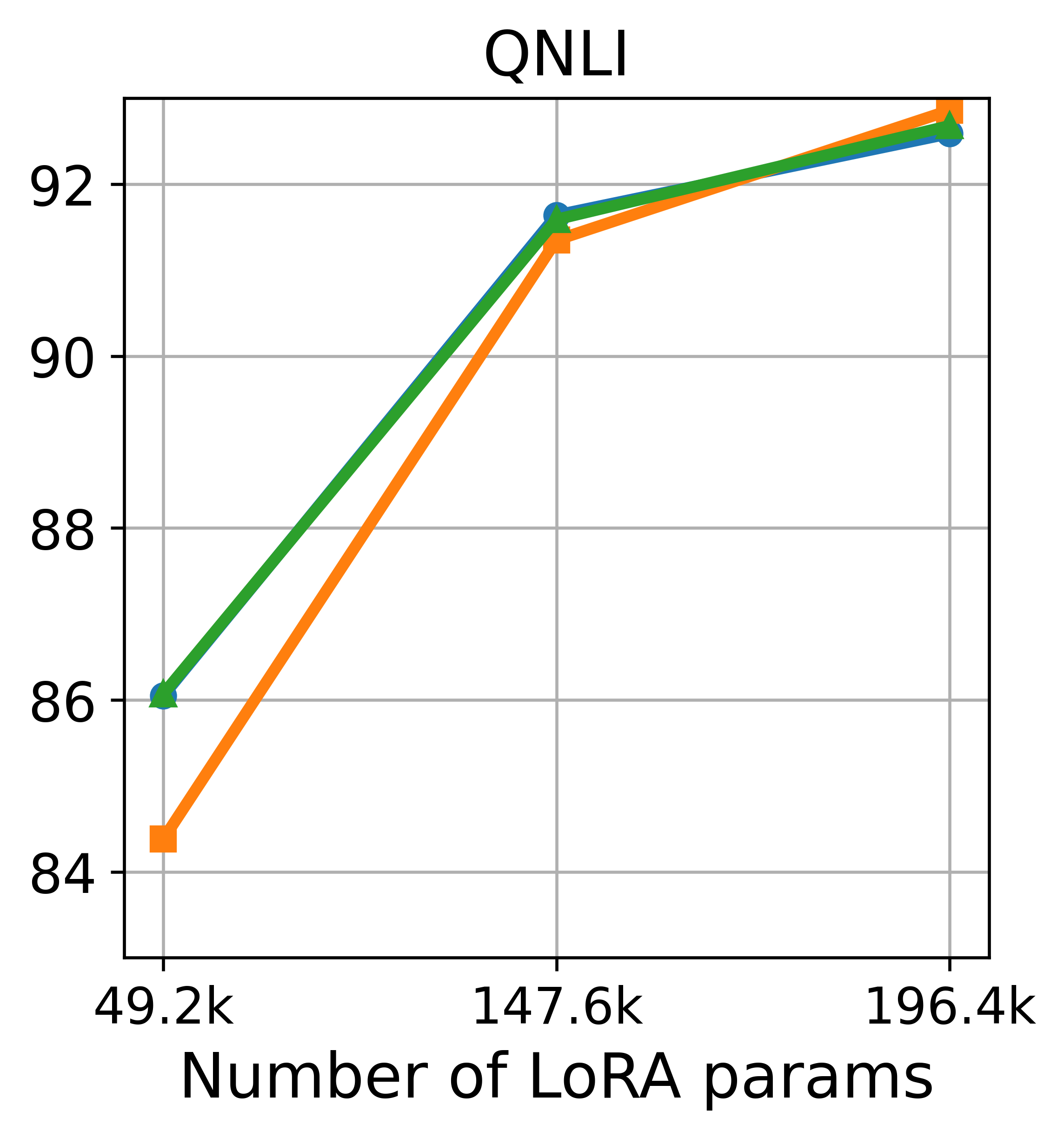}
  \label{fig:sub2}
\end{subfigure}
\hfill
\begin{subfigure}
  \centering
  \includegraphics[width=0.154\linewidth]{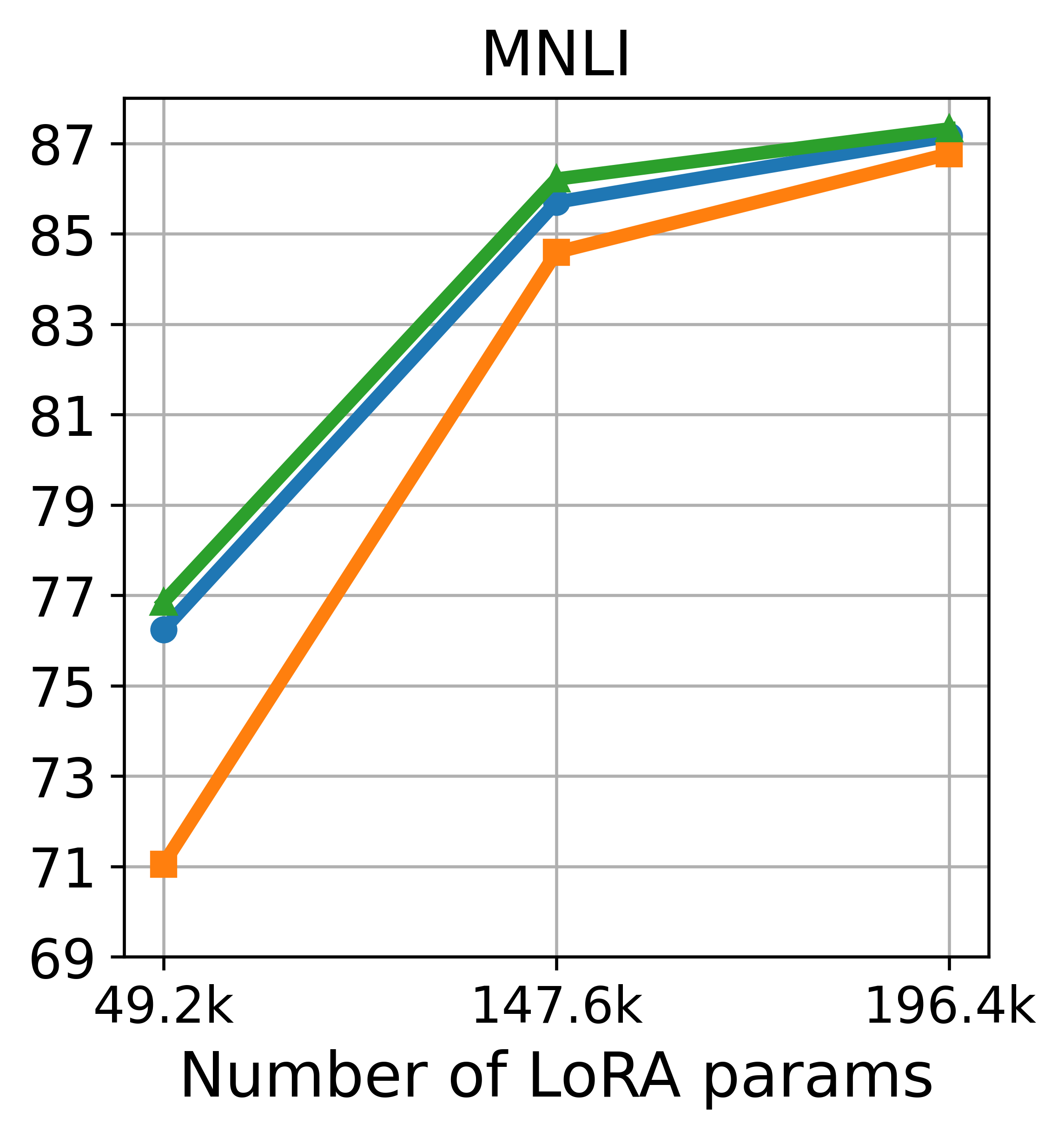}
  \label{fig:sub3}
\end{subfigure}
\hfill
\begin{subfigure}
  \centering
  \includegraphics[width=0.154\linewidth]{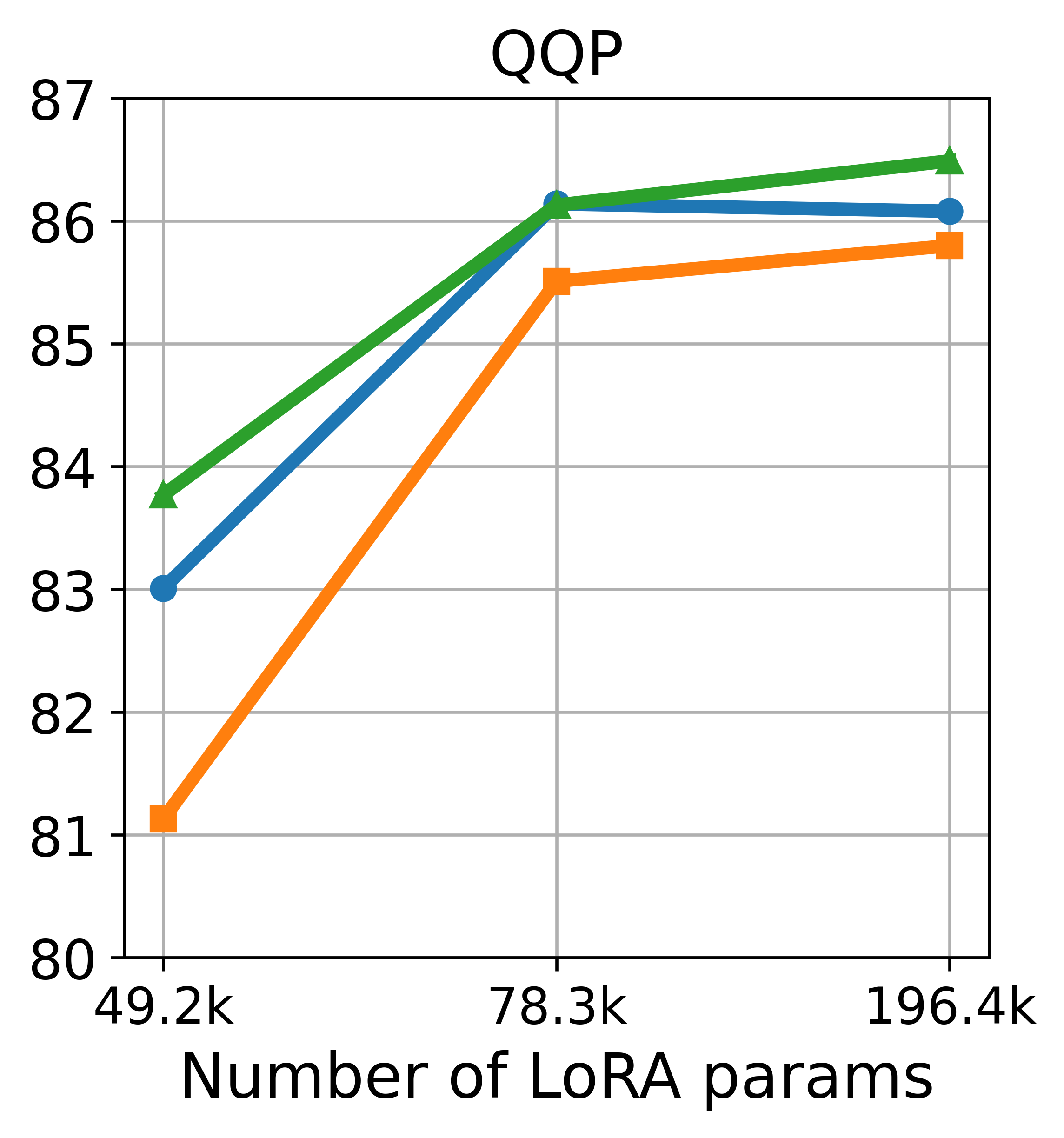}
  \label{fig:sub4}
\end{subfigure}
\hfill
\begin{subfigure}
  \centering
  \includegraphics[width=0.154\linewidth]{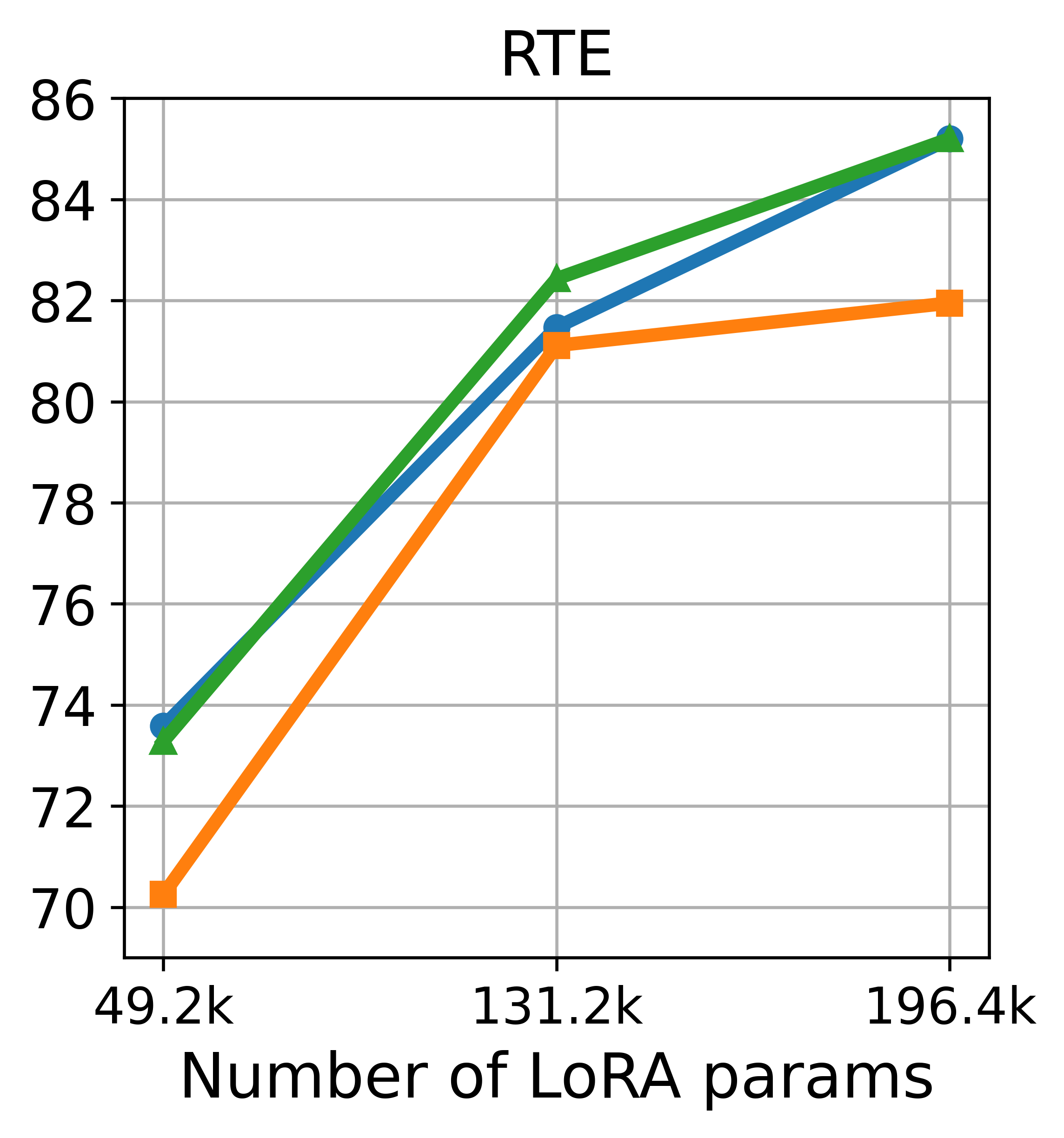}
  \label{fig:sub5}
\end{subfigure}
 \caption{Results with RoBERTa-Large models on GLUE under different fine-tuning parameter budgets, involving three clients with rank 4.}
    \label{fig:five_subfigures}
\hfill
\end{center}
\end{figure}

\begin{figure}[H]
    \begin{minipage}{0.58\textwidth}
        {\scriptsize
    \centering
    \begin{tabular}{ccccc}
    \toprule
         & BoolQ & PIQA & SIQA & HellaSwag \\
         \midrule
        LoRA & $\text{61.42}_{\pm \text{0.29}}$ & $\text{33.19}_{\pm \text{9.8}}$ & $\text{31.88}_{\pm \text{3.95}}$ & $\text{21.23}_{\pm \text{2.82}}$  \\
        FFA-LoRA & $\text{53.43}_{\pm \text{4.3}}$ & $\text{35.49}_{\pm \text{9.55}}$ & $\text{10.63}_{\pm \text{8.44}}$ & $\text{11.81}_{\pm \text{4.53}}$ \\
        \rowcolor{ours}{RoLoRA}  &  $\text{\textbf{61.83}}_{\pm \text{0.22}}$ & $\text{\textbf{61.26}}_{\pm \text{3.3}}$ & $\text{\textbf{39.76}}_{\pm \text{0.41}}$ & $\text{\textbf{27.49}}_{\pm \text{2.34}}$ \\
        \midrule
        & WinoGrande & ARC-e & ARC-c & OBQA \\
        \midrule
         LoRA &  $\text{31.36}_{\pm \text{5.02}}$ & $\text{27.36}_{\pm \text{0.89}}$ & $\text{32.03}_{\pm \text{1.14}}$  &$\text{26.07}_{\pm \text{2.32}}$ \\
        FFA-LoRA &  $\text{1.61}_{\pm \text{2.14}}$ & $\text{6.88}_{\pm \text{0.42}}$ & $\text{7.93}_{\pm \text{0.89}}$ & $\text{15.0}_{\pm \text{5.41}}$\\
        \rowcolor{ours}{RoLoRA}  &  $\text{\textbf{47.67}}_{\pm \text{0.75}}$ & $\text{\textbf{33.19}}_{\pm \text{1.29}}$ & $\text{\textbf{40.13}}_{\pm \text{1.73}}$ & $\text{\textbf{31.67}}_{\pm \text{1.4}}$\\
        \bottomrule
    \end{tabular}
    \captionof{table}{Results with Llama-2-7B models on commonsense reasoning tasks. This involves 50 clients using rank 8.}
    \label{tab:reasoning}}
    \end{minipage}
    \hfill
      \centering
    \begin{minipage}{0.36\textwidth}
        \centering
        \includegraphics[width=0.67\textwidth]{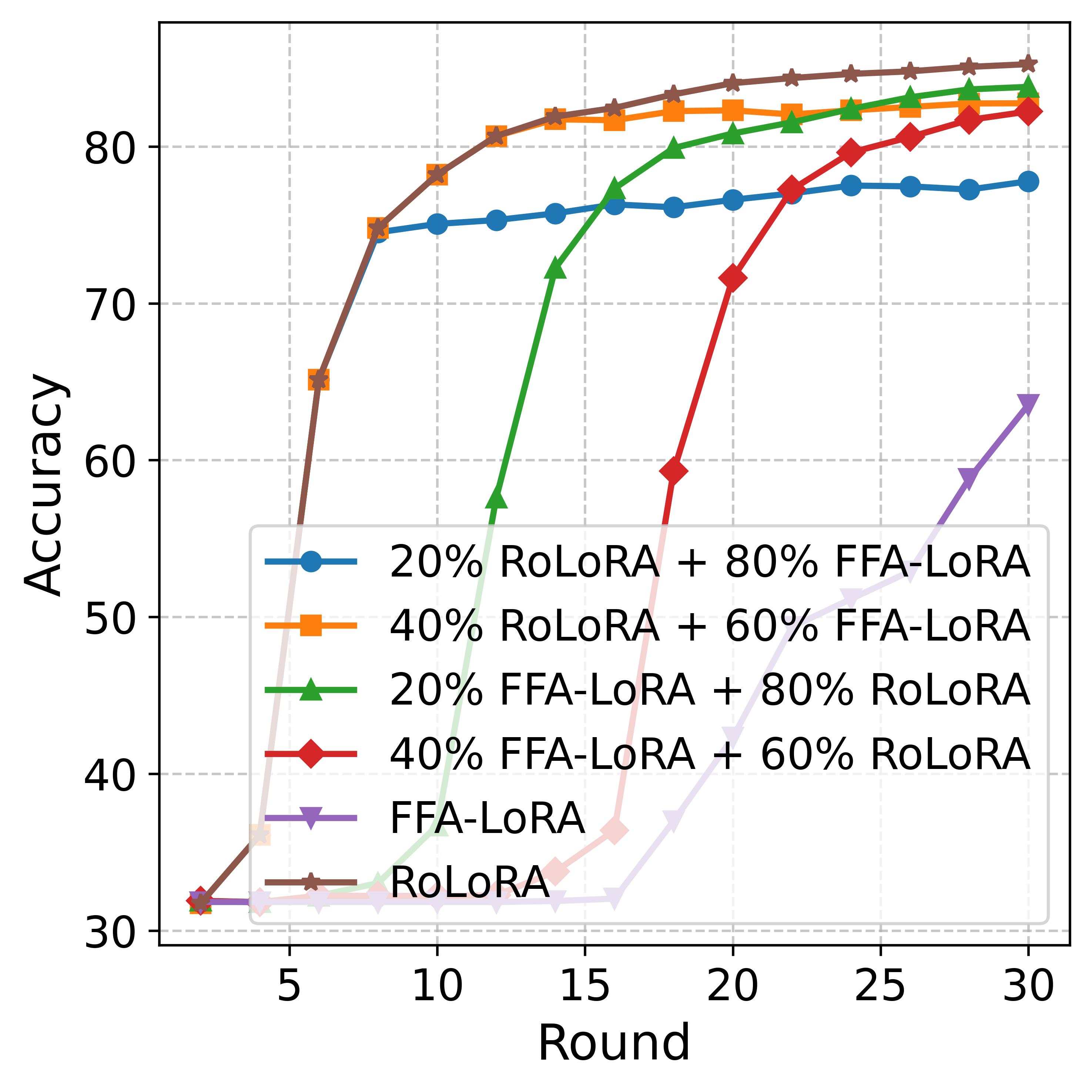}
        \caption{Ablation study on learning vs. freezing 
A on MNLI task.}
        \label{fig:ablation}
    \end{minipage}
    
\end{figure}

\vspace{-0.2cm}
\subsection{Commonsense Reasoning Tasks}


\paragraph{Results.}
We evaluate RoLoRA against FFA-LoRA and LoRA on Llama-2-7B\cite{touvron2023llama2openfoundation} for commonsense reasoning tasks. In Table~\ref{tab:reasoning}, we compare the results of three methods with Llama-2-7B models on 8 commonsense reasoning tasks. The configurations are presented in Appendix~\ref{sec:setup-CRT}. The performance is reported as the mean accuracy with standard deviations across 3 trials. RoLoRA consistently achieves the highest accuracy across all tasks, demonstrating significant improvements over both LoRA and FFA-LoRA. We also highlights that FFA-LoRA exhibits large performance variances across trials, such as a standard deviation of 9.55 for PIQA and 8.44 for SIQA, respectively. This significant variability is likely due to the initialization quality of parameter $\mathbf{A}$, as different initializations could lead to varying optimization trajectories and final performance outcomes as discussed in Section~\ref{sec:analysis}. Additional results on this task are presented in Table~\ref{tab:reasoning-rank2-appendix} in Appendix~\ref{sec:setup-CRT}.

\subsection{Ablation Study}

\paragraph{Learning vs. Freezing A}
we conducted an experiment comparing performance of FFA-LoRA, RoLoRA, and different mixing strategies under the setting with 50 clients. In these strategies, for example, 20\%RoLoRA+80\%FFA-LoRA means we finetune with RoLoRA (where A is learned) for the first 20\% of communication rounds, followed by FFA-LoRA (where A is frozen) for the remaining 80\%. The results are shown in the Figure~\ref{fig:ablation}. We observe that finetuning with RoLoRA generally leads to faster convergence and higher final accuracy. This highlights the benefits of learning A, especially in early training.
\begin{wrapfigure}[15]{R}{0.4\textwidth}
\begin{center}
\includegraphics[width=0.4\textwidth]{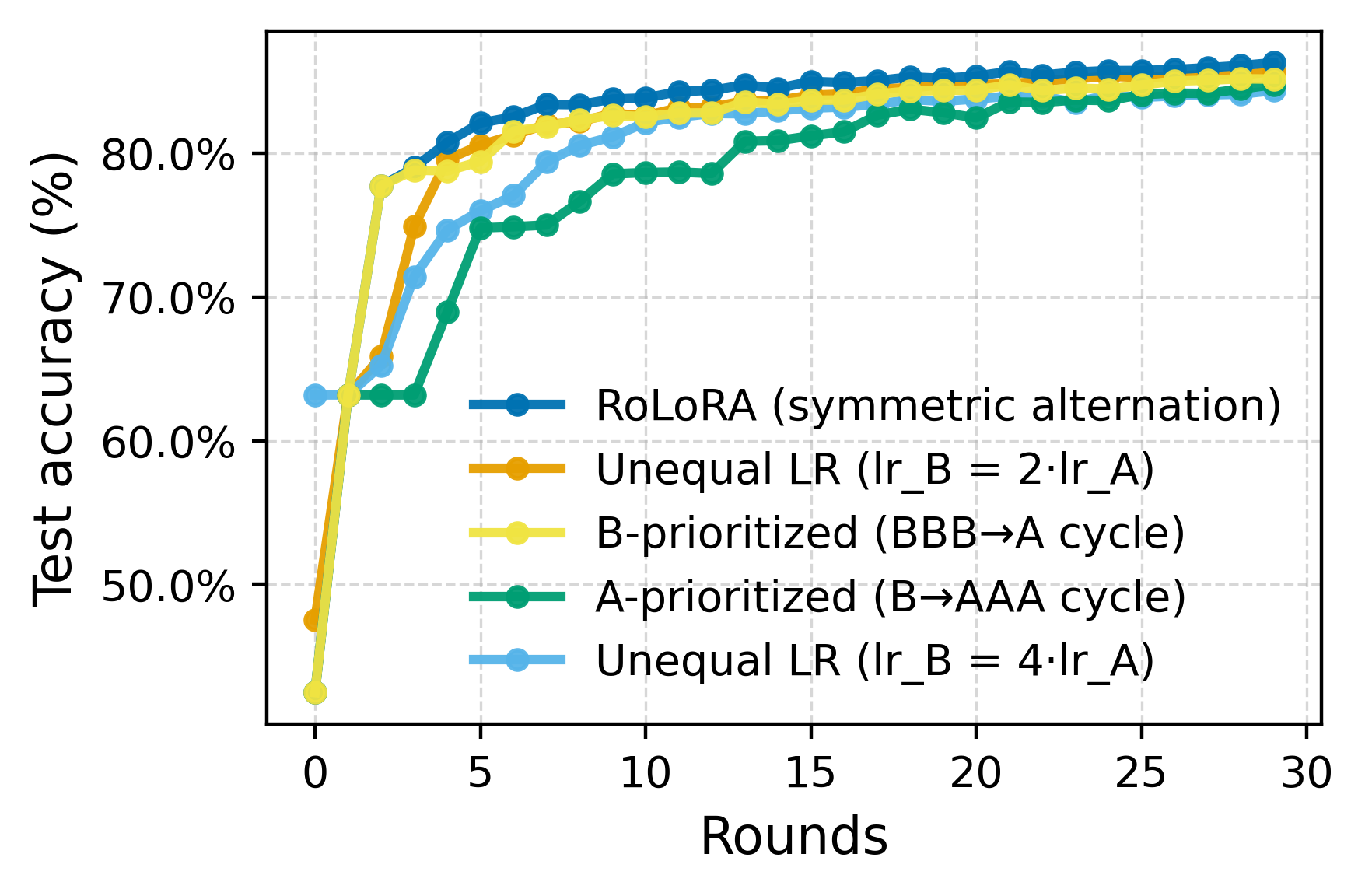}
 \caption{Asymmetry vs. Symmetry in LoRA updates. Accuracy vs. round.}
 \label{fig:asy-sy}
 \end{center}
 \end{wrapfigure}
\paragraph{Symmetry vs. Asymmetry Update}
In standard LoRA implementations, LoRA-$A$ is randomly initialized while LoRA-$B$ is set to zero, which implicitly assigns asymmetric roles. In our study, we view LoRA-$A$ as a learnable basis and LoRA-$B$ as coefficients on that basis. Motivated by this, we investigated whether an asymmetric update policy might be preferred. We systematically compare four strategies: (i) the symmetric alternation used in standard RoLoRA, (ii) B-prioritized multi-step updates (B, B, B, A, $\dots$), (iii) A-prioritized multi-step updates (B, A, A, A, $\dots$), and (iv) unequal learning rates (e.g., setting $lr_{B}=2lr_{A}$ or $lr_{B}=4lr_{A}$). As shown in Figure~\ref{fig:asy-sy}, balanced AB alternation yields the highest accuracy and the most stable trajectory, while aggressively prioritizing either A or B degrades performance.

\paragraph{Effect of Local Steps on RoLoRA and FFA-LoRA} To evaluate the effect of local steps, we have conducted an ablation study on the number of local steps in a 50-client setting with rank-4 adapters, as shown in Table~\ref{tab:ablation-local-steps}. For a fair comparison, we kept the total computational budget ($\#$Local Steps $\times$ $\#$Total Rounds) constant across all settings. FFA-LoRA's performance consistently degrades on both datasets as the number of local steps increases. This indicates that with more local work per round, FFA-LoRA suffers severely from client drift, where local models overfit to their own data. RoLoRA maintains high performance across all settings. 

\begin{table}[h]
    \centering
    {\scriptsize
    \begin{tabular}{cccccc}
    \toprule
         & (Local Steps, Total Rounds)  & (1,600)  & (5,120) & (10,60) & (20,30)\\
         \midrule
        \multirow{2}{*}{MNLI} & FFA-LoRA & 72.52 \com 0.68 & 71.73 \com 1.17  & 69.64 \com 4.31 & 69.97 \com 5.57 \\
         &  \cellcolor{ours} RoLoRA  & \cellcolor{ours}84.39 \com 0.34 &\cellcolor{ours} 84.96 \com 0.18 & \cellcolor{ours}84.79  \com 0.23& \cellcolor{ours} 82.98 \com 3.36\\
         \midrule
       \multirow{2}{*}{QQP}  &  FFA-LoRA & 80.51 \com 1.38  & 80.2 \com 1.65 & 79.07 \com 1.21 & 78.44 \com 0.41\\
         &\cellcolor{ours} RoLoRA & \cellcolor{ours}85.24 \com 0.56 &\cellcolor{ours} 85.44 \com 0.8 &\cellcolor{ours} 84.77 \com 0.77 &\cellcolor{ours} 85.71 \com 0.18\\
         \bottomrule
    \end{tabular}}
    \vspace{4pt}
    \caption{Results on RoBERTa-Large on MNLI and QQP with different local steps while keeping total computational budget constant.}
    \label{tab:ablation-local-steps}
\end{table}
\vspace{-3mm}
\subsection{More Results}We provide additional experimental results in the Appendix, including: (1) evaluations of Llama-2-7B on HumanEval and MMLU tasks (Appendix~\ref{sec:setup-LGT}); (2) comparisons of communication and time costs in Table~\ref{tab:comm-time-cost-comp}; and (3) evaluations under a fixed communication budget in Figure~\ref{fig:align-comm-cost}.

\vspace{-3mm}
\section{Conclusion}
\vspace{-3mm}
In this work, we introduced RoLoRA, a federated finetuning framework based on alternating optimization for LoRA adapters. Our approach addresses key limitations of prior methods by jointly learning both the down-projection and up-projection matrices, thereby enhancing the expressiveness and robustness of the adapted models. Through a combination of theoretical analysis on linear models and validation on nonlinear models, we established the importance of optimizing both components in LoRA. Extensive experimental evaluations across various language models tasks, and diverse federated learning settings confirmed that RoLoRA consistently outperforms existing baselines, particularly in large-scale scenarios under constrained communication and resource conditions.

\newpage
\begin{ack}
Resources used in preparing this research were provided, in part, by the Province of Ontario, the Government of Canada through CIFAR, and companies sponsoring the Vector Institute www.vectorinstitute.ai/partnerships/.
\end{ack}





\bibliography{neurips_2025}
\bibliographystyle{plain}

\newpage
\appendix

\renewcommand{\thesection}{A\arabic{section}}
\section*{Appendix}
\addcontentsline{toc}{section}{Appendix}
\appendixtoc

\newpage
\section{Algorithms}
\begin{algorithm}[H]
   \caption{RoLoRA iterations}
   \label{alg:rolora-llm}
\begin{algorithmic}[1]
   \STATE {\bfseries Input:} number of iterations $T$, number of clients $N$
   \FOR{$t=1$ {\bfseries to} $T$}
   \FOR{$i=1$ {\bfseries to} $N$}
   \STATE Fix $\mathbf{A}^{t}$, $\mathbf{B}_i^{t+1} = \text{GD-update}(\mathbf{A}^{t},\mathbf{B}^{t})$
   \STATE Transmits $\mathbf{B}_i^{t+1}$ to server
   \ENDFOR
   \STATE Server aggregates $\mathbf{B}^{t+1}=\frac{1}{N}\sum_{i=1}^{N}\mathbf{B}_i^{t+1}$, broadcasts $\mathbf{B}^{t+1}$
   \FOR{$i=1$ {\bfseries to} $N$}
   \STATE Fix $\mathbf{B}^{t+1}$, $\mathbf{A}_i^{t+1} = \text{GD-update}(\mathbf{A}^{t},\mathbf{B}^{t+1})$
   \STATE Transmits $\mathbf{A}_i^{t+1}$ to server
   \ENDFOR
   \STATE Server aggregates $\mathbf{A}^{t+1}=\frac{1}{N}\sum_{i=1}^{N}\mathbf{A}_i^{t+1}$, broadcasts $\mathbf{A}^{t+1}$
   \ENDFOR
\end{algorithmic}
\end{algorithm}

\begin{algorithm}[H]
   \caption{RoLoRA for linear regressor, Alt-min-GD iterations}
   \label{alg:rolora-linear}
\begin{algorithmic}[1]
   \STATE {\bfseries Input:} GD Step size $\eta$, number of iterations $T$, number of clients $N$
   \FOR{$t=1$ {\bfseries to} $T$}
   \STATE Let $\mathbf{a} \leftarrow$ $\mathbf{a}^{t-1}, \mathbf{b} \leftarrow$ $\mathbf{b}^{t-1}$.
   \FOR{$i=1$ {\bfseries to} $N$}
   \STATE  set $\tilde{\mathbf{b}}_i \leftarrow \argmin_{\mathbf{b}} l_i(\mathbf{a},\mathbf{b})$
   \ENDFOR
   \STATE $\Bar{\mathbf{b}}=\frac{1}{N}\sum_{i=1}^{N}\tilde{\mathbf{b}}_i$
   \FOR{$i=1$ {\bfseries to} $N$}
   \STATE Compute $\nabla_{\mathbf{a}} l_i(\mathbf{a},\Bar{\mathbf{b}})$
   \ENDFOR
   \STATE $\hat{\mathbf{a}}^{+} \leftarrow \mathbf{a}-\frac{\eta}{N}\sum_{i=1}^{N}\nabla_{\mathbf{a}} l_i(\mathbf{a},\Bar{\mathbf{b}}),~ \hat{\mathbf{a}} \leftarrow \frac{\hat{\mathbf{a}}^{+}}{\lVert \hat{\mathbf{a}}^{+}\rVert}$
   \STATE $\mathbf{a}^t \leftarrow \hat{\mathbf{a}},~ \mathbf{b}^t\leftarrow  \Bar{\mathbf{b}}$
   \ENDFOR
\end{algorithmic}
\end{algorithm}

\section{Discussion} \label{app:discussion}
{
While Algorithm~\ref{alg:rolora-linear} is conceptually inspired by alternating optimization techniques in matrix sensing and multi-task linear representation learning (MLRL), but it differs in the algorithmic design, and application focus.

\paragraph{Connection to Matrix Sensing.} The problem in Eq.~\eqref{obj-xi} is an instance of matrix sensing. Traditional matrix sensing methods \cite{10.1145/2488608.2488693} focus on recovering low-rank structures from centralized data, whereas RoLoRA is designed for a federated setting, where both data and computation are decentralized across multiple clients. A related line of work is federated matrix factorization, which also applies alternating minimization techniques in distributed environments.  Wang et al.~\cite{wang2020federatedmatrixfactorizationalgorithm} perform alternating minimization between local matrix factors within each communication round but do not alternate the aggregation steps. In contrast, RoLoRA alternates both the updates and aggregations of down- and up-projection matrices across rounds, fundamentally changing the communication pattern and mitigating interference between matrix components during aggregation.

\paragraph{Connection to MLRL.} As discussed in Section~\ref{intro}, we connect the objective of learning down-projection matrices in a federated setting to multitask
linear representation learning (MLRL)\cite{pmlr-v139-collins21a,thekumparampil2021sampleefficientlinearmetalearning,Vaswani_2024,du2021fewshot}. We follow a similar derivation framework to MLRL works that employ alternating optimization, such as FedRep\cite{pmlr-v139-collins21a} and Vaswani et al.~\cite{Vaswani_2024}. However, their focus is on a multi-task setting, where the heads are kept diverse and not aggregated. As a result, they perform alternating minimization and gradient descent between the local representation and head within each communication round but do not perform the aggregation steps for the local heads. Furthermore, there are fundamental differences in model structure: in FedRep, the local head acts as a down-projection, whereas in RoLoRA, the corresponding component $\mathbf{B}$ serves as an up-projection. This distinction stems from RoLoRA’s foundation in the LoRA framework, where adaptation uses a down-projection followed by an up-projection, with $\mathbf{B}$ mapping back to the original feature space.

\paragraph{Connection to Personalized FL.} While RoLoRA is positioned within the global model paradigm, it is worth noting that certain algorithmic structures share interesting similarities. For instance, Mishchenko et al. \cite{mishchenko2025partially} and Pillutla et al. \cite{pmlr-v162-pillutla22a} fall under the category of personalized federated learning, but their update mechanism bears resemblance to RoLoRA's alternating optimization. In \cite{mishchenko2025partially}, clients compute the gradient of their local loss with respect to the global parameters, holding the personalized parameters fixed, and send only this gradient to the server for aggregation. 
Pillutla et al.~\cite{pmlr-v162-pillutla22a} adopt an alternating update strategy in the proposed FedAlt algorithm, where clients update personal parameters while holding global parameters fixed, followed by updating the global parameters. Both resemble the alternating structure in RoLoRA. However, RoLoRA operates entirely under a global model setting without personalized components. Furthermore, our alternating scheme is motivated by the decomposition of low-rank adaptation matrices, separating the optimization and also the aggregation of up- and down-projection matrices. Moreover, the underlying proof techniques differ entirely: they employ standard FL convergence analysis, whereas our approach draws on techniques from matrix sensing to highlight the importance of learning the down-projection.

\paragraph{Limitations and Future Works.} While RoLoRA demonstrates robust performance across various federated learning settings, our work has a few limitations. The study of learning down-projections has primarily focused on linear models, which may not fully capture the complexities of highly non-linear language models. While empirical validation has been conducted on non-linear models, a rigorous theoretical proof is still lacking. We leave the theoretical extension to simple non-linear models as important future work, although empirical results suggest the method remains effective. Second, the current analysis assumes full client participation in each communication round, which may not hold in real-world federated deployments with intermittent connectivity or resource constraints. A theoretical guarantee for partial client participation is needed. 
}

\section{Theoretical Analysis on Linear Model} \label{app:theo-linear}
\subsection{Notation}\label{sec:notation}
Table~\ref{tab:notation} provides a summary of the symbols used throughout this theoretical analysis.
\begin{table}[h]
    \centering
    \begin{tabular}{c|c}
    \toprule
        Notation & Description \\
        \midrule
      $\mathbf{a}^*, \mathbf{b}_i^*$   & Ground truth parameters of client $i$\\
       $ \Bar{\mathbf{b}}^*$  & Average of $\mathbf{b}_i^*$\\
       $\mathbf{a}^t, \mathbf{b}^t$  & Global model parameters of $t$-th iteration \\
        $\delta_t$ & The angle distance between $\mathbf{a}^*$ and $\mathbf{a}^t$, $|\sin \theta (\mathbf{a}^*,\mathbf{a}^t)|$\\
        $\eta$  & Step size\\
        $\mathbf{I_d}$ & $d \times d$ identity matrix \\
        $\lVert \cdot\rVert$ & $l_2$ norm of a vector\\
        $\|\cdot\|_{op}$ & Operator norm ($l_2$ norm) of a matrix\\
        $|\cdot|$ & Absolute value of a scalar\\
        $\|\cdot\|_{\psi_2}$ & Sub-Gaussian norm of a sub-Gaussian random variable\\   
        $\|\cdot\|_{\psi_1}$ & Sub-exponential norm of a sub-exponential random variable\\   
        $N $ & Total number of clients \\
        $\hat{\mathbf{a}}^{+}$ & Updated $\mathbf{a}$ by gradient descent\\
        $\hat{\mathbf{a}}$ & Normalized  $\hat{\mathbf{a}}^{+}$\\
        $\tilde{\mathbf{b}}_i $ & analytic solution for $\mathbf{b}$ in the local objective function \\
        $\Bar{\mathbf{b}} $  & Average of $\tilde{\mathbf{b}}_i$\\
        \bottomrule
    \end{tabular}
    \vspace{4pt}
    \caption{Notations}
    \label{tab:notation}
\end{table}

\subsection{Auxiliary}

\begin{definition}[Sub-Gaussian Norm]
    The sub-Gaussian norm of a random variable $\xi $, denoted as $ \|\xi\|_{\psi_2} $, is defined as:
\[
\|\xi\|_{\psi_2} = \inf \{ t > 0 : \mathbb{E}[\exp(\xi^2 / t^2)] \leq 2 \}.
\]
A random variable is said to be \textit{sub-Gaussian} if its sub-Gaussian norm is finite. Gaussian random variables are sub-Gaussian. The sub-Gaussian norm of a standard Gaussian random variable $\xi \sim \mathcal N(0,1)$ is $\| \xi \|_{\psi_2} = \sqrt{8/3}$.
\end{definition}

\begin{definition}[Sub-Exponential Norm]
    The sub-exponential norm of a random variable \( \xi \), denoted as \( \|\xi\|_{\psi_1} \), is defined as:
\[
\|\xi\|_{\psi_1} = \inf \{ t > 0 : \mathbb{E}[\exp(|\xi| / t)] \leq 2 \}.
\]
\end{definition}

A random variable is said to be \textit{sub-exponential} if its sub-exponential norm is finite.

\begin{lemma}[The product of sub-Gaussians is sub-exponential] \label{lemma:sub-exp-sub-gau}
    Let $\xi$ and $\upsilon$ be sub-Gaussian random variables. Then $\xi \upsilon$ is sub-exponential. Moreover, 
    \[
    \| \xi \upsilon \|_{\psi_1} \leq \| \xi \|_{\psi_2} \cdot \| \upsilon \|_{\psi_2}
    \]
\end{lemma}

\begin{lemma}[Sum of independent sub-Gaussians] \label{lemma:sum-indep-sub-gau}
    Let $X_1,\cdots, X_N$ be independent mean-zero sub-Gaussian random variables. Then $\sum_{i=1}^N X_i$ is also sub-Gaussian with 
    \[
    \left\| \sum_{i=1}^N X_i \right\|^2_{\psi_2} \le C \sum_{i=1}^N \|X_i\|^2_{\psi_2} ,
    \]
    where $C$ is some absolute constant. 
\end{lemma}
\begin{proof}
    See proof of Lemma 2.6.1 of \cite{vershynin2018high}.
\end{proof}

\begin{corollary}\label{cor:inn-prod-sub-gau}
    For random vector $\mathbf x\in \mathbb R^d$ with entries being independent standard Gaussian random variables, the inner product $\mathbf a^\top \mathbf x$ is sub-Gaussian for any fixed $\mathbf a \in \mathbb R^d$, and 
    \[
    \left\| \mathbf a^\top \mathbf x \right\|_{\psi_2} \le C'\|\mathbf a\|
    \]
    where $C'$ is some absolute constant.
\end{corollary}

\begin{proof}
    Note that $\mathbf a^\top \mathbf x = \sum_{i = 1}^d a_i \xi_i$, where $\xi_i \sim \mathcal N(0,1)$ is the $i$-th entry of the random vector $\mathbf x$. Choose $C$ to be the absolute constant specified in Lemma \ref{lemma:sum-indep-sub-gau} for standard Gaussian random variables, and set $C' = \sqrt{8C/3}$. We have
    \[
    \left\| \mathbf a^\top \mathbf x \right\|^2_{\psi_2} \le C \sum_{i=1}^N \|a_i\xi_i\|^2_{\psi_2} \overset{(a)}{=} C \sum_{i=1}^N a_i^2 \|\xi_i\|^2_{\psi_2} \overset{(b)}{=} \frac{8}{3}\cdot C \|\mathbf a\|^2~~\Rightarrow~~ \left\| \mathbf a^\top \mathbf x \right\|_{\psi_2} \le \sqrt{\frac{8C}{3}}\|\mathbf a\| = C'\|\mathbf a\|.
    \]    
    Step (a) makes use of the homogeneity of the sub-Gaussian norm, and step (b) uses the fact that $\|\xi\|_{\psi_2} = \sqrt{8/3}$ for $\xi \sim \mathcal N(0,1)$.
\end{proof}

\begin{definition}[$\epsilon$-net]
Consider a subset $\mathcal A \subseteq \mathbb R^d$ in the $d$-dimensional Euclidean space. Let $\epsilon > 0$. A subset $\mathcal N \subseteq \mathcal A$ is called an $\epsilon$-net of $\mathcal A$ if every point of $\mathcal A$ is within a distance $\epsilon$ of some point in $\mathcal N$, i.e., 
\[
\forall\,\mathbf x \in \mathcal A, ~\exists\, \mathbf x' \in \mathcal N, ~\| \mathbf x-\mathbf x'\| \le \epsilon.
\]
    
\end{definition}

\begin{lemma}[Computing the operator norm on a net] \label{eps-net}
Let $\mathbf{a} \in \mathbb{R}^d$ and $\epsilon \in [0,1)$. Then, for any $\epsilon$-net $\mathcal{N}$ of the sphere  $\mathcal{S}^{d-1}$, we have
\[
    \lVert \mathbf{a} \rVert \leq \frac{1}{1-\epsilon} \sup_{\mathbf{x} \in \mathcal{N}} \langle \mathbf{a}, \mathbf{x}\rangle 
\]
\end{lemma}
\begin{proof}
    See proof of Lemma 4.4.1 of \cite{vershynin2018high}.
\end{proof}

\begin{theorem}[Bernstein's inequality] \label{berstein}
Let $X_1, \dots, X_N$ be independent mean-zero sub-exponential random variables. Then, for every $t \geq 0$, we have
\[
\mathbb{P} \left( \left| \sum_{i=1}^N X_i \right| \geq t \right) \leq 2 \exp \left( -c \min \left( \frac{t^2}{\sum_{i=1}^N \| X_i \|_{\psi_1}^2}, \frac{t}{\max_i \| X_i \|_{\psi_1}} \right) \right),
\]
where $c > 0$ is an absolute constant.
\end{theorem}
\begin{proof}
    See proof of Theorem 2.8.1 of \cite{vershynin2018high}.
\end{proof}

\subsection{Homogeneous Case} \label{sec:homo-analysis}
\paragraph{Outline of Proof.}
In this section, we first analyze the minimization step for updating ${\mathbf{b}_i^t}$ (Lemma~\ref{lemma:b-bar-g-bound}), then establish a bound on the deviation of the gradient from its expectation with respect to $\mathbf{a}$ (Lemma~\ref{lemma:grad-egrad-bound}), and finally derive a bound on $|\sin \theta (\mathbf{a}^{t+1}, \mathbf{a}^*)|$ based on the gradient descent update rule for $\mathbf{a}$ (Lemma~\ref{lemma:delta-t} or Lemma~\ref{lem:lemma-43-appendix}). The proof of Theorem~\ref{convergence-1} is provided in Section~\ref{subsec:mainProof}, where the result is obtained by recursively applying Lemma~\ref{lemma:delta-t} over $T$ iterations.

\begin{lemma} \label{lemma:b-bar-g-bound}
    Let $\mathbf{a} = \mathbf{a}^t$. Let $\delta^t = \lVert (\mathbf{I}_d-\mathbf{a}^*\mathbf{a}^{*^\top})\mathbf{a}\rVert = \lVert (\mathbf{I}_d-\mathbf{a}\mathbf{a}^\top)\mathbf{a}^*\rVert$ denote the angle distance between $\mathbf{a}^{*}$ and $\mathbf{a}$. Let $\mathbf{g}^\top = \mathbf{a}^\top\mathbf{a}^* \mathbf{b}^{*^\top}, \Bar{\mathbf{b}} = \frac{1}{N}\sum_{i=1}^N{\mathbf{b}}_i$, If $m=\Omega(q)$, and $q=\max(\frac{\log(N)}{[\min(\epsilon_1,\epsilon_2)]^2},\frac{d\log(\frac{2}{\epsilon_0})}{\epsilon_2^2})$, then with probability $1-q^{-10}$,
    \begin{align}\label{eq:b_bar_g_diff}
        \lVert \Bar{\mathbf{b}}-\mathbf{g}\rVert \le \epsilon' \delta^t \lVert \mathbf{b}^* \rVert
    \end{align}
    where $\epsilon' = \frac{\epsilon_2}{(1-\epsilon_0)(1-\epsilon_{1})}$, for $\epsilon_0,\epsilon_1,\epsilon_2 \in (0,1)$.
\end{lemma}
\begin{proof}
    
 We drop superscript $t$ for simplicity. Following Algorithm~\ref{alg:rolora-linear}, we start by computing the update for $\Tilde{\mathbf{b}}_i$.  With $\mathbf{g}^\top = \mathbf{a}^\top \mathbf{a}^* \mathbf{b}^{*^\top}$, we get:
\begin{align}
    \Tilde{\mathbf{b}}_i^\top &= \frac{\mathbf{a}^\top \mathbf{X}_i^\top \mathbf{X}_i \mathbf{a}^* \mathbf{b}^{*^\top} }{\mathbf{a}^\top \mathbf{X}_i^\top\mathbf{X}_i \mathbf{a}} \\
    & = \frac{\mathbf{a}^\top \mathbf{X}_i^\top \mathbf{X}_i \mathbf{a}\mathbf{a}^\top \mathbf{a}^* \mathbf{b}^{*^\top} + \mathbf{a}^\top \mathbf{X}_i^\top \mathbf{X}_i (\mathbf{I}_d - \mathbf{a}\mathbf{a}^\top) \mathbf{a}^* \mathbf{b}^{*^\top}}{\mathbf{a}^\top \mathbf{X}_i^\top\mathbf{X}_i \mathbf{a}} \\
    & = \mathbf{g}^\top + \frac{ \mathbf{a}^\top \mathbf{X}_i^\top \mathbf{X}_i (\mathbf{I}_d - \mathbf{a}\mathbf{a}^\top) \mathbf{a}^* \mathbf{b}^{*^\top}}{\mathbf{a}^\top \mathbf{X}_i^\top\mathbf{X}_i \mathbf{a}}.
\end{align}
Therefore, 
\begin{equation}\label{difference_bi_g}
    \|\Tilde{\mathbf{b}}_i - \mathbf{g}\|  \leq |\mathbf{a}^\top \mathbf{X}_i^\top\mathbf{X}_i \mathbf{a}|^{-1}\cdot \lVert \mathbf{a}^\top \mathbf{X}_i^\top \mathbf{X}_i (\mathbf{I}_d - \mathbf{a}\mathbf{a}^\top)\mathbf{a}^*\mathbf{b}^{*^\top}\rVert = \|\mathbf{X}_i \mathbf{a}\|^{-2}\cdot \lVert \mathbf{a}^\top \mathbf{X}_i^\top \mathbf{X}_i (\mathbf{I}_d - \mathbf{a}\mathbf{a}^\top)\mathbf{a}^*\mathbf{b}^{*^\top}\rVert.
\end{equation}
Note that since each entry of $\mathbf{X}_i$ is independent and identically distributed according to a standard Gaussian, and $\lVert \mathbf{a}\rVert = 1$, $\mathbf{X}_i \mathbf{a} $ is a random standard Gaussian vector. By Theorem 3.1.1 of \cite{vershynin2018high}, the following is true for any $\epsilon_1 \in (0,1)$
\begin{align}
    \mathbb{P} \left\{ \lVert \mathbf{X}_i \mathbf{a}\rVert^2 \leq (1-\epsilon_1)m\right\} \leq \exp{(- \frac{c_1  \epsilon_1^2 m}{K^4})} \label{denominator-bound}
\end{align}
where $K = \|\xi\|_{\psi_2} \ge 1$ for $\xi \sim \mathcal N(0,1)$ and $c_1$ is some large absolute constant that makes \eqref{denominator-bound} holds. Next we upper bound $\rVert \mathbf{a}^\top \mathbf{X}_i^\top \mathbf{X}_i (\mathbf{I}_d - \mathbf{a}\mathbf{a}^\top)\mathbf{a}^*\mathbf{b}^{*^\top}\lVert$. Note that $\mathbb{E}[\mathbf{a}^\top \mathbf{X}_i^\top \mathbf{X}_i (\mathbf{I}_d - \mathbf{a}\mathbf{a}^\top)\mathbf{a}^*\mathbf{b}^{*^\top}] = \mathbf{a}^\top \mathbb{E}[\mathbf{X}_i^\top \mathbf{X}_i] (\mathbf{I}_d - \mathbf{a}\mathbf{a}^\top)\mathbf{a}^*\mathbf{b}^{*^\top}=m\mathbf{a}^\top(\mathbf{I}_d - \mathbf{a}\mathbf{a}^\top)\mathbf{a}^*\mathbf{b}^{*^\top}=0$. First we need to apply sub-exponential Berstein inequality to bound the deviation from this mean, and then apply epsilon net argument. Let $\mathcal N$ be any $\epsilon_0$-net of the unit sphere $\mathcal S^{d-1}$ in the $d$-dimensional real Euclidean space, then by Lemma~\ref{eps-net}, we have
\begin{align}
    \lVert \mathbf{a}^\top \mathbf{X}_i^\top \mathbf{X}_i (\mathbf{I}_d - \mathbf{a}\mathbf{a}^\top)\mathbf{a}^*\mathbf{b}^{*^\top}\rVert &\le \frac{1}{1-\epsilon_0} \max_{\mathbf{w}\in \mathcal N}  \mathbf{a}^\top \mathbf{X}_i^\top \mathbf{X}_i (\mathbf{I}_d - \mathbf{a}\mathbf{a}^\top)\mathbf{a}^*\mathbf{b}^{*^\top}\mathbf{w}\\
    &\le \frac{1}{1-\epsilon_0} \max_{\mathbf{w}\in \mathcal N}  |\mathbf{a}^\top \mathbf{X}_i^\top \mathbf{X}_i (\mathbf{I}_d - \mathbf{a}\mathbf{a}^\top)\mathbf{a}^*\mathbf{b}^{*^\top}\mathbf{w}| \label{eps-numerator}
\end{align}

Meanwhile, denote the $j$-th row of $\mathbf{X}_{i}$ by $\mathbf{x}_{i,j}^\top$, for every $\mathbf w \in \mathcal N$, we have
\begin{align}
    \mathbf{a}^\top \mathbf{X}_i^\top \mathbf{X}_i (\mathbf{I}_d - \mathbf{a}\mathbf{a}^\top)\mathbf{a}^*\mathbf{b}^{*^\top}\mathbf{w} = \sum_{j=1}^{m} (\mathbf{a}^\top \mathbf{x}_{i,j})(\mathbf{x}_{i,j}^\top(\mathbf{I}_d - \mathbf{a}\mathbf{a}^\top)\mathbf{a}^*\mathbf{b}^{*^\top}\mathbf{w}) \label{numerator}
\end{align}
 On the right hand side of \eqref{numerator}, $\mathbf{a}^\top \mathbf{x}_{i,j}$ and $\mathbf{x}_{i,j}^\top(\mathbf{I}_d - \mathbf{a}\mathbf{a}^\top)\mathbf{a}^*\mathbf{b}^{*^\top}\mathbf{w}$ are sub-Gaussian random variables. Thus, the summands on the right-hand side of \eqref{numerator} are products of sub-Gaussian random variables, making them sub-exponential. Now by choosing $c_2 = (C')^2$ for the $C'$ in Corollary \ref{cor:inn-prod-sub-gau}, we have the following chain of inequalities for all $i, j$:
\begin{align}
\|(\mathbf{a}^\top \mathbf{x}_{i,j})(\mathbf{x}_{i,j}^\top(\mathbf{I}_d - \mathbf{a}\mathbf{a}^\top)\mathbf{a}^*\mathbf{b}^{*^\top}\mathbf{w})\|_{\psi_1} &\le \|\mathbf{a}^\top \mathbf{x}_{i,j} \|_{\psi_2} \cdot \| \mathbf{x}_{i,j}^\top (\mathbf{I}_d - \mathbf{a}\mathbf{a}^\top)\mathbf{a}^*\mathbf b^{*^\top} \mathbf w \|_{\psi_2}\label{eq:bd_psi1_1}\\
&\le c_2\cdot \| \mathbf{a} \|  \cdot \| (\mathbf I_d - \mathbf{aa}^\top) \mathbf a^*  {\mathbf{b}^{*}}^\top \mathbf w \| \label{eq:bd_psi1_2}\\
& \le  c_2\cdot \| \mathbf{a} \|  \cdot \| (\mathbf I_d - \mathbf{aa}^\top) \mathbf a^*  {\mathbf{b}^{*}}^\top \|_{op} \|\mathbf w \|\label{eq:bd_psi1_3}\\
& \le  c_2  \cdot \| (\mathbf I_d - \mathbf{aa}^\top) \mathbf a^*  {\mathbf{b}^{*}}^\top \|_{op} \label{eq:bd_psi1_4}
\end{align}
Equation \eqref{eq:bd_psi1_1} is due to Lemma \ref{lemma:sub-exp-sub-gau}, \eqref{eq:bd_psi1_2} is due to Corollary \ref{cor:inn-prod-sub-gau}, \eqref{eq:bd_psi1_4} is by the fact that $\|\mathbf a\| = \|\mathbf w\| = 1$.

Furthermore, these summands are mutually independent and have zero mean. By applying sub-exponential Bernstein's inequality (Theorem~\ref{berstein}) with $t=\epsilon_2 m \lVert (\mathbf{I}_d - \mathbf{a}\mathbf{a}^\top) \mathbf{a}^*\mathbf{b}^{*^\top}\rVert_{op}$, we get
\begin{align}
    &\mathbb{P} \left\{ \left| \mathbf{a}^\top \mathbf{X}_i^\top \mathbf{X}_i (\mathbf{I}_d - \mathbf{a}\mathbf{a}^\top)\mathbf{a}^*\mathbf{b}^{*^\top}\mathbf{w} \right|  \geq \epsilon_2 m \lVert (\mathbf{I}_d - \mathbf{a}\mathbf{a}^\top ) \mathbf{a}^*\mathbf{b}^{*^\top}\rVert_{op}\right\} \\&= \mathbb{P} \left\{ \left|\sum_{j=1}^{m} (\mathbf{a}^\top \mathbf{x}_{i,j})(\mathbf{x}_{i,j}^\top(\mathbf{I}_d - \mathbf{a}\mathbf{a}^\top)\mathbf{a}^*\mathbf{b}^{*^\top}\mathbf{w})\right|  \geq \epsilon_2 m \lVert (\mathbf{I}_d - \mathbf{a}\mathbf{a}^\top ) \mathbf{a}^*\mathbf{b}^{*^\top}\rVert_{op}\right\}  \\ 
    & \leq 2\exp\left(-c \min \left(\frac{\epsilon_2^2 m^2 \lVert (\mathbf{I}_d - \mathbf{a}\mathbf{a}^\top ) \mathbf{a}^*\mathbf{b}^{*^\top}\rVert_{op}^2}{\sum_{j=1}^{m} \| (\mathbf{a}^\top \mathbf{x}_{i,j}) (\mathbf{x}_{i,j}^\top(\mathbf{I}_d - \mathbf{a}\mathbf{a}^\top)\mathbf{a}^*\mathbf{b}^{*^\top}\mathbf{w})\|^2_{\psi_1}}  ,  \frac{\epsilon_2 m \lVert (\mathbf{I}_d - \mathbf{a}\mathbf{a}^\top ) \mathbf{a}^*\mathbf{b}^{*^\top}\rVert_{op}}{\max_j \| (\mathbf{a}^\top \mathbf{x}_{i,j}) (\mathbf{x}_{i,j}^\top(\mathbf{I}_d - \mathbf{a}\mathbf{a}^\top)\mathbf{a}^*\mathbf{b}^{*^\top}\mathbf{w})\|_{\psi_1}} \right) \right) \\
    &= 2\exp{(-c_3 \epsilon_2^2 m )}   \label{numerator-bound}
\end{align}
for any fixed $\mathbf w \in \mathcal N$, $\epsilon_2 \in (0,1)$, and some absolute constant $c_3$. \eqref{numerator-bound} follows because 
\begin{align}
    \frac{\epsilon_2^2 m^2 \lVert (\mathbf{I}_d - \mathbf{a}\mathbf{a}^\top ) \mathbf{a}^*\mathbf{b}^{*^\top}\rVert_{op}^2}{\sum_{j=1}^{m} \| (\mathbf{a}^\top \mathbf{x}_{i,j}) (\mathbf{x}_{i,j}^\top(\mathbf{I}_d - \mathbf{a}\mathbf{a}^\top)\mathbf{a}^*\mathbf{b}^{*^\top}\mathbf{w})\|^2_{\psi_1}} \geq \frac{\epsilon_2^2 m}{c_2^2} \\
    \frac{\epsilon_2 m \lVert (\mathbf{I}_d - \mathbf{a}\mathbf{a}^\top ) \mathbf{a}^*\mathbf{b}^{*^\top}\rVert_{op}}{\max_j \| (\mathbf{a}^\top \mathbf{x}_{i,j}) (\mathbf{x}_{i,j}^\top(\mathbf{I}_d - \mathbf{a}\mathbf{a}^\top)\mathbf{a}^*\mathbf{b}^{*^\top}\mathbf{w})\|_{\psi_1}} \geq \frac{\epsilon_2 m }{c_2}
\end{align}
And $\frac{\epsilon_2^2 m}{c_2^2} \leq \frac{\epsilon_2 m }{c_2}$. Now we apply union bound over all elements in $\mathcal N$. By Corollary 4.2.13 in \cite{vershynin2018high}, there exists an $\epsilon_0$-net $\mathcal N$ with $|\mathcal N| \le (\frac{2}{\epsilon_0} +1)^d$, therefore for this choice of $\mathcal N$,                                               
\begin{align}
    &\mathbb{P} \left\{ \max_{\mathbf w \in \mathcal N}\left| \mathbf{a}^\top \mathbf{X}_i^\top \mathbf{X}_i (\mathbf{I}_d - \mathbf{a}\mathbf{a}^\top)\mathbf{a}^*\mathbf{b}^{*^\top}\mathbf{w} \right|  \geq \epsilon_2 m \lVert (\mathbf{I}_d - \mathbf{a}\mathbf{a}^\top ) \mathbf{a}^*\mathbf{b}^{*^\top}\rVert_{op}\right\}\\
    &\leq \sum_{\mathbf w \in \mathcal N}\mathbb{P} \left\{ \left| \mathbf{a}^\top \mathbf{X}_i^\top \mathbf{X}_i (\mathbf{I}_d - \mathbf{a}\mathbf{a}^\top)\mathbf{a}^*\mathbf{b}^{*^\top}\mathbf{w} \right|  \geq \epsilon_2 m \lVert (\mathbf{I}_d - \mathbf{a}\mathbf{a}^\top ) \mathbf{a}^*\mathbf{b}^{*^\top}\rVert_{op}\right\}\\
    &\leq \left(\frac{2}{\epsilon_0} +1\right)^d\cdot 2\exp{(-c_3 \epsilon_2^2 m )}\\
    &= 2\exp{(d\log(1+{\textstyle\frac{2}{\epsilon_0}})-c_3 \epsilon_2^2 m )}\label{numerator-bound-2}
\end{align}

Combining \eqref{difference_bi_g},\eqref{denominator-bound}, \eqref{eps-numerator}, and \eqref{numerator-bound-2}, we get
\begin{align}
     \mathbb{P} \left\{\lVert \Tilde{\mathbf{b}}_i-\mathbf{g}\rVert \le \epsilon' \|(\mathbf{I}_d - \mathbf{a}\mathbf{a}^\top ) \mathbf{a}^*\mathbf{b}^{*^\top}\|_{op} \right\} \geq 1-  p_0\label{difference_bi_g_bound}
\end{align}

where $\epsilon' = \frac{\epsilon_2}{(1-\epsilon_0)(1-\epsilon_{1})}$ and $p_0 = 2\exp{(d\log(1+{\textstyle\frac{2}{\epsilon_0}})-c_3 \epsilon_2^2 m )} + \exp{(- \frac{c_1 \epsilon_1^2 m}{K^4})} $. Using a union bound over $i \in [N]$, we have
\begin{align}
  \mathbb{P} \left\{\bigcap_{i=1}^N \lVert \Tilde{\mathbf{b}}_i- \mathbf{g} \rVert \leq \epsilon' \|(\mathbf{I}_d - \mathbf{a}\mathbf{a}^\top) \mathbf{a}^* \mathbf{b}^{*^\top}\|_{op} \right\} \geq 1 - Np_0.  
\end{align}
Next we bound $\lVert \mathbf{\Bar{b}}-\mathbf{g}\rVert$ where $ \Bar{\mathbf{b}}$ is the average of $\{\mathbf{b}_i\}_{i=1}^N$.
\begin{align}
\lVert \mathbf{\Bar{b}}-\mathbf{g}\rVert & = \lVert \frac{1}{N}\sum_{i=1}^{N}(\Tilde{\mathbf{b}}_i-\mathbf{g})\rVert    \\
&\leq \frac{1}{N}\sum_{i=1}^{N}\lVert\Tilde{\mathbf{b}}_i-\mathbf{g}\rVert \label{b-bar-g-difference-bound-2}\\
&\leq \epsilon' \|(\mathbf{I}_d - \mathbf{a}\mathbf{a}^\top ) \mathbf{a}^*\mathbf{b}^{*^\top}\|_{op} \\
&= \epsilon' \|(\mathbf{I}_d - \mathbf{a}\mathbf{a}^\top ) \mathbf{a}^*\| \cdot \|\mathbf{b}^{*^\top}\|\label{b-bar-g-difference-bound-4} \\
& = \epsilon' \delta^t \lVert \mathbf{b}^*\rVert \label{b-bar-g-difference-bound}
\end{align}
with probability $1-Np_0$. \eqref{b-bar-g-difference-bound-2} follows by Jensen's inequality. \eqref{b-bar-g-difference-bound-4} follows since $\|\mathbf{u}\mathbf{v}^\top\|_{op} = \| \mathbf{u}\| \cdot  \|\mathbf{v} \|$. \eqref{b-bar-g-difference-bound} follows since $\delta^t = \lVert (\mathbf{I}_d-\mathbf{a}\mathbf{a}^{\top})\mathbf{a}^*\rVert$. 

If $m=\Omega(q)$, where $q=\max(\frac{\log(N)}{[\min(\epsilon_1,\epsilon_2)]^2},\frac{d\log(\frac{2}{\epsilon_0})}{\epsilon_2^2})$, then $1-Np_0 >1-\exp(-Cq)> 1-q^{-10}$ for large absolute constant $C$. Then with probability at least $1-q^{-10}$, 
\begin{align}
    \lVert \mathbf{\Bar{b}}-\mathbf{g}\rVert \leq \epsilon' \delta^t \| \mathbf{b}^* \|
\end{align}
\end{proof}

\begin{lemma} \label{lemma:grad-egrad-bound}
     Let $\mathbf{a} = \mathbf{a}^t$. Let $\delta^t = \lVert (\mathbf{I}_d-\mathbf{a}^*\mathbf{a}^{*^\top})\mathbf{a}\rVert = \lVert (\mathbf{I}_d-\mathbf{a}\mathbf{a}^\top)\mathbf{a}^*\rVert$ denote the angle distance between $\mathbf{a}^{*}$ and $\mathbf{a}$. Then for $Nm=\Omega(\frac{d\log(\frac{2}{\epsilon_0})}{\epsilon_3^2})$ and $q=\max(\frac{\log(N)}{[\min(\epsilon_1,\epsilon_2)]^2},\frac{d\log(\frac{2}{\epsilon_0})}{\epsilon_2^2})$, with probability at least $1-2q^{-10}$,
    \begin{align}
        \lVert  \nabla_{\mathbf{a}}l(\mathbf{a},\Bar{\mathbf{b}})-\mathbb{E}[\nabla_{\mathbf{a}}l(\mathbf{a},\Bar{\mathbf{b}})]\rVert \leq 2\Tilde{\epsilon} ((\epsilon')^2+\epsilon')\delta^{t}\lVert \mathbf{b}^*\rVert^2 
    \end{align}
    where $\Tilde{\epsilon} = \frac{\epsilon_3}{1-\epsilon_0}$, and $\epsilon' = \frac{\epsilon_2}{(1-\epsilon_0)(1-\epsilon_{1})}$, for $\epsilon_0,\epsilon_1,\epsilon_2,\epsilon_3 \in (0,1)$.
\end{lemma}
\begin{proof}
Based on the loss function $l(\mathbf{a},\mathbf{b}) = \frac{1}{N} \sum_{i=1}^{N} l_i(\mathbf{a},\mathbf{b}) = \frac{1}{Nm} \sum_{i=1}^{N} \lVert  \mathbf{X}_i\mathbf{a}^*\mathbf{b}^{*^\top}- \mathbf{X}_i\mathbf{a}\mathbf{b}^\top\rVert^2$, we bound the expected gradient with respect to $\mathbf{a}$ and the deviation from it. The gradient with respect to $\mathbf{a}$ and its expectation are computed as:
\begin{align}
\nabla_{\mathbf{a}}l(\mathbf{a},\Bar{\mathbf{b}}) &= \frac{2}{Nm}\sum_{i=1}^{N}( \mathbf{X}_i^\top\mathbf{X}_i\mathbf{a}\Bar{\mathbf{b}}^\top \Bar{\mathbf{b}}- \mathbf{X}_i^\top\mathbf{Y}_i\Bar{\mathbf{b}} )\\
& =  \frac{2}{Nm}\sum_{i=1}^{N}(\mathbf{X}_i^\top\mathbf{X}_i\mathbf{a}\Bar{\mathbf{b}}^\top \Bar{\mathbf{b}}-\mathbf{X}_i^\top\mathbf{X}_i\mathbf{a}^* \mathbf{b}^{*^\top}\Bar{\mathbf{b}})\\
& = \frac{2}{Nm}\sum_{i=1}^{N} \mathbf{X}_i^\top\mathbf{X}_i (\mathbf{a}\Bar{\mathbf{b}}^\top-\mathbf{a}^* \mathbf{b}^{*^\top}) \Bar{\mathbf{b}} \\
\mathbb{E}[\nabla_{\mathbf{a}}l(\mathbf{a},\Bar{\mathbf{b}})]  &=  \frac{2}{Nm}\sum_{i=1}^{N} m(\mathbf{a}\Bar{\mathbf{b}}^\top-\mathbf{a}^* \mathbf{b}^{*^\top}) \Bar{\mathbf{b}} \\
& = 2(\mathbf{a}\Bar{\mathbf{b}}^\top-\mathbf{a}^* \mathbf{b}^{*^\top}) \Bar{\mathbf{b}} \label{exp-grad}
\end{align}

Next, we bound $\lVert  \nabla_{\mathbf{a}}l(\mathbf{a},\Bar{\mathbf{b}})-\mathbb{E}[\nabla_{\mathbf{a}}l(\mathbf{a},\Bar{\mathbf{b}})]\rVert$. Construct $\epsilon_0$-net $\mathcal{N}$ over $d$ dimensional unit spheres $\mathcal{S}^{d-1}$, by Lemma~\ref{eps-net}, we have
\begin{align}
    &\lVert  \nabla_{\mathbf{a}}l(\mathbf{a},\Bar{\mathbf{b}})-\mathbb{E}[\nabla_{\mathbf{a}}l(\mathbf{a},\Bar{\mathbf{b}})]\rVert \\ &\leq \frac{1}{1-\epsilon_0} \max_{\mathbf{w}\in \mathcal{N}} \left|\mathbf{w}^\top\nabla_{\mathbf{a}}l(\mathbf{a},\Bar{\mathbf{b}})-\mathbf{w}^\top\mathbb{E}[\nabla_{\mathbf{a}}l(\mathbf{a},\Bar{\mathbf{b}})]\right| \label{grad-egrad-bound-with-w}\\
    & \le \frac{1}{1-\epsilon_0} \frac{2}{Nm}\max_{\mathbf{w}\in \mathcal{N}} \left| \sum_{i=1}^{N}\sum_{j=1}^{m}(\mathbf{x}_{i,j}^\top\mathbf{w})(\mathbf{x}_{i,j}(\mathbf{a}\Bar{\mathbf{b}}^\top-\mathbf{a}^* \mathbf{b}^{*^\top})) \Bar{\mathbf{b}} - \mathbf{w}^\top(\mathbf{a}\Bar{\mathbf{b}}^\top-\mathbf{a}^* \mathbf{b}^{*^\top}) \Bar{\mathbf{b}} \right| \label{eps-grad-egrad}
\end{align}
where $\mathbf{x}_{i,j}^\top$ is the $j$-th row of $\mathbf{X}_{i}$. Observe that $\mathbf{x}_{i,j}^\top\mathbf{w}$ and $\mathbf{x}_{i,j}(\mathbf{a}\Bar{\mathbf{b}}^\top-\mathbf{a}^* \mathbf{b}^{*^\top})\Bar{\mathbf{b}}$ are sub-Gaussian variables. Thus the product of them are sub-exponentials. For the right hand side of \eqref{eps-grad-egrad}, the summands are sub-exponential random variables with sub-exponential norm
\begin{align}
    &\| (\mathbf{x}_{i,j}^\top\mathbf{w})(\mathbf{x}_{i,j}(\mathbf{a}\Bar{\mathbf{b}}^\top-\mathbf{a}^* \mathbf{b}^{*^\top})) \Bar{\mathbf{b}} - \mathbf{w}^\top(\mathbf{a}\Bar{\mathbf{b}}^\top-\mathbf{a}^* \mathbf{b}^{*^\top}) \Bar{\mathbf{b}} \|_{\psi_1}\label{eq:bd2_psi1_1}\\
    &\leq \| (\mathbf{x}_{i,j}^\top\mathbf{w})(\mathbf{x}_{i,j}(\mathbf{a}\Bar{\mathbf{b}}^\top-\mathbf{a}^* \mathbf{b}^{*^\top})) \Bar{\mathbf{b}} \|_{\psi_1} + \| \mathbf{w}^\top(\mathbf{a}\Bar{\mathbf{b}}^\top-\mathbf{a}^* \mathbf{b}^{*^\top}) \Bar{\mathbf{b}}\|_{\psi_1}\label{eq:bd2_psi1_2}\\
    &\leq \| (\mathbf{x}_{i,j}^\top\mathbf{w})(\mathbf{x}_{i,j}(\mathbf{a}\Bar{\mathbf{b}}^\top-\mathbf{a}^* \mathbf{b}^{*^\top})) \Bar{\mathbf{b}} \|_{\psi_1} + \frac{| \mathbf{w}^\top(\mathbf{a}\Bar{\mathbf{b}}^\top-\mathbf{a}^* \mathbf{b}^{*^\top}) \Bar{\mathbf{b}}|}{\log 2}\label{eq:bd2_psi1_3}\\
    &\leq c_2\cdot \| 
    \mathbf w \|\cdot \lVert \mathbf{a}\Bar{\mathbf{b}}^\top-\mathbf{a}^* \mathbf{b}^{*^\top} \rVert_{op} \cdot \lVert \Bar{\mathbf{b}} \rVert + \frac{| \mathbf{w}^\top(\mathbf{a}\Bar{\mathbf{b}}^\top-\mathbf{a}^* \mathbf{b}^{*^\top}) \Bar{\mathbf{b}}|}{\log 2} \label{eq:bd2_psi1_4}\\
    &\leq c_2\cdot \| 
    \mathbf w \|\cdot \lVert \mathbf{a}\Bar{\mathbf{b}}^\top-\mathbf{a}^* \mathbf{b}^{*^\top} \rVert_{op} \cdot \lVert \Bar{\mathbf{b}} \rVert + \frac{\| \mathbf{w}\|\cdot\|(\mathbf{a}\Bar{\mathbf{b}}^\top-\mathbf{a}^* \mathbf{b}^{*^\top}) \Bar{\mathbf{b}}\|}{\log 2}\label{eq:bd2_psi1_5}\\
    &\leq c_2\cdot \| 
    \mathbf w \|\cdot \lVert \mathbf{a}\Bar{\mathbf{b}}^\top-\mathbf{a}^* \mathbf{b}^{*^\top} \rVert_{op} \cdot \lVert \Bar{\mathbf{b}} \rVert + \frac{\| \mathbf{w}\|\cdot\|\mathbf{a}\Bar{\mathbf{b}}^\top-\mathbf{a}^* \mathbf{b}^{*^\top}\|_{op}\cdot \| \Bar{\mathbf{b}}\|}{\log 2}\label{eq:bd2_psi1_6}\\
    &= c_4 \cdot\lVert \mathbf{a}\Bar{\mathbf{b}}^\top-\mathbf{a}^* \mathbf{b}^{*^\top} \rVert_{op} \cdot \lVert \Bar{\mathbf{b}} \rVert \label{eq:bd2_psi1_7}
\end{align}
where $c_4 = c_2 + \frac{1}{\log 2}$ is some absolute constant greater than 1. Equation \eqref{eq:bd2_psi1_3} is due to the fact that for a constant $c \in\mathbb R$, 
\[
\|c\|_{\psi_1} = \inf_t \exp\left\{\frac{|c|}{t} \le 2\right\} = \frac{|c|}{\log 2}.
\]
Equation \eqref{eq:bd2_psi1_4} is derived similarly as \eqref{eq:bd_psi1_1}-\eqref{eq:bd_psi1_3}.

The summands in \eqref{eps-grad-egrad} are mutually independent and have zero mean. Applying sub-exponential Bernstein inequality (Theorem~\ref{berstein}) with $t=\epsilon_3 Nm\lVert \mathbf{a}\Bar{\mathbf{b}}^\top-\mathbf{a}^* \mathbf{b}^{*^\top} \rVert_{op} \cdot \lVert \Bar{\mathbf{b}} \rVert $,
\begin{align}
    & \mathbb{P} \left\{  \left| \sum_{i=1}^{N}\sum_{j=1}^{m}[((\mathbf{x}_{i,j}^\top\mathbf{w})(\mathbf{x}_{i,j}(\mathbf{a}\Bar{\mathbf{b}}^\top-\mathbf{a}^* \mathbf{b}^{*^\top})) - \mathbf{w}^\top(\mathbf{a}\Bar{\mathbf{b}}^\top-\mathbf{a}^* \mathbf{b}^{*^\top}) )\Bar{\mathbf{b}}] \right| \geq \epsilon_3 Nm \lVert \mathbf{a}\Bar{\mathbf{b}}^\top-\mathbf{a}^* \mathbf{b}^{*^\top} \rVert_{op} \cdot \lVert \Bar{\mathbf{b}} \rVert  \right\} \label{berstein-2}\\
    & \leq 2\exp \left( -c \min (\frac{\epsilon_3^2Nm}{c_4^2}, \frac{\epsilon_3 N m}{c_4})\right) \\
    & = 2\exp{(-c_5\epsilon_3^2 Nm)}
\end{align}
for any fixed $\mathbf w \in \mathcal N$, $\epsilon_3 \in (0,1)$ and some absolute constant $c_5$. 

Now we apply union bound over all $\mathbf{w} \in \mathcal{N}$ using Corollary 4.2.13 of \cite{vershynin2018high}. We can conclude that 
\begin{align}
    &\mathbb{P}  \left\{ \max_{\mathbf w \in \mathcal N}   \left|\mathbf{w}^\top\nabla_{\mathbf{a}}l(\mathbf{a},\Bar{\mathbf{b}})-\mathbf{w}^\top\mathbb{E}[\nabla_{\mathbf{a}}l(\mathbf{a},\Bar{\mathbf{b}})]\right|\geq 2\epsilon_3 \lVert \mathbf{a}\Bar{\mathbf{b}}^\top-\mathbf{a}^* \mathbf{b}^{*^\top} \rVert_{op} \cdot \lVert \Bar{\mathbf{b}} \rVert  \right\} \label{eps-net-boung-grad-egrad-w}\\
     &\leq \sum_{\mathbf w \in \mathcal N} \mathbb{P}  \left\{   \left|\mathbf{w}^\top\nabla_{\mathbf{a}}l(\mathbf{a},\Bar{\mathbf{b}})-\mathbf{w}^\top\mathbb{E}[\nabla_{\mathbf{a}}l(\mathbf{a},\Bar{\mathbf{b}})]\right|\geq 2\epsilon_3 \lVert \mathbf{a}\Bar{\mathbf{b}}^\top-\mathbf{a}^* \mathbf{b}^{*^\top} \rVert_{op} \cdot \lVert \Bar{\mathbf{b}} \rVert   \right\}\\
    & \leq 2 \exp{(d\log(1+\frac{2}{\epsilon_0})-c_5\epsilon_3^2Nm)}
\end{align}
Combining \eqref{b-bar-g-difference-bound}, \eqref{grad-egrad-bound-with-w} , and \eqref{eps-net-boung-grad-egrad-w}, with probability at least $1-2 \exp{(d\log(1+\frac{2}{\epsilon_0})-c_5\epsilon_3^2Nm)}-q^{-10}$,
\begin{align}
\lVert  \nabla_{\mathbf{a}}l(\mathbf{a},\Bar{\mathbf{b}})-\mathbb{E}[\nabla_{\mathbf{a}}l(\mathbf{a},\Bar{\mathbf{b}})]\rVert &\leq \frac{1}{1-\epsilon_0} \max_{\mathbf{w}\in \mathcal{N}} \left|\mathbf{w}^\top\nabla_{\mathbf{a}}l(\mathbf{a},\Bar{\mathbf{b}})-\mathbf{w}^\top\mathbb{E}[\nabla_{\mathbf{a}}l(\mathbf{a},\Bar{\mathbf{b}})]\right| \\
   &\leq \frac{2\epsilon_3}{1-\epsilon_0} \lVert\mathbf{a}\Bar{\mathbf{b}}^\top-\mathbf{a}^* \mathbf{b}^{*^\top} \rVert_{op} \cdot \lVert \Bar{\mathbf{b}} \rVert \label{grad-egrad-4} \\
    & = \frac{2\epsilon_3}{1-\epsilon_0}  \| \mathbf{a}(\Bar{\mathbf{b}}-\mathbf{g})^\top - (\mathbf{I}_d-\mathbf{a}\mathbf{a}^\top)\mathbf{a}^*\mathbf{b}^{*^\top} \|_{op} \cdot \lVert \Bar{\mathbf{b}} \rVert
 \\ &\leq \frac{2\epsilon_3}{1-\epsilon_0} (\lVert \mathbf{a}^\top (\Bar{\mathbf{b}}-\mathbf{g} )\rVert + \| (\mathbf{I}_d-\mathbf{a}\mathbf{a}^\top)\mathbf{a}^*  \mathbf{b}^{*^\top}\|_{op}) \lVert \Bar{\mathbf{b}} \rVert  \label{grad-egrad-5}\\
 &\leq \frac{2\epsilon_3}{1-\epsilon_0} (\lVert \Bar{\mathbf{b}}-\mathbf{g} \rVert + \| (\mathbf{I}_d-\mathbf{a}\mathbf{a}^\top)\mathbf{a}^*  \| \cdot \|\mathbf{b}^{*}\|) \lVert \Bar{\mathbf{b}} \rVert\\
    & =  \frac{2\epsilon_3}{1-\epsilon_0} (\lVert \Bar{\mathbf{b}}-\mathbf{g} \rVert + \delta^t \lVert \mathbf{b}^*\rVert) \lVert \Bar{\mathbf{b}}-\mathbf{g}+\mathbf{g} \rVert \label{grad-egrad-1} \\
    & \leq \frac{2\epsilon_3}{1-\epsilon_0} (\lVert \Bar{\mathbf{b}}-\mathbf{g} \rVert + \delta^t \lVert \mathbf{b}^*\rVert)(\lVert \Bar{\mathbf{b}}-\mathbf{g} \rVert + \lVert \mathbf{g} \rVert) \\
    & \leq \frac{2\epsilon_3}{1-\epsilon_0} (\lVert \Bar{\mathbf{b}}-\mathbf{g} \rVert + \delta^t \lVert \mathbf{b}^*\rVert)(\lVert \Bar{\mathbf{b}}-\mathbf{g} \rVert + \lVert \mathbf{b}^* \rVert) \label{grad-egrad-2} \\
    &  = \frac{2\epsilon_3}{1-\epsilon_0} ( \lVert \Bar{\mathbf{b}}-\mathbf{g} \rVert^2 + \delta^t \lVert \Bar{\mathbf{b}}-\mathbf{g} \rVert  \lVert \mathbf{b}^* \rVert + \lVert \Bar{\mathbf{b}}-\mathbf{g} \rVert  \lVert \mathbf{b}^* \rVert + \delta^t  \lVert \mathbf{b}^* \rVert^2) \\
    &  \leq \frac{2\epsilon_3}{1-\epsilon_0} ((\epsilon')^2 (\delta^{t})^2 + \epsilon' (\delta^{t})^2 + \epsilon' \delta^t + \delta^t )\lVert \mathbf{b}^* \rVert^2 \\
    & \leq \frac{2\epsilon_3}{1-\epsilon_0} (\epsilon'+1)^2 \delta^t \lVert \mathbf{b}^* \rVert^2 \label{grad-egrad-3}\\
    &  = 2\Tilde{\epsilon} (\epsilon'+1)^2 \delta^t \lVert \mathbf{b}^* \rVert^2 \label{grad-egrad-final}
\end{align}

with $\Tilde{\epsilon} = \frac{\epsilon_3}{1-\epsilon_0}$.  \eqref{grad-egrad-4} uses \eqref{eps-net-boung-grad-egrad-w}. \eqref{grad-egrad-5} follows by triangle inequality. \eqref{grad-egrad-1} follows by $\delta^t = \|(\mathbf{I}_d - \mathbf{a}\mathbf{a}^\top)\mathbf{a}^* \|$. \eqref{grad-egrad-2} uses $\lVert \mathbf{g} \rVert = \lVert  \mathbf{b}^* \mathbf{a}^{*^\top} \mathbf{a} \rVert \leq \lVert  \mathbf{b}^* \rVert $. \eqref{grad-egrad-3} follows by $(\delta^{t})^{2} < \delta^t$ since $\delta^t \in (0,1)$. 

If $Nm=\Omega(\frac{d\log(\frac{2}{\epsilon_0})}{\epsilon_3^2})$, then existing large constant $C$, 
\begin{align}
1-2 \exp{(d\log(1+\frac{2}{\epsilon_0})-c_5\epsilon_3^2Nm)}-q^{-10}&>1-\exp(-Cd)-q^{-10} \\ 
&> 1-d^{-10}-q^{-10}\\
&>1-2q^{-10}
\end{align}Thus with probability at least $1-2q^{-10}$, \eqref{grad-egrad-final} holds.
\end{proof}

\begin{lemma}[Lemma 4.3] \label{lem:lemma-43-appendix}
    Let $\mathbf{a} = \mathbf{a}^t$. Let $\delta^t = \lVert (\mathbf{I}_d-\mathbf{a}^*\mathbf{a}^{*^\top})\mathbf{a}\rVert = \lVert (\mathbf{I}_d-\mathbf{a}\mathbf{a}^\top)\mathbf{a}^*\rVert$ denote the angle distance between $\mathbf{a}^{*}$ and $\mathbf{a}$. Assume that Assumption~\ref{client-norm-1} holds and $\delta^t \leq \delta^{t-1} \leq \dots \leq \delta^0$. Let $m$ be the number of samples for each updating step, let $\epsilon'=\frac{\epsilon_2}{(1-\epsilon_0)(1-\epsilon_1)},\Tilde{\epsilon}=\frac{\epsilon_3}{1-\epsilon_0}$ for $\epsilon_0, \epsilon_1, \epsilon_2, \epsilon_3 \in (0,1)$, if
    \[
    m= \Omega\left(\max\left\{\frac{\log(N)}{[\min(\epsilon_1,\epsilon_2)]^2},\frac{d\log(\frac{2}{\epsilon_0})}{\epsilon_2^2}\right\}\right)
    \]

    and $\epsilon',\Tilde{\epsilon}<\frac{1-(\delta^0)^2}{16}$, for any $t$ and $\eta \leq \frac{1}{L_{max}^2}$, then we have,
    \begin{align}
        \delta^{t+1} \leq \delta^{t} \sqrt{1-\eta (1-\delta^{0^2})\|\mathbf b^*\|^2} 
    \end{align}
    with probability at least $1-2q^{-10}$ for $q = \max\left\{\frac{\log(N)}{[\min(\epsilon_1,\epsilon_2)]^2},\frac{d\log(\frac{2}{\epsilon_0})}{\epsilon_2^2}\right\}$.
\end{lemma}
\begin{proof}

Recall that $\hat{\mathbf{a}}^{+} = \mathbf{a}-\eta   \nabla_{\mathbf{a}}l(\mathbf{a},\Bar{\mathbf{b}})$. We substract and add $\mathbb{E}[\nabla_{\mathbf{a}}l(\mathbf{a},\Bar{\mathbf{b}})]$, obtain
\begin{align}
    \hat{\mathbf{a}}^{+} = \mathbf{a}-\eta\mathbb{E}[\nabla_{\mathbf{a}}l(\mathbf{a},\Bar{\mathbf{b}})] + \eta (\mathbb{E}[\nabla_{\mathbf{a}}l(\mathbf{a},\Bar{\mathbf{b}})]-\nabla_{\mathbf{a}}l(\mathbf{a},\Bar{\mathbf{b}}))
\end{align}
Multiply both sides by the projection operator $\mathbf{P} = \mathbf{I}_d-\mathbf{a}^*(\mathbf{a}^*)^\top$,
\begin{align}
   \mathbf{P} \hat{\mathbf{a}}^{+} &= \mathbf{P}\mathbf{a}-\eta\mathbf{P}\mathbb{E}[\nabla_{\mathbf{a}}l(\mathbf{a},\Bar{\mathbf{b}})] + \eta \mathbf{P}(\mathbb{E}[\nabla_{\mathbf{a}}l(\mathbf{a},\Bar{\mathbf{b}})]-\nabla_{\mathbf{a}}l(\mathbf{a},\Bar{\mathbf{b}})) \\
   & = \mathbf{P}\mathbf{a}- 2\eta\mathbf{P}(\mathbf{a}\Bar{\mathbf{b}}^\top-\mathbf{a}^* \mathbf{b}^{*^\top}) \Bar{\mathbf{b}} + \eta \mathbf{P}(\mathbb{E}[\nabla_{\mathbf{a}}l(\mathbf{a},\Bar{\mathbf{b}})]-\nabla_{\mathbf{a}}l(\mathbf{a},\Bar{\mathbf{b}})) \label{gd-step-2}\\
   & =  \mathbf{P}\mathbf{a}-2\eta\mathbf{P}\mathbf{a}\Bar{\mathbf{b}}^\top \Bar{\mathbf{b}} + \eta \mathbf{P}(\mathbb{E}[\nabla_{\mathbf{a}}l(\mathbf{a},\Bar{\mathbf{b}})]-\nabla_{\mathbf{a}}l(\mathbf{a},\Bar{\mathbf{b}})) \label{gd-step-3} \\
   & = \mathbf{P}\mathbf{a}(1-2\eta \Bar{\mathbf{b}}^\top \Bar{\mathbf{b}})+\eta \mathbf{P}(\mathbb{E}[\nabla_{\mathbf{a}}l(\mathbf{a},\Bar{\mathbf{b}})]-\nabla_{\mathbf{a}}l(\mathbf{a},\Bar{\mathbf{b}})) 
\end{align}
where \eqref{gd-step-2} uses $\mathbb{E}[\nabla_{\mathbf{a}}l(\mathbf{a},\Bar{\mathbf{b}})] = 2(\mathbf{a}\Bar{\mathbf{b}}^\top-\mathbf{a}^* \mathbf{b}^{*^\top}) \Bar{\mathbf{b}} $, \eqref{gd-step-3} follows by $\mathbf{P}\mathbf{a}^*=0$. Thus, we get 
\begin{align}
   \lVert \mathbf{P} \hat{\mathbf{a}}^{+} \rVert \leq  \lVert \mathbf{P}\mathbf{a}\rVert|1-2\eta \Bar{\mathbf{b}}^\top \Bar{\mathbf{b}}|+\eta \lVert (\mathbb{E}[\nabla_{\mathbf{a}}l(\mathbf{a},\Bar{\mathbf{b}})]-\nabla_{\mathbf{a}}l(\mathbf{a},\Bar{\mathbf{b}})) \rVert
\end{align}
Normalizing the left hand side, we obtain
\begin{align}
    \frac{\lVert\mathbf{P} \hat{\mathbf{a}}^{+}\rVert}{\lVert \hat{\mathbf{a}}^{+} \rVert} &\leq \frac{\lVert \mathbf{P}\mathbf{a}\rVert|1-2\eta \Bar{\mathbf{b}}^\top \Bar{\mathbf{b}}|+\eta \lVert (\mathbb{E}[\nabla_{\mathbf{a}}l(\mathbf{a},\Bar{\mathbf{b}})]-\nabla_{\mathbf{a}}l(\mathbf{a},\Bar{\mathbf{b}})) \rVert}{\lVert \hat{\mathbf{a}}^{+} \rVert} \label{normalize-2} \\
    \Rightarrow  \delta^{t+1}& \leq\frac{\delta^t|1-2\eta \Bar{\mathbf{b}}^\top \Bar{\mathbf{b}}|+ \eta  \lVert \mathbb{E}[\nabla_{\mathbf{a}}l(\mathbf{a},\Bar{\mathbf{b}})]-\nabla_{\mathbf{a}}l(\mathbf{a},\Bar{\mathbf{b}}) \rVert}{\lVert \hat{\mathbf{a}}^{+} \rVert} \label{normalize-3}\\
    & = \frac{E_1+E_2}{\lVert \hat{\mathbf{a}}^{+} \rVert}
\end{align}
where \eqref{normalize-3} follows by $\delta^{t+1} =   \frac{\lVert\mathbf{P} \hat{\mathbf{a}}^{+}\rVert}{\lVert \hat{\mathbf{a}}^{+} \rVert} $ and $\delta^t = \lVert \mathbf{P}\mathbf{a}\rVert$. We need to upper bound $E_1$ and $E_2$ accordingly. $E_2$ is upper bounded based on Lemma~\ref{lemma:grad-egrad-bound}. With probability at least $1-2q^{-10}$,
\begin{align}
    E_2 &= \eta  \lVert \mathbb{E}[\nabla_{\mathbf{a}}l(\mathbf{a},\Bar{\mathbf{b}})]-\nabla_{\mathbf{a}}l(\mathbf{a},\Bar{\mathbf{b}}) \rVert \\
    &\leq 2\eta \Tilde{\epsilon} (\epsilon'+1)^2\delta^t\lVert \mathbf{b}^*\rVert^2 \label{E2}
\end{align}

To upper bound $E_1$, we need to lower bound $\lVert \Bar{\mathbf{b}}\rVert^2$. We can first lower bound $\lVert \Bar{\mathbf{b}}\rVert$ by:
\begin{align}
    \lVert \Bar{\mathbf{b}}\rVert & = \lVert \mathbf{g} - (\mathbf{g}-\Bar{\mathbf{b}})\rVert \\
    &\geq  \lVert \mathbf{g} \rVert - \lVert \mathbf{g}-\Bar{\mathbf{b}} \rVert \\
    &=\sqrt{1-(\delta^{t})^2} \lVert {\mathbf{b}^*}\rVert- \lVert \mathbf{g}-\Bar{\mathbf{b}} \rVert \label{b-bar-3}\\
    &\geq \sqrt{1-(\delta^{t})^2} \lVert {\mathbf{b}^*}\rVert - \epsilon'\delta^t \lVert \mathbf{b}^*\rVert  \label{b-bar}
\end{align}
with probability at least $1-q^{-10}$. \eqref{b-bar-3} follows by $\mathbf{g}^\top = \mathbf{a}^\top \mathbf{a}^* \mathbf{b}^{*^\top}$ and $\mathbf{a}^\top \mathbf{a}^* = \cos{\theta}(\mathbf{a}, \mathbf{a}^*)$, \eqref{b-bar} follows by Lemma~\ref{lemma:b-bar-g-bound}.
Assuming $\delta^t\leq \dots \leq \delta^0$, we choose $\epsilon'<\frac{{1-(\delta^0)^2}}{16}$ to make $\sqrt{1-(\delta^{t})^2} \lVert {\mathbf{b}^*}\rVert - \epsilon'\delta^t \lVert \mathbf{b}^*\rVert\geq0$.
Hence $\lVert \Bar{\mathbf{b}}\rVert^2$ is lower bounded by:
\begin{align}
    \lVert \Bar{\mathbf{b}}\rVert^2 & \geq (\sqrt{1-(\delta^{t})^2} \lVert {\mathbf{b}^*}\rVert - \epsilon'\delta^t \lVert \mathbf{b}^*\rVert )^2 \\
    & = (1-(\delta^{t})^2)\lVert {\mathbf{b}^*}\rVert^2 + (\epsilon')^2(\delta^{t})^2\lVert{\mathbf{b}^*}\rVert^2 -2\epsilon' \delta^{t}\sqrt{1-(\delta^{t})^2}\lVert {\mathbf{b}^*}\rVert^2 \\
    &\geq (1-(\delta^{t})^2)\lVert {\mathbf{b}^*}\rVert^2 + (\epsilon')^2(\delta^{t})^2\lVert{\mathbf{b}^*}\rVert^2 - \epsilon' \lVert{\mathbf{b}^*}\rVert^2 \label{b-bar-square-lower-bound-step-1}\\
    &\geq (1-(\delta^{0})^2)\lVert {\mathbf{b}^*}\rVert^2 - \epsilon' \lVert{\mathbf{b}^*}\rVert^2 \label{b-bar-square-lower-bound-step-2}
\end{align}
with probability at least $1-q^{-10}$. \eqref{b-bar-square-lower-bound-step-1} follows by $xy\leq \frac{1}{2}$ for $x^2+y^2=1$, \eqref{b-bar-square-lower-bound-step-2} follows by assuming $\delta^t \leq \delta^{t-1}\leq \dots \leq \delta^{0}$. $E_1$ is upper bounded below.
\begin{align}
    E_1 &= \delta^t|1-2\eta \Bar{\mathbf{b}}^\top \Bar{\mathbf{b}}| \\
    & \leq \delta^t |1-2\eta ((1-(\delta^{0})^2) - \epsilon' )\lVert {\mathbf{b}^*}\rVert^2 | \label{E1}
\end{align}
with probability at least $1-q^{-10}$. Next we lower bound $\lVert \hat{\mathbf{a}}^{+} \rVert$.
\begin{align}
    \lVert \hat{\mathbf{a}}^{+} \rVert^2 & = \lVert \mathbf{a}-\eta \nabla_{\mathbf{a}}l(\mathbf{a},\Bar{\mathbf{b}})\rVert^2 \\
    & = \mathbf{a}^\top\mathbf{a}+\eta^2\lVert\nabla_{\mathbf{a}}l(\mathbf{a},\Bar{\mathbf{b}}) \rVert^2 - 2\eta\mathbf{a}^\top \nabla_{\mathbf{a}}l(\mathbf{a},\Bar{\mathbf{b}}) \label{a-hat-norm-2} \\
    & \geq \mathbf{a}^\top\mathbf{a} - 2\eta\mathbf{a}^\top \nabla_{\mathbf{a}}l(\mathbf{a},\Bar{\mathbf{b}}) \label{a-hat-norm-3} \\
    & = 1- 2\eta\mathbf{a}^\top \nabla_{\mathbf{a}}l(\mathbf{a},\Bar{\mathbf{b}})  \label{a-hat-norm-4}\\
    & = 1- 2\eta  \mathbf{a}^\top (\nabla_{\mathbf{a}}l(\mathbf{a},\Bar{\mathbf{b}}) -  \mathbb{E}[\nabla_{\mathbf{a}}l(\mathbf{a},\Bar{\mathbf{b}})])-2\eta  \mathbf{a}^\top  \mathbb{E}[\nabla_{\mathbf{a}}l(\mathbf{a},\Bar{\mathbf{b}})] 
\end{align}
where \eqref{a-hat-norm-3} follows by $\eta^2\lVert\nabla_{\mathbf{a}}l(\mathbf{a},\Bar{\mathbf{b}}) \rVert^2 \geq 0$, and \eqref{a-hat-norm-4} follows by $\mathbf{a}^\top\mathbf{a}=1$. The first subtrahend $2\eta  \mathbf{a}^\top (\nabla_{\mathbf{a}}l(\mathbf{a},\Bar{\mathbf{b}}) -  \mathbb{E}[\nabla_{\mathbf{a}}l(\mathbf{a},\Bar{\mathbf{b}})])$ is upper bounded such that
\begin{align}
    2\eta  \mathbf{a}^\top (\nabla_{\mathbf{a}}l(\mathbf{a},\Bar{\mathbf{b}}) -  \mathbb{E}[\nabla_{\mathbf{a}}l(\mathbf{a},\Bar{\mathbf{b}})]) & \leq  2\eta \| \mathbf{a} \|\cdot \| (\nabla_{\mathbf{a}}l(\mathbf{a},\Bar{\mathbf{b}}) -  \mathbb{E}[\nabla_{\mathbf{a}}l(\mathbf{a},\Bar{\mathbf{b}})])\|   \\
    & = 2\eta \| (\nabla_{\mathbf{a}}l(\mathbf{a},\Bar{\mathbf{b}}) -  \mathbb{E}[\nabla_{\mathbf{a}}l(\mathbf{a},\Bar{\mathbf{b}})])\| \\
    &\leq 4\eta \Tilde{\epsilon} (\epsilon'+1)^2\lVert \mathbf{b}^* \rVert^2 \label{denominator-a-grad-egrad}
\end{align}
with probability at least $1-2q^{-10}$.  
\eqref{denominator-a-grad-egrad} uses Lemma~\ref{lemma:b-bar-g-bound}. The second subtrahend is upper bounded such that
\begin{align}
    2\eta  \mathbf{a}^\top  \mathbb{E}[\nabla_{\mathbf{a}}l(\mathbf{a},\Bar{\mathbf{b}})] & =  4\eta  \mathbf{a}^\top (\mathbf{a}\Bar{\mathbf{b}}^\top-\mathbf{a}^* \mathbf{b}^{*^\top}) \Bar{\mathbf{b}} \\
    & = 4\eta   \mathbf{a}^\top (\mathbf{a}\Bar{\mathbf{b}}^\top-\mathbf{a}^* \mathbf{b}^{*^\top}) \mathbf{g} - 4\eta  \mathbf{a}^\top (\mathbf{a}\Bar{\mathbf{b}}^\top-\mathbf{a}^* \mathbf{b}^{*^\top}) (\mathbf{g}-\Bar{\mathbf{b}} ) \label{a-expectation-inner-product}
\end{align}
where $\mathbf{a}^\top (\mathbf{a}\Bar{\mathbf{b}}^\top-\mathbf{a}^* \mathbf{b}^{*^\top}) \mathbf{g}=-\mathbf{a}^\top ((\mathbf{I}_d-\mathbf{a}\mathbf{a}^\top)\mathbf{a}^*\mathbf{b}^{*^\top}+\mathbf{a}(\mathbf{g}-\mathbf{\Bar{b}})^\top) \mathbf{g}=(\mathbf{\Bar{b}}-\mathbf{g})^\top \mathbf{g}$. The second term is simplified via $\mathbf{a}^\top (\mathbf{a}\Bar{\mathbf{b}}^\top-\mathbf{a}^* \mathbf{b}^{*^\top}) (\mathbf{g}-\Bar{\mathbf{b}} ) =\mathbf{a}^\top((\mathbf{I}_d-\mathbf{a}\mathbf{a}^\top)\mathbf{a}^*\mathbf{b}^{*^\top}+\mathbf{a}(\mathbf{g}-\mathbf{\Bar{b}})^\top) (\mathbf{\Bar{b}}-\mathbf{g}) = -(\mathbf{g}-\Bar{\mathbf{b}} )^2$. Both simplifications use $\mathbf{a}^\top(\mathbf{I}_d-\mathbf{a}\mathbf{a}^\top)=0$ and $\mathbf{a}^\top\mathbf{a}=1$. \eqref{a-expectation-inner-product} becomes 
\begin{align}
    2\eta  \mathbf{a}^\top  \mathbb{E}[\nabla_{\mathbf{a}}l(\mathbf{a},\Bar{\mathbf{b}})] &= 4\eta (\mathbf{\Bar{b}}-\mathbf{g})^\top \mathbf{g} + 4\eta (\mathbf{g}-\Bar{\mathbf{b}} )^2 \\
    &\leq 4\eta \lVert \mathbf{g}-\mathbf{\Bar{b}} \rVert \lVert \mathbf{b}^* \rVert +4 \eta \lVert \mathbf{g}-\mathbf{\Bar{b}} \rVert^2\\
    & \leq 4\eta \epsilon' \delta^t \lVert \mathbf{b}^*\rVert^2 + 4\eta (\epsilon')^2 (\delta^{t})^2 \lVert \mathbf{b}^*\rVert^2  \label{denominator-a-expectation-grad-1}\\
    &\leq 4\eta ((\epsilon')^2+\epsilon')\lVert \mathbf{b}^*\rVert^2 \label{denominator-a-expectation-grad}
\end{align}
with probability at least $1-q^{-10}$. \eqref{denominator-a-expectation-grad-1} uses Lemma~\ref{lemma:b-bar-g-bound}. Combining \eqref{denominator-a-grad-egrad} and \eqref{denominator-a-expectation-grad}, we obtain
\begin{align}
    \lVert \hat{\mathbf{a}}^+\rVert^2 &\geq 1-4\eta \Tilde{\epsilon} (\epsilon'+1)^2\lVert \mathbf{b}^* \rVert^2 - 4\eta ((\epsilon')^2+\epsilon')\lVert \mathbf{b}^*\rVert^2 \label{sqrt-a-hat-plus}
\end{align}
with probability at least $1-2q^{-10}$. Combining \eqref{E1}, \eqref{E2} and \eqref{sqrt-a-hat-plus}, we obtain
\begin{align}
   \delta^{t+1} \leq \frac{E_1+E_2}{\lVert \hat{\mathbf{a}}^+\rVert} \leq \frac{\delta^t( 1-2\eta ((1-(\delta^{0})^2) - \epsilon' )\lVert {\mathbf{b}^*}\rVert^2 +2\eta \Tilde{\epsilon} (\epsilon' +1)^2\lVert \mathbf{b}^*\rVert^2)}{\sqrt{1-4\eta \Tilde{\epsilon} (\epsilon'+1)^2\lVert \mathbf{b}^* \rVert^2 - 4\eta ((\epsilon')^2+\epsilon')\lVert \mathbf{b}^*\rVert^2}} = \delta^t C
\end{align}
We can choose $\epsilon',\Tilde{\epsilon}<\frac{1-(\delta^0)^2}{16}$ such that $(1-(\delta^{0})^2)> \max (4(\Tilde{\epsilon}(\epsilon'+1)^2+(\epsilon')^2+\epsilon'), 2\epsilon' +2\Tilde{\epsilon}(\epsilon'+1)^2)$ holds. Then we obtain
\begin{align}
    C &= \frac{ 1-2\eta ((1-(\delta^{0})^2) - \epsilon' )\lVert {\mathbf{b}^*}\rVert^2 +2\eta \Tilde{\epsilon} (\epsilon' +1)^2\lVert \mathbf{b}^*\rVert^2}{\sqrt{1-4\eta \Tilde{\epsilon} (\epsilon'+1)^2\lVert \mathbf{b}^* \rVert^2 - 4\eta ((\epsilon')^2+\epsilon')\lVert \mathbf{b}^*\rVert^2}} \\
    &=  \frac{ 1-2\eta (1-(\delta^{0})^2)\lVert {\mathbf{b}^*}\rVert^2 + 2\eta \epsilon' \lVert {\mathbf{b}^*}\rVert^2 +2\eta \Tilde{\epsilon} (\epsilon' +1)^2\lVert \mathbf{b}^*\rVert^2}{\sqrt{1-4\eta (\Tilde{\epsilon} (\epsilon'+1)^2 +  (\epsilon')^2+\epsilon')\lVert \mathbf{b}^*\rVert^2}} \\
    &\leq \frac{ 1-2\eta (1-(\delta^{0})^2)\lVert {\mathbf{b}^*}\rVert^2 + \eta (2\epsilon' + 2\Tilde{\epsilon} (\epsilon'+1)^2)\lVert \mathbf{b}^*\rVert^2}{\sqrt{1-4\eta (\Tilde{\epsilon} (\epsilon'+1)^2 +  (\epsilon')^2+\epsilon')\lVert \mathbf{b}^*\rVert^2}} \\
    &\leq \frac{1-\eta (1-(\delta^{0})^2)\lVert \mathbf{b}^*\rVert^2}{\sqrt{1-\eta (1-(\delta^{0})^2)\lVert \mathbf{b}^*\rVert^2}} \\
    & = \sqrt{1-\eta (1-(\delta^{0})^2)\lVert \mathbf{b}^*\rVert^2}
\end{align}
Assuming $\eta \leq \frac{1}{L_{max}^2} \le \frac{1}{\lVert \mathbf{b}^*\rVert^2}$, $1-\eta (1-(\delta^{0})^2)\lVert \mathbf{b}^*\rVert^2$ is strictly positive. Therefore we obtain, with probability at least $1-2q^{-10}$, 
\begin{align}
    \delta^{t+1} &\leq \delta^{t} \sqrt{1-\eta (1-(\delta^{0})^2)\lVert \mathbf{b}^*\rVert^2}.
\end{align}

\end{proof}
\subsection{Proof of Theorem 5.4} 
\label{subsec:mainProof}
\begin{proof}
    In Lemma~\ref{lemma:delta-t}, we have shown the angle distance between $\mathbf{a}$ and $\mathbf{a}^*$ decreasing in $t$-th iteration such that with probability at least $1-2q^{-10}$ for $q=\max\{\log(N),d\}$, $\delta^{t+1} \leq \delta^{t} C $ for $c\in(0,1), C = \sqrt{1-c(1-(\delta^{0})^2)}$ with proper choice of step size $\eta$.
    \paragraph{Proving $\delta^{1}\leq \delta^{0} C $}.
    Now we are to prove that for the first iteration, $\delta^{1}\leq \delta^{0} C $ with certain probability. 

    By Lemma~\ref{lemma:b-bar-g-bound}, we get $\|\Bar{\mathbf{b}}-\mathbf{g}\| \leq \epsilon' \delta^0 \|\mathbf{b}^* \|$ with probability at least $1-q^{-10}$. 

    By Lemma~\ref{lemma:grad-egrad-bound}, we get $\lVert  \nabla_{\mathbf{a}}l(\mathbf{a},\Bar{\mathbf{b}})-\mathbb{E}[\nabla_{\mathbf{a}}l(\mathbf{a},\Bar{\mathbf{b}})]\rVert \leq 2\Tilde{\epsilon} ((\epsilon')^2+\epsilon')\delta^{0}\lVert \mathbf{b}^*\rVert^2$ with probability at least $1-2q^{-10}$.

By Lemma~\ref{lem:lemma-43-appendix}, without assuming decreasing angles, we obtain, with probability at least $1-2q^{-10}$, 
\begin{align}
    \delta^{1} &\leq \delta^{0} \sqrt{1-\eta (1-(\delta^{0})^2)\lVert \mathbf{b}^*\rVert^2} .
\end{align}

    \paragraph{Inductive Hypothesis.}
    Based on inductive hypothesis, by proving $\delta^{1}\leq \delta^{0} C$, the assumption that $\delta^{t}\leq \delta^{t-1} C\leq \dots \leq \delta^{1} C^{t-1}$, and proving $\delta^{t+1}\leq \delta^{t} C $, we conclude that $\delta^{t}\leq \delta^{t-1} C$ holds for all $t \in [T]$ iterations.
    We take a union bound over all $t \in [T]$ such that,
    \begin{align}
        \mathbb{P} \left\{\bigcap_{t=0}^{T-1} \delta^{t+1} \leq \delta^{t} \sqrt{1-c(1-(\delta^{0})^2)} \right\} \geq 1 - 2Tq^{-10}.  
    \end{align}
    \paragraph{Solve for $T$.}
    In order to achieve $\epsilon$-recovery of $\mathbf{a}^*$, we need
    \begin{align}
        \delta^{0} ({1-c(1-(\delta^{0})^2)})^{\frac{T}{2}} &\leq \epsilon \\
        ({1-c (1-(\delta^{0})^2)})^{\frac{T}{2}} &\leq \frac{\epsilon}{\delta^0} \\
        \frac{T}{2} \log{({1-c (1-(\delta^{0})^2)})} &\leq \log (\frac{\epsilon}{\delta^0})  \\
        \end{align} 
        We proceed such that
        \begin{align}
        T &\geq \frac{2\log (\frac{\epsilon}{\delta^0})}{\log{({1-c (1-(\delta^{0})^2)})}} \\
         &> \frac{2\log (\frac{\epsilon}{\delta^0})}{-c (1-(\delta^{0})^2)} \label{T-2}\\
         &= \frac{2}{c (1-(\delta^{0})^2)} \log(\frac{\delta^0}{\epsilon}) 
    \end{align} 
    where \eqref{T-2} follows by using $\log(1-x)<-x$ for $|x|<1$. Thus, with probability at least $1-2Tq^{-10}$, $\delta^T = \sin \theta ({\mathbf{a}^T},\mathbf{a}^*) \leq \epsilon$. 

    \paragraph{Convergence to the target model.} We now aim to upper bound $\| \mathbf{a}^T(\mathbf{b}^{T+1})^\top- \mathbf{a}^* ({\mathbf{b}}^*)^\top\|$. Recall that $(\mathbf g^T)^\top = (\mathbf a^T)^\top \mathbf a^* \mathbf b^{*^\top} $ and $\delta^T = \|  (\mathbf I_d - \mathbf{a}^T(\mathbf a^T)^\top )  \mathbf{a}^* \|$, we have
    \begin{align}
        \lVert \mathbf{a}^T(\mathbf{b}^{T+1})^\top- \mathbf{a}^* \mathbf b^{*^\top}\rVert &= \lVert \mathbf{a}^T(\mathbf{b}^{T+1})^\top - \mathbf{a}^T(\mathbf{g}^{T})^\top + \mathbf{a}^T(\mathbf{g}^{T})^\top- \mathbf{a}^* \mathbf b^{*^\top}\rVert \\
        &\le  \lVert \mathbf{a}^T(\mathbf{b}^{T+1})^\top - \mathbf{a}^T(\mathbf{g}^{T})^\top\| + \| \mathbf{a}^T(\mathbf{g}^{T})^\top- \mathbf{a}^* \mathbf b^{*^\top}\rVert \\
        &= \lVert \mathbf{a}^T(\mathbf{b}^{T+1}-\mathbf{g}^{T})^\top \| + \| (\mathbf{a}^T(\mathbf a^T)^\top  - \mathbf I_d) \mathbf{a}^* \mathbf b^{*^\top}\rVert \\
        &= \|\mathbf{a}^T\| \|\mathbf{b}^{T+1}-\mathbf{g}^{T} \| + \|  (\mathbf I_d - \mathbf{a}^T(\mathbf a^T)^\top )  \mathbf{a}^* \| \|{\mathbf{b}}^*\rVert \\
        &\le \epsilon' \delta^T \|\mathbf b^*\| + \delta^T\|\mathbf b^*\| \label{eq:conv_target_1}\\
        &= (1 + \epsilon') \epsilon \|\mathbf b^*\| \\
        &= (1 + \epsilon') \epsilon \|\mathbf a^* \mathbf b^{*^\top}\|\label{eq:conv_target_2}
    \end{align}
    where \eqref{eq:conv_target_1}  is by Lemma \ref{lemma:b-bar-g-bound} and the fact that $\|\mathbf a^T\| = 1$, and \eqref{eq:conv_target_2} is due to the fact that $\| \mathbf x \mathbf y^\top\| =\| \mathbf x \|\| \mathbf y\| $ and $\|\mathbf a^*\| = 1$.
\end{proof}

\subsubsection{Proof of Proposition 5.5} \label{fa-heter-saturate}
\begin{proof}
    We start by fixing $\mathbf{a}^0$ and updating $\mathbf{b}_i$ to minimize the objective. Let $\mathbf{a}=\mathbf{a}^0$. We obtain
    \begin{align}
        \mathbf{b}_i^\top & = \frac{\mathbf{a}^\top \mathbf{X}_i^\top \mathbf{X}_i \mathbf{a}^* \mathbf{b}^{*^\top} }{\mathbf{a}^\top \mathbf{X}_i^\top\mathbf{X}_i \mathbf{a}} \\
        ({\mathbf{b}}^{FFA})^\top  & = \frac{1}{N} \sum_{i=1}^{N} \frac{\mathbf{a}^\top \mathbf{X}_i^\top \mathbf{X}_i \mathbf{a}^* \mathbf{b}^{*^\top} }{\mathbf{a}^\top \mathbf{X}_i^\top\mathbf{X}_i \mathbf{a}}
    \end{align}
   let $\Bar{\mathbf{b}}={\mathbf{b}}^{FFA}$. We aim to compute the expected value of $\frac{1}{N}\sum_{i=1}^{N}\frac{1}{m}\lVert \mathbf{X}_i \mathbf{a}^*\mathbf{b}^{*^{\top}}- \mathbf{X}_i \mathbf{a}\mathbf{\Bar{b}}^{\top}\rVert^2 $ where the expectation is over all the randomness in the $\mathbf{X}_i$. We define 
    \begin{align}
        s_i = \frac{\mathbf{a}^\top \mathbf{X}_i^\top \mathbf{X}_i \mathbf{a}^* }{\mathbf{a}^\top \mathbf{X}_i^\top\mathbf{X}_i \mathbf{a}}  = \frac{(\mathbf{X_i}\mathbf{a})^\top(\mathbf{X_i}\mathbf{a}^*)}{\lVert \mathbf{X}_i \mathbf{a} \rVert^2}
    \end{align}
    so that $\Bar{\mathbf{b}} = \frac{1}{N}\sum_{i=1}^{N}s_i \mathbf{b}^*=\Bar{s}\mathbf{b}^*$. For each $i$, the norm becomes
    \begin{align}
        \lVert \mathbf{X}_i \mathbf{a}^*\mathbf{b}^{*^{\top}}- \mathbf{X}_i \mathbf{a}\mathbf{\Bar{b}}^{\top}\rVert^2 & = \lVert(\mathbf{X}_i \mathbf{a}^* - \Bar{s}\mathbf{X}_i \mathbf{a}){\mathbf{b}}^{*^\top}\rVert^2 \\
        & = \lVert \mathbf{X}_i \mathbf{a}^* - \Bar{s}\mathbf{X}_i \mathbf{a} \rVert^2 \lVert \mathbf{b}^* \rVert^2
    \end{align}
    using the fact that $\lVert \mathbf{u}\mathbf{v}^\top\rVert^2 = \lVert \mathbf{u}\rVert^2 \lVert \mathbf{v}\rVert^2$ for vectors $\mathbf{u}$ and $\mathbf{v}$. Therefore, $\mathbb{E}[\frac{1}{N}\sum_{i=1}^{N}\frac{1}{m}\lVert \mathbf{X}_i \mathbf{a}^*\mathbf{b}^{*^{\top}}- \mathbf{X}_i \mathbf{a}\mathbf{\Bar{b}}^{\top}\rVert^2] $ is reduced to $\mathbb{E}[\frac{1}{N}\sum_{i=1}^{N}\frac{1}{m}\lVert \mathbf{X}_i \mathbf{a}^*- \mathbf{X}_i \mathbf{a}\rVert^2] \cdot \lVert\mathbf{b}^*\rVert^2$.

    Since each entry of $\mathbf{X}_i$ is independently and identically distributed according to a standard Gaussian distribution, both $\mathbf{a}^*$ and $\mathbf{a}$ are unit vectors, the vectors $ \mathbf{X}_i \mathbf{a}^*$ and $ \mathbf{X}_i \mathbf{a}$ are $\mathcal{N}(0,\mathbf{I}_m)$. The cross-covariance is $\alpha \mathbf{I}_m$ where $\alpha = \mathbf{a}^\top \mathbf{a}^*$. 
    
    By linearity, we can show that $\frac{1}{N}\sum_{i=1}^{N}\frac{1}{m}\lVert \mathbf{X}_i \mathbf{a}^*- \Bar{s} \mathbf{X}_i \mathbf{a}\rVert^2$ has the same expectation as $\frac{1}{m}\lVert \mathbf{X}_1 \mathbf{a}^* - \Bar{s}\mathbf{X}_1 \mathbf{a} \rVert^2$ because all $(\mathbf{X}_i \mathbf{a}^*, \mathbf{X}_i \mathbf{a})$ are i.i.d. pairs. Let $z_1 = \frac{s_1}{N}$ and $z_2 = \frac{s_2+\dots + s_N}{N}$, we have $\lVert \mathbf{X}_1 \mathbf{a}^* - z_1\mathbf{X}_1 \mathbf{a} - z_2\mathbf{X}_1 \mathbf{a} \rVert^2$. Let $\mathbf v = \mathbf{X}_1 \mathbf{a}^*, \mathbf{u}_1 = z_1 \mathbf{X}_1 \mathbf{a}$ and $\mathbf{u}_2 = z_2 \mathbf{X}_1 \mathbf{a}$. Thus,
    \begin{align}
        \lVert \mathbf{X}_1 \mathbf{a}^* - z_1\mathbf{X}_1 \mathbf{a} - z_2\mathbf{X}_1 \mathbf{a} \rVert^2 &= \|\mathbf{v}-\mathbf{u}_1-\mathbf{u}_2 \|^2 \\
        & = \mathbf{v}^\top \mathbf{v} + \mathbf{u}_1^\top \mathbf{u}_1 + \mathbf{u}_2^\top \mathbf{u}_2 -2 \mathbf{v}^\top \mathbf{u}_1 - 2\mathbf{v}^\top \mathbf{u}_2 + 2 \mathbf{u}_1^\top \mathbf{u}_2 \label{expand}
    \end{align}
    Now we compute the expectation term by term.
    \paragraph{Expected value of $\mathbf{v}^\top \mathbf{v}$} We have $\mathbb{E}[\mathbf{v}^\top \mathbf{v}]= \mathbb{E}[\| \mathbf{X}_1 \mathbf{a}^*\|^2]=m$.
    \paragraph{Expected value of $\mathbf{u}_1^\top \mathbf{u}_1$} We have
    \begin{align}
        \mathbf{u}_1^\top \mathbf{u}_1 & = z_1^2 \| \mathbf{X}_1 \mathbf{a} \|^2 \\
        & = \frac{s_1^2}{N^2} \| \mathbf{X}_1 \mathbf{a} \|^2 \\
        & = \frac{1}{N^2} \frac{((\mathbf{X}_1 \mathbf{a})^\top(\mathbf{X}_1 \mathbf{a}^*))^2}{\|  \mathbf{X}_1 \mathbf{a} \|^4} \| \mathbf{X}_1 \mathbf{a} \|^2 \\
         & = \frac{1}{N^2} \frac{((\mathbf{X}_1 \mathbf{a})^\top(\mathbf{X}_1 \mathbf{a}^*))^2}{\|  \mathbf{X}_1 \mathbf{a} \|^2} \label{u1-square-4}
    \end{align}
    Note that $(\mathbf{X}_1 \mathbf{a}^*, \mathbf{X}_1 \mathbf{a})$ is a correlated Gaussian pair with correlation $\alpha = \mathbf{a}^\top\mathbf{a}^*$. Without loss of generality, we assume $\mathbf{a}=\mathbf{e}_1$ thus $\mathbf{a}^*=\alpha \mathbf{e}_1+\sqrt{1-\alpha^2}\mathbf{e}_2$, where $\mathbf{e}_1$ and $\mathbf{e}_2$ are standard basis vectors in $\mathbb{R}^d$. So we can get $\mathbf{X}_1\mathbf{a}=\mathbf{X}_1 \mathbf{e}_1=\mathbf{x}_{1,1}$, where $\mathbf{x}_{1,1}$ denotes the first column of $\mathbf{X}_1$. Accordingly $\mathbf{X}_1\mathbf{a}^* = \alpha \mathbf{X}_1 \mathbf{e}_1+\sqrt{1-\alpha^2}\mathbf{X}_1 \mathbf{e}_2= \alpha \mathbf{x}_{1,1} + \sqrt{1-\alpha^2} \mathbf{x}_{1,2}$ where $\mathbf{x}_{1,2}$ denotes the second column of $\mathbf{X}_{1}$. Therefore \eqref{u1-square-4} can be written as $\frac{1}{N^2}\frac{(\mathbf{x}_{1,1}^\top(\alpha \mathbf{x}_{1,1} + \beta \mathbf{x}_{1,2}))^2}{\| \mathbf{x}_{1,1} \|^2}$. Now we take expectation of it.
    \begin{align}
        \mathbb{E}\left[\frac{1}{N^2} \frac{((\mathbf{X}_1 \mathbf{a})^\top(\mathbf{X}_1 \mathbf{a}^*))^2}{\|  \mathbf{X}_1 \mathbf{a} \|^2} \right] = \mathbb{E}\left[\frac{1}{N^2}\frac{(\mathbf{x}_{1,1}^\top(\alpha \mathbf{x}_{1,1} + \beta \mathbf{x}_{1,2}))^2}{\| \mathbf{x}_{1,1} \|^2}\right] & = \frac{1}{N^2} \mathbb{E}\left[ \frac{(\mathbf{x}_{1,1}^\top(\alpha \mathbf{x}_{1,1} + \beta \mathbf{x}_{1,2}))^2}{\| \mathbf{x}_{1,1} \|^2} \right] 
    \end{align}
    Let $r_1=\| \mathbf{x}_{1,1}\|^2$ and $r_2 = \mathbf{x}_{1,1}^\top \mathbf{x}_{1,2}$. We have
    \begin{align}
        \mathbb{E}\left[ \frac{(\mathbf{x}_{1,1}^\top(\alpha \mathbf{x}_{1,1} + \beta \mathbf{x}_{1,2}))^2}{\| \mathbf{x}_{1,1} \|^2} \right]  & = \mathbb{E}\left[ \frac{(\alpha r_1+\beta r_2)^2}{r_1} \right] \\
        & = \mathbb{E}\left[ \frac{\alpha^2 r_1^2 +\beta^2 r_2^2 + 2\alpha \beta r_1 r_2}{r_1} \right]\\
        & = \mathbb{E}\left[ \alpha^2 r_1\right] + \mathbb{E}\left[ \frac{\beta^2 r_2^2}{r_1}\right] + \mathbb{E}\left[ 2\alpha \beta r_2\right]
    \end{align}
    where $\mathbb{E}\left[ \alpha^2 r_1\right] = \alpha^2 \mathbb{E}\left[ \| \mathbf{x}_{1,1}\|^2 \right]=\alpha^2 m$, and $\mathbb{E}\left[ 2\alpha \beta r_2\right] = 2\alpha \beta \mathbb{E}\left[ r_2\right]=2\alpha \beta \mathbb{E}\left[ \mathbf{x}_{1,1}^\top \mathbf{x}_{1,2} \right]=0$ because $\mathbf{x}_{1,1}$ and $\mathbf{x}_{1,2}$ are independent standard Gaussian vectors. Then we analyze $\mathbb{E}\left[ \frac{\beta^2 r_2^2}{r_1}\right] = \beta^2 \mathbb{E}\left[ \frac{r_2^2}{r_1}\right]$. Condition on $\mathbf{x}_{1,1}$, 
    \begin{align}
        \mathbb{E}\left[ r_2 | \mathbf{x}_{1,1}\right]=\mathbb{E}\left[ \mathbf{x}_{1,1}^\top \mathbf{x}_{1,2} | \mathbf{x}_{1,1}\right]= \mathbf{x}_{1,1}^\top \mathbb{E}\left[ \mathbf{x}_{1,2} \right] = 0
    \end{align}
    and Var$(r_2|\mathbf{x}_{1,1}) = \| \mathbf{x}_{1,1} \|^2=r_1$, thus
    \begin{align}
        r_2|\mathbf{x}_{1,1} = \mathbf{x}_{1,1}^\top \mathbf{x}_{1,2} | \mathbf{x}_{1,1}  \sim \mathcal{N}(0,r_1) \label{x-1-x-2-con-x-1}
    \end{align} Then we obtain  
    \begin{align}
        \mathbb{E}\left[ r_2^2 | \mathbf{x}_{1,1}\right]=r_1 \label{x-1-x-2-square-given-x1}
    \end{align}
    Therefore $\mathbb{E}\left[ \frac{r_2^2}{r_1} | \mathbf{x}_{1,1}\right]=\frac{\mathbb{E}\left[ r_2^2 | \mathbf{x}_{1,1} \right]}{r_1}= 1$. We take total expectation $\mathbb{E}\left[ \frac{r_2^2}{r_1}\right] = \mathbb{E}\left[ \mathbb{E}\left[ \frac{r_2^2}{r_1} | \mathbf{x}_{1,1}\right] \right]=1$. Summarizing,
    \begin{align}
         \mathbb{E}\left[\frac{((\mathbf{X}_1 \mathbf{a})^\top(\mathbf{X}_1 \mathbf{a}^*))^2}{\|  \mathbf{X}_1 \mathbf{a} \|^2} \right] & =  \mathbb{E}\left[ \frac{(\alpha r_1+\beta r_2)^2}{r_1} \right]  \\
        & =  \mathbb{E}\left[ \alpha^2 r_1\right] + \mathbb{E}\left[ \frac{\beta^2 r_2^2}{r_1}\right] + \mathbb{E}\left[ 2\alpha \beta r_2\right]  \\
        & = 
   \alpha^2 m + \beta^2  \label{E-u1-u1-1} \\
   \mathbb{E}\left[ \mathbf{u}_1^\top \mathbf{u}_1 \right] & = \frac{1}{N^2}\mathbb{E}\left[\frac{((\mathbf{X}_1 \mathbf{a})^\top(\mathbf{X}_1 \mathbf{a}^*))^2}{\|  \mathbf{X}_1 \mathbf{a} \|^2} \right]\\&= \frac{\alpha^2 m + (1-\alpha^2)}{N^2} \label{u1-u1}
    \end{align}
    \paragraph{Expected value of $\mathbf{u}_2^\top \mathbf{u}_2$} We have $\mathbf{u}_2^\top \mathbf{u}_2=z_2^2 \| \mathbf{X}_1 \mathbf{a}\|^2$ where $z_2 = \frac{s_2 + \dots +s_N}{N}$ is independent of pair $(\mathbf{X}_1 \mathbf{a}^*, \mathbf{X}_1 \mathbf{a})$. To compute $\mathbb{E} \left[ z_2^2 \| \mathbf{X}_1 \mathbf{a}\|^2 \right]$, first we condition on $z_2$ to obtain,  
    \begin{align}
        \mathbb{E} \left[ z_2^2 \| \mathbf{X}_1 \mathbf{a}\|^2 | z_2 \right] =  z_2^2  \mathbb{E} \left[  \| \mathbf{X}_1 \mathbf{a}\|^2\right] =z^2_2 m
    \end{align}
    Then we take total expectation $ \mathbb{E} \left[ z_2^2 \| \mathbf{X}_1 \mathbf{a}\|^2 \right] = \mathbb{E} \left[ \mathbb{E} \left[ z_2^2 \| \mathbf{X}_1 \mathbf{a}\|^2 | z_2 \right] \right] = \mathbb{E} \left[z_2^2 m \right] = m \mathbb{E} \left[z_2^2 \right] $. 
    \begin{align}
        \mathbb{E} \left[z_2^2 \right] & = \mathbb{E} \left[ \frac{(s_2+\dots +s_N)^2}{N^2} \right] \\
        & = \frac{1}{N^2} \mathbb{E} \left[ \sum_{i=2}^{N} s_i^2 + \sum_{\substack{i=1, j=1\\ i\neq j}}^{N} s_i s_j\right] \\
        & = \frac{1}{N^2}  \left(\sum_{i=2}^{N} \mathbb{E} \left[s_i^2 \right]+ \sum_{\substack{i=1, j=1\\ i\neq j}}^{N}  \mathbb{E} \left[ s_i s_j \right]\right)
    \end{align}
    Write $s_i = \frac{(\mathbf{X_i}\mathbf{a})^\top(\mathbf{X_i}\mathbf{a}^*)}{\lVert \mathbf{X}_i \mathbf{a} \rVert^2}$. Without loss of generality, we assume $\mathbf{a}=\mathbf{e}_1$ thus $\mathbf{a}^*=\alpha \mathbf{e}_1+\sqrt{1-\alpha^2}\mathbf{e}_2$, where $\mathbf{e}_1$ and $\mathbf{e}_2$ are standard basis vectors in $\mathbb{R}^d$. 
    Thus, we have 
$
\mathbf{X}_i\mathbf{a} = \mathbf{X}_i\mathbf{e}_1 = \mathbf{x}_{i,1},
$
where \(\mathbf{x}_{i,1}\) represents the first column of \(\mathbf{X}_i\). Similarly, 
$
\mathbf{X}_i\mathbf{a}^* = \alpha \mathbf{X}_i\mathbf{e}_1 + \sqrt{1-\alpha^2}\mathbf{X}_i\mathbf{e}_2 = \alpha \mathbf{x}_{i,1} + \sqrt{1-\alpha^2} \mathbf{x}_{i,2},
$
where \(\mathbf{x}_{i,2}\) denotes the second column of \(\mathbf{X}_i\).

Hence,
\begin{align}
(\mathbf{X_i}\mathbf{a})^\top(\mathbf{X_i}\mathbf{a}^*)  = \mathbf{x}_{i,1}^\top (\alpha \mathbf{x}_{i,1} + \sqrt{1-\alpha^2} \mathbf{x}_{i,2}) = \alpha \| \mathbf{x}_{i,1} \|^2 + \sqrt{1-\alpha^2} (\mathbf{x}_{i,1}^\top \mathbf{x}_{i,2}) \label{x-i-1-top-x-i-a-star}
\end{align}
With $\| \mathbf{X_i}\mathbf{a} \|^2 = \| \mathbf{x}_{i,1}\|^2$, we have
\begin{align}
    s_i = \frac{\alpha \| \mathbf{x}_{i,1} \|^2 + \sqrt{1-\alpha^2} (\mathbf{x}_{i,1}^\top \mathbf{x}_{i,2})}{\| \mathbf{x}_{i,1}\|^2}=\alpha + \sqrt{1-\alpha^2} \frac{\mathbf{x}_{i,1}^\top \mathbf{x}_{i,2}}{\| \mathbf{x}_{i,1}\|^2} \label{s-i}
\end{align}
Let $R = \frac{\mathbf{x}_{i,1}^\top \mathbf{x}_{i,2}}{\| \mathbf{x}_{i,1}\|^2}$. Then 
\begin{align}
    \mathbb{E}\left[ s_i^2 \right] & = \mathbb{E}\left[ \left( \alpha + \sqrt{1-\alpha^2} R\right)^2 \right]\\
    & =  \mathbb{E}\left[ \alpha^2 + (1-\alpha^2)R^2+ 2\alpha\sqrt{1-\alpha^2} R\right] \\
    & = \alpha^2+ (1-\alpha^2) \mathbb{E}\left[ R^2 \right] +2\alpha\sqrt{1-\alpha^2} \mathbb{E}\left[ R \right] \label{E-s-i-square}
\end{align}
For $\mathbb{E}\left[ R \right]= \mathbb{E}\left[ \frac{\mathbf{x}_{i,1}^\top \mathbf{x}_{i,2}}{\| \mathbf{x}_{i,1}\|^2} \right]$, similarly as in \eqref{x-1-x-2-con-x-1}, $\mathbf{x}_{i,1}^\top \mathbf{x}_{i,2}| \mathbf{x}_{i,1} \sim \mathcal{N}(0, \| \mathbf{x}_{i,1}\|^2)$, thus $\frac{\mathbf{x}_{i,1}^\top \mathbf{x}_{i,2}}{\| \mathbf{x}_{i,1}\|^2} |\mathbf{x}_{i,1} \sim  \mathcal{N}(0, \frac{1}{\| \mathbf{x}_{i,1}\|^2})$, then 
\begin{align}
    \mathbb{E}\left[ R \right]=\mathbb{E}\left[ \mathbb{E}\left[ R | \mathbf{x}_{i,1} \right] \right] =0
\end{align}
For $\mathbb{E}\left[ R^2 \right]$, since $\mathbf{x}_{i,1}^\top \mathbf{x}_{i,2}| \mathbf{x}_{i,1} \sim \mathcal{N}(0, \| \mathbf{x}_{i,1}\|^2)$, so $\mathbb{E}\left[  (\mathbf{x}_{i,1}^\top \mathbf{x}_{i,2})^2 |\mathbf{x}_{i,1} \right]=\| \mathbf{x}_{i,1}\|^2$. Thus, with $R^2 = \frac{(\mathbf{x}_{i,1}^\top \mathbf{x}_{i,2})^2}{\| \mathbf{x}_{i,1}\|^4}$, 
\begin{align}
   &\mathbb{E}\left[ R^2 | \mathbf{x}_{i,1} \right]= \frac{\mathbb{E}\left[  (\mathbf{x}_{i,1}^\top \mathbf{x}_{i,2})^2 |\mathbf{x}_{i,1} \right]}{\| \mathbf{x}_{i,1}\|^4} = \frac{1}{\| \mathbf{x}_{i,1}\|^2} \\
   &\mathbb{E}\left[ R^2 \right] = \mathbb{E}\left[ \frac{1}{\| \mathbf{x}_{i,1}\|^2} \right]
\end{align}
For a $m$-dimensional standard Gaussian vector, $\| \mathbf{x}_{i,1}\|^2$ follows a chi-squared distribution with $m$ degrees of freedom. Therefore, $\mathbb{E}\left[ R^2 \right] = \frac{1}{m-2}$. \eqref{E-s-i-square} becomes
\begin{align}
     \mathbb{E}\left[ s_i^2 \right] &= \alpha^2+ (1-\alpha^2) \mathbb{E}\left[ R^2 \right] +2\alpha\sqrt{1-\alpha^2} \mathbb{E}\left[ R \right]\\
     &= \alpha^2 +(1-\alpha^2)\frac{1}{m-2}
\end{align}
Now we compute $\mathbb{E}\left[ s_is_j \right]$ for $i\neq j$. By independence of $s_i$ and $s_j$, $\mathbb{E}\left[ s_is_j \right]= \mathbb{E}\left[ s_i \right] \cdot \mathbb{E}\left[ s_j \right]=\mathbb{E}\left[ s_i \right]^2$. Take expectation of \eqref{s-i}, \begin{align}
    \mathbb{E}\left[ s_i \right] &= \mathbb{E}\left[ \alpha + \sqrt{1-\alpha^2} \frac{\mathbf{x}_{i,1}^\top \mathbf{x}_{i,2}}{\| \mathbf{x}_{i,1}\|^2} \right] \\
    & = \alpha + \sqrt{1-\alpha^2}   \mathbb{E}\left[ \frac{\mathbf{x}_{i,1}^\top \mathbf{x}_{i,2}}{\| \mathbf{x}_{i,1}\|^2} \right] \\
    &= \alpha + \sqrt{1-\alpha^2}   \mathbb{E}\left[ R \right] \\
    & = \alpha \label{expectation-s-i}
\end{align}
Hence, $\mathbb{E}\left[ s_is_j \right]=\alpha^2$. Summarizing,
\begin{align}
\mathbb{E}\left[\mathbf{u}_2^\top \mathbf{u}_2\right] &=m  \mathbb{E}\left[ z_2^2\right]\\
& = \frac{m}{N^2} \left((N-1)\mathbb{E}\left[ s_i^2\right]  + (N-1)(N-2)\mathbb{E}\left[ s_i s_j\right]\right) \\
& = \frac{m}{N^2}\left( (N-1) 
 (\alpha^2 +(1-\alpha^2)\frac{1}{m-2})+(N-1)(N-2)\alpha^2\right) \\
 & = \frac{m}{N^2}\left[ (N-1)^2 \alpha^2 + (N-1)\frac{1-\alpha^2}{m-2}\right] \label{u2-u2}
\end{align}

    \paragraph{Expected value of $\mathbf{v}^\top \mathbf{u}_1$} We have $\mathbf{v}^\top \mathbf{u}_1 = z_1 (\mathbf{X}_1 \mathbf{a}^*)^\top (\mathbf{X}_1 \mathbf{a})= \frac{s_1}{N}(\mathbf{X}_1 \mathbf{a}^*)^\top (\mathbf{X}_1 \mathbf{a})$. We factor out $\frac{1}{N}$, $\mathbb{E}\left[ \mathbf{v}^\top \mathbf{u}_1\right]=\frac{1}{N}\mathbb{E}\left[ s_1 (\mathbf{X}_1 \mathbf{a}^*)^\top (\mathbf{X}_1 \mathbf{a})\right]$. By \eqref{E-u1-u1-1},
    \begin{align}
        \mathbb{E}\left[ s_1 (\mathbf{X}_1 \mathbf{a}^*)^\top (\mathbf{X}_1 \mathbf{a})\right] & = \mathbb{E}\left[\frac{((\mathbf{X}_1 \mathbf{a})^\top(\mathbf{X}_1 \mathbf{a}^*))^2}{\|  \mathbf{X}_1 \mathbf{a} \|^2} \right] \\
        & = \alpha^2m + (1-\alpha^2)
    \end{align}
Then 
\begin{align}
    \mathbb{E}\left[ \mathbf{v}^\top \mathbf{u}_1\right]= \frac{\alpha^2m + (1-\alpha^2)}{N} \label{v-u1}
\end{align}

    \paragraph{Expected value of $\mathbf{v}^\top \mathbf{u}_2$} We have $\mathbf{v}^\top \mathbf{u}_2 = z_2 (\mathbf{X}_1 \mathbf{a}^*)^\top (\mathbf{X}_1 \mathbf{a})$. Condition on $z_2$ which is independent of $(\mathbf{X}_1 \mathbf{a}^*)^\top (\mathbf{X}_1 \mathbf{a})$, we obtain
    \begin{align}
        \mathbb{E}\left[  z_2 (\mathbf{X}_1 \mathbf{a}^*)^\top (\mathbf{X}_1 \mathbf{a}) |z_2\right] & = z_2 \mathbb{E}\left[  (\mathbf{X}_1 \mathbf{a}^*)^\top (\mathbf{X}_1 \mathbf{a}) \right] 
    \end{align}
Still we assume $\mathbf{a}=\mathbf{e}_1$ thus $\mathbf{a}^*=\alpha \mathbf{e}_1+\sqrt{1-\alpha^2}\mathbf{e}_2$, where $\mathbf{e}_1$ and $\mathbf{e}_2$ are standard basis vectors in $\mathbb{R}^d$. 
With $\mathbf{X}_1\mathbf{a}=\mathbf{X}_1 \mathbf{e}_1=\mathbf{x}_{1,1}$, where $\mathbf{x}_{1,1}$ denotes the first column of $\mathbf{X}_1$, and $\mathbf{X}_1\mathbf{a}^* = \alpha \mathbf{X}_1 \mathbf{e}_1+\sqrt{1-\alpha^2}\mathbf{X}_1 \mathbf{e}_2= \alpha \mathbf{x}_{1,1} + \sqrt{1-\alpha^2} \mathbf{x}_{1,2}$ where $\mathbf{x}_{1,2}$ denotes the second column of $\mathbf{X}_{1}$, using \eqref{x-i-1-top-x-i-a-star}, 
\begin{align}
    \mathbb{E}\left[  (\mathbf{X}_1 \mathbf{a}^*)^\top (\mathbf{X}_1 \mathbf{a}) \right] &= \mathbb{E}\left[ \alpha \| \mathbf{x}_{1,1}\|^2 + \sqrt{1-\alpha^2} (\mathbf{x}_{1,1}^\top \mathbf{x}_{1,2})\right] \\
    & =  \alpha \mathbb{E}\left[ \| \mathbf{x}_{1,1}\|^2 \right] + z_2 \sqrt{1-\alpha^2} \mathbb{E}\left[ \mathbf{x}_{1,1}^\top \mathbf{x}_{1,2}\right] \\
    & = \alpha m \label{exp-x1-a-x-1-a-star}
\end{align}
Thus $z_2 \mathbb{E}\left[  (\mathbf{X}_1 \mathbf{a}^*)^\top (\mathbf{X}_1 \mathbf{a}) \right]= z_2 \alpha m$, Then we take total expectation
\begin{align}
\mathbb{E}\left[\mathbb{E}\left[  z_2 (\mathbf{X}_1 \mathbf{a}^*)^\top (\mathbf{X}_1 \mathbf{a}) |z_2\right]\right] &= \mathbb{E}\left[ z_2 \alpha m \right] \\
& = \alpha m \mathbb{E}\left[ z_2 \right]
\end{align}
where $z_2=\frac{s_2+\dots +s_N}{N}$. Therefore,
\begin{align}
    \alpha m \mathbb{E}\left[ z_2 \right] = \frac{\alpha m}{N} \sum_{i=2}^N \mathbb{E}\left[s_i\right]
    =\frac{ m}{N} (N-1) \alpha^2 \label{expectation-z-2}
\end{align}
where \eqref{expectation-z-2} follows by $\mathbb{E}\left[ s_i \right]=\alpha$. Summarizing, we obtain $\mathbb{E}\left[ \mathbf{v}^\top \mathbf{u}_2\right]= \frac{ m}{N} (N-1) \alpha^2$.

    \paragraph{Expected value of $\mathbf{u}_1^\top \mathbf{u}_2$} We have $\mathbf{u}_1^\top \mathbf{u}_2= z_1 z_2 \| \mathbf{X}_1 \mathbf{a} \|^2$. By definition of $z_1$ and $z_2$, we obtain
    \begin{align}
         z_1 z_2 \| \mathbf{X}_1 \mathbf{a} \|^2 = \frac{1}{N^2}((\mathbf{X}_1 \mathbf{a}^*)^\top (\mathbf{X}_1 \mathbf{a}) ) \sum_{i=2}^N s_i
    \end{align}
    Since $(\mathbf{X}_1 \mathbf{a}^*)^\top (\mathbf{X}_1 \mathbf{a})$ depends only on $\mathbf{X}_1$ , $\sum_{i=2}^N s_i$ is independent of $\mathbf{X}_1$, we obtain
    \begin{align}
        \mathbb{E}\left[ \frac{1}{N^2}((\mathbf{X}_1 \mathbf{a}^*)^\top (\mathbf{X}_1 \mathbf{a}) ) \sum_{i=2}^N s_i \right]
       &  = \frac{1}{N^2} \mathbb{E}\left[ (\mathbf{X}_1 \mathbf{a}^*)^\top (\mathbf{X}_1 \mathbf{a})\right] \cdot \mathbb{E}\left[\sum_{i=2}^N s_i \right] \\
       & = \frac{1}{N^2} \mathbb{E}\left[ (\mathbf{X}_1 \mathbf{a}^*)^\top (\mathbf{X}_1 \mathbf{a})\right] \cdot (N-1) \mathbb{E}\left[ s_i \right] \\
       & = \frac{(N-1)m\alpha^2 }{N^2} \label{u1-u2}
    \end{align}
    where \eqref{u1-u2} follows by $\mathbb{E}\left[ (\mathbf{X}_1 \mathbf{a}^*)^\top (\mathbf{X}_1 \mathbf{a})\right]=\alpha m$ and $\mathbb{E}\left[ s_i \right]=\alpha$.
    
    Combining \eqref{u1-u1}, \eqref{u2-u2},\eqref{v-u1},\eqref{expectation-z-2},\eqref{u1-u2} and \eqref{expand}, 
    \begin{align}
        \frac{1}{m}\lVert \mathbf{X}_1 \mathbf{a}^*- \Bar{s} \mathbf{X}_1 \mathbf{a}\rVert^2 & = \frac{1}{m}(\mathbf{v}^\top \mathbf{v} + \mathbf{u}_1^\top \mathbf{u}_1 + \mathbf{u}_2^\top \mathbf{u}_2 -2 \mathbf{v}^\top \mathbf{u}_1 - 2\mathbf{v}^\top \mathbf{u}_2 + 2 \mathbf{u}_1^\top \mathbf{u}_2) \\
        & = (1-\alpha^2)\left[ 1+ \frac{N(4-m)-2}{N^2m(m-2)}\right] \\
        & = (\delta^0)^2(1+\Tilde{c})
    \end{align}
    where $\delta^0$ is the angle distance between $\mathbf{a}$ and $\mathbf{a}^*$. The quantity $ \Tilde{c}= \frac{N(4-m)-2}{N^2m(m-2)}=O(\frac{1}{Nm})$ as $N$ and $m$ approach infinity. Therefore,

\begin{align}
        \mathbb{E}\left[\frac{1}{N}\sum_{i=1}^{N}\frac{1}{m}\lVert \mathbf{X}_i \mathbf{a}^*\mathbf{b}^{*^{\top}}- \mathbf{X}_i \mathbf{a}\mathbf{\Bar{b}}^{\top}\rVert^2\right] = (1+\Tilde{c})  (\delta^0)^2 \lVert \mathbf{b}^*\rVert^2
    \end{align}

\end{proof}

\subsection{Heterogeneous Case} \label{rank-1-heter-clients}
Consider a federated setting with $N$ clients,  each with the following local linear model
\begin{equation}
    f_i(\mathbf{X}_i)=\mathbf{X}_i\mathbf{a}\mathbf{b}^\top \label{linear-model}
\end{equation}
where $\mathbf{a} \in \mathbb{R}^{d}$ is a unit vector and $\mathbf{b} \in \mathbb{R}^d$ are the LoRA weights corresponding to rank $r=1$. In this setting, we model the local data of $i$-th client such that $\mathbf{Y}_i=\mathbf{X}_i\mathbf{a}^*\mathbf{b}_i^{*^\top}$ for some ground truth LoRA weights $\mathbf{a}^* \in \mathbb{R}^{d}$, which is a unit vector, and local $\mathbf{b}_i^* \in \mathbb{R}^d$. We consider the following objective
\begin{equation}
    \min_{\mathbf{a}\in \mathbb{R}^{d},\mathbf{b}\in \mathbb{R}^d} \frac{1}{N}\sum_{i=1}^N l_i(\mathbf{a},\mathbf{b}) 
\end{equation}
We consider the local population loss $l_i(\mathbf{a},\mathbf{b}) = \lVert \mathbf{a}^*\mathbf{b}_i^{*^\top}- \mathbf{a}\mathbf{b}^\top\rVert^2$.

We aim to learn a shared model ($\mathbf{a},\mathbf{b}$) for all the clients. It is straightforward to observe that ($\mathbf{{a}}',\mathbf{b}'$) is a global minimizer of if and only if $\mathbf{{a}}'\mathbf{b'}^{^\top}=\mathbf{a}^*\Bar{\mathbf{b}}^*$, where $\Bar{\mathbf{b}}^*=\frac{1}{N}\sum_{i=1}^{N} \mathbf{b}_i^*$. The solution is unique and satisfies $\mathbf{a}'=\mathbf{a}^*$ and $\mathbf{b}'=\Bar{\mathbf{b}}^*$. With this global minimizer, we obtain the corresponding minimum global error of $\frac{1}{N} \sum_{i=1}^N \lVert \mathbf{a}^*(\mathbf{b}_i^*-\Bar{\mathbf{b}}^*)^\top \rVert^2$.

We aim to show that the training procedure described in Algorithm~\ref{alg:rolora-linear} learns the global minimizer $(\mathbf{a}^*,\Bar{\mathbf{b}}^*)$. First, we make typical assumption and definition.
\begin{assumption}  \label{client-norm}
     There exists $L_{max} < \infty$ (known a priori),  s.t. $\lVert \Bar{\mathbf{b}}^*\rVert \leq  L_{max}$.
\end{assumption}

\begin{definition} (Client variance) \label{client-var}
    For $\gamma>0$, we define $\gamma^2 := \frac{1}{N} \sum_{i=1}^{N}\lVert \mathbf{b}_i^*-\Bar{\mathbf{b}}^*\rVert^2$, where $\Bar{\mathbf{b}}^* = \frac{1}{N} \sum_{i=1}^{N} \mathbf{b}_i^*$. 
\end{definition}

\begin{theorem} (Convergence of RoLoRA for linear regressor in heterogeneous setting) \label{convergence}
Let $\delta^t = \lVert (\mathbf{I}_d-\mathbf{a}^*\mathbf{a}^{*^\top})\mathbf{a}^t\rVert $ be the angle distance between $\mathbf{a}^{*}$ and $\mathbf{a}^{t}$ of $t$-th iteration. Suppose we are in the setting described in Section~\ref{rank-1-heter-clients} and apply Algorithm~\ref{alg:rolora-linear} for optimization. Given a random initial $\mathbf{a}^0$, an initial angle distance $\delta_0 \in (0,1)$, we set the step size $\eta \leq \frac{1}{2L^2_{\max}}$ and the number of iterations $T \geq \frac{1}{c (1-(\delta^{0})^2)} \log(\frac{\delta^0}{\epsilon}) $ for $c \in (0,1)$. Under these conditions, we achieve the following
\begin{align}
    \sin \theta (\mathbf{a}^T, \mathbf{a}^*)\leq \epsilon \nonumber , \text{ and } \lVert \mathbf{a}^T(\mathbf{b}^{T+1})^\top- \mathbf{a}^* (\Bar{\mathbf{b}}^*)^\top\rVert \leq \epsilon \lVert \mathbf{a}^* (\Bar{\mathbf{b}}^*)^\top \rVert \nonumber
\end{align}
which we refer to as $\epsilon$-accurate recovery of the global minimizer.
\end{theorem}
Theorem~\ref{convergence} follows by recursively applying Lemma~\ref{converge-iter} for $T$ iterations. We start by computing the update rule for ${\mathbf{a}}$ as in Lemma~\ref{a-update-homo}. Using Lemma~\ref{a-update-homo}, we analyze the convergence of ${\mathbf{a}}$ in Lemma~\ref{converge-iter}. We also show the global error that can be achieved by FFA-LoRA within this setting in Proposition~\ref{fa-saturate}.

\begin{lemma} (Update for ${\mathbf{a}}$) \label{a-update-homo}
    In RoLoRA for linear regressor, the update for ${\mathbf{a}}$ and $b$ in each iteration is:
    \begin{align}
    \mathbf{b}^{t+1} &=\Bar{\mathbf{b}} = \Bar{\mathbf{b}}^*\mathbf{a^*}^\top \mathbf{a}^t \\
       \mathbf{a}^{t+1} &= \hat{\mathbf{a}} = \frac{\mathbf{a}^t- 2\eta (\mathbf{a}^t\Bar{\mathbf{b}}^\top\Bar{\mathbf{b}}-\mathbf{a}^*\mathbf{\Bar{b}}^{*^\top}\Bar{\mathbf{b}})}{\lVert \hat{\mathbf{a}}^{+} \rVert}
    \end{align}
where $\Bar{\mathbf{b}}^* = \sum_{i=1}^{N} \mathbf{b}_i^*, \lVert \hat{\mathbf{a}}^{+}\rVert =  \lVert \mathbf{a}^t- 2\eta (\mathbf{a}^t\Bar{\mathbf{b}}^\top\Bar{\mathbf{b}}-\mathbf{a}^*\mathbf{\Bar{b}}^{*^\top}\Bar{\mathbf{b}})\rVert$.
\end{lemma}
\begin{proof}
\textbf{Minimization for $b_i$. } At the start of each iteration, each client computes the analytic solution for $\mathbf{b}_i$ by fixing $\mathbf{a}$ and solving their local objective $\argmin_{\mathbf{b}_i} \lVert \mathbf{a}^*\mathbf{b}_i^{*^\top}- \mathbf{a}\mathbf{b}_i^\top \rVert^2$, where $\mathbf{a}^*$ and $\mathbf{a}$ are both unit vectors. Setting $\mathbf{a} = \mathbf{a}^t$, we obtain $\mathbf{b}_i$ such that
\begin{equation}
    \mathbf{b}_i = \frac{\mathbf{b}_i^{*}\mathbf{a}^{*^\top} \mathbf{a} }{\mathbf{a}^\top \mathbf{a}} =\mathbf{b}_i^{*}  \mathbf{a}^{*^\top} \mathbf{a} \label{b-update}
\end{equation}
\eqref{b-update} follows since $\mathbf{a}^\top \mathbf{a}=1$.

\textbf{Aggregation for $\mathbf{b}_i$. } The server simply computes the average of $\{\mathbf{b}_i\}_{i=1}^N$ and gets 
\begin{equation}
\Bar{\mathbf{b}} = \sum_{i=1}^N \mathbf{b}_i = \sum_{i=1}^N \mathbf{b}_i^* \mathbf{a^*}^\top \mathbf{a} = \Bar{\mathbf{b}}^*\mathbf{a^*}^\top \mathbf{a}
\end{equation}
The server then sends $\Bar{\mathbf{b}}$ to clients for synchronization.

\textbf{Gradient Descent for $\hat{\mathbf{a}}$.}
In this step, each client fixes $\mathbf{b}_i$ to $\Bar{\mathbf{b}}$ received from the server and update $\mathbf{a}$ using gradient descent. With the following gradient
\begin{equation}
    \nabla_{\mathbf{a}}l_i(\mathbf{a}, \Bar{\mathbf{b}}) = 2(\mathbf{a}\Bar{\mathbf{b}}^\top\Bar{\mathbf{b}}-\mathbf{a}^*\mathbf{b}_i^{*^\top}\Bar{\mathbf{b}})
\end{equation}
Thus, with step size $\eta$, $\mathbf{a}$ is updated such as
\begin{align}
\hat{\mathbf{a}}^{+} &= \mathbf{a}-\frac{\eta}{N}\sum_{i=1}^{N} \nabla_{\mathbf{a}}l_i(\mathbf{a},\Bar{\mathbf{b}}) \nonumber\\
&= \mathbf{a}-2\frac{\eta}{N} \sum_{i=1}^{N} (\mathbf{a}\Bar{\mathbf{b}}^\top\Bar{\mathbf{b}}-\mathbf{a}^*\mathbf{b}_i^{*^\top}\Bar{\mathbf{b}}) \nonumber\\
& = \mathbf{a}- 2\eta (\mathbf{a}\Bar{\mathbf{b}}^\top\Bar{\mathbf{b}}-\mathbf{a}^*\mathbf{\Bar{b}}^{*^\top}\Bar{\mathbf{b}}) \\
\hat{\mathbf{a}} &= \frac{\mathbf{a}- 2\eta (\mathbf{a}\Bar{\mathbf{b}}^\top\Bar{\mathbf{b}}-\mathbf{a}^*\mathbf{\Bar{b}}^{*^\top}\Bar{\mathbf{b}})}{\lVert \hat{\mathbf{a}}^{+} \rVert}
\end{align}
\end{proof}

\begin{lemma} \label{converge-iter}
    Let $\delta_t = |\sin \theta (\mathbf{a}^*, \mathbf{a}^t)|$ be the angle distance between $\mathbf{a}^*$ and $ \mathbf{a}^t$. Assume that Assumption~\ref{client-norm} holds and $\delta_t \leq \delta_{t-1} \leq \dots \leq \delta_0$, if $\eta \leq \frac{1}{2L^2_{\max}}$, then 
    \begin{equation}
        |\sin \theta (\mathbf{a}^{t+1}, \mathbf{a}^*)| = \delta_{t+1} \leq \delta_{t} \cdot (1-2\eta (1-(\delta^0)^2) \|\Bar{\mathbf b}^*\|^2)
    \end{equation}
\end{lemma}
\begin{proof}
    From Lemma~\ref{a-update-homo}, $\mathbf{a}^{t+1}$ and $\mathbf{b}^{t+1}$ are computed as follows:
    \begin{align}
    \mathbf{b}^{t+1} &=\Bar{\mathbf{b}} = \Bar{\mathbf{b}}^*\mathbf{a^*}^\top \mathbf{a}^t \\
       \mathbf{a}^{t+1} &=  \frac{\mathbf{a}^t- 2\eta (\mathbf{a}^t\Bar{\mathbf{b}}^\top\Bar{\mathbf{b}}-\mathbf{a}^*\mathbf{\Bar{b}}^{*^\top}\Bar{\mathbf{b}})}{\lVert \mathbf{a}^t- 2\eta (\mathbf{a}^t\Bar{\mathbf{b}}^\top\Bar{\mathbf{b}}-\mathbf{a}^*\mathbf{\Bar{b}}^{*^\top}\Bar{\mathbf{b}})\rVert}\label{as-update}
    \end{align}
    Note that $\mathbf{a}^t$ and $\mathbf{a}^{t+1}$ are both unit vectors. Now, we multiply both sides of Equation~\eqref{as-update} by the projection operator $\mathbf{P} = \mathbf{I}_d-\mathbf{a}^*(\mathbf{a}^*)^\top$, which is the projection to the direction orthogonal to $\mathbf{a}^*$. We obtain:
    \begin{align}
        \mathbf{P}\mathbf{a}^{t+1} &= \frac{\mathbf{P}\mathbf{a}^t- 2\eta \mathbf{P}\mathbf{a}^t \Bar{\mathbf{b}}^\top\Bar{\mathbf{b}}+\mathbf{P}\mathbf{a}^*\mathbf{\Bar{b}}^{*^\top}\Bar{\mathbf{b}}}{\lVert \mathbf{a}^t- 2\eta (\mathbf{a}^t\Bar{\mathbf{b}}^\top\Bar{\mathbf{b}}-\mathbf{a}^*\mathbf{\Bar{b}}^{*^\top}\Bar{\mathbf{b}})\rVert} \label{projection} \\
        & = \frac{\mathbf{P}\mathbf{a}^t- 2\eta \mathbf{P}\mathbf{a}^t\Bar{\mathbf{b}}^\top\Bar{\mathbf{b}}}{\lVert \mathbf{a}^t- 2\eta (\mathbf{a}^t\Bar{\mathbf{b}}^\top\Bar{\mathbf{b}}-\mathbf{a}^*\mathbf{\Bar{b}}^{*^\top}\Bar{\mathbf{b}})\rVert} \label{proj-simp}
    \end{align}

    The third term of the numerator is canceled since $\mathbf{P}\mathbf{a}^* = (\mathbf{I}_d-\mathbf{a}^*(\mathbf{a}^*)^\top)\mathbf{a}^*=0$. Thus,
    \begin{equation}
        \|\mathbf{P}\mathbf{a}^{t+1}\| \leq \frac{\|\mathbf{P}\mathbf{a}^t\| |1- 2\eta \mathbf{\Bar{b}}^{\top}\Bar{\mathbf{b}}|}{\lVert \mathbf{a}^t- 2\eta (\mathbf{a}^t\Bar{\mathbf{b}}^\top\Bar{\mathbf{b}}-\mathbf{a}^*\mathbf{\Bar{b}}^{*^\top}\Bar{\mathbf{b}})\rVert}
    \end{equation}

    Let $\delta_t = |\sin \theta (\mathbf{a}^*, \mathbf{a}^t)|$. Equation~\eqref{projection} becomes:
    \begin{align}
        \delta_{t+1} &\leq  \delta^{t} \frac{ |1- 2\eta \mathbf{\Bar{b}}^{\top}\Bar{\mathbf{b}}|}{\lVert \mathbf{a}^t- 2\eta (\mathbf{a}^t\Bar{\mathbf{b}}^\top\Bar{\mathbf{b}}-\mathbf{a}^*\mathbf{\Bar{b}}^{*^\top}\Bar{\mathbf{b}})\rVert} \\
        & = \delta_{t} \frac{|1- 2\eta \mathbf{\Bar{b}}^{\top}\Bar{\mathbf{b}}|}{\lVert \mathbf{a}^t(1- 2\eta \mathbf{\Bar{b}}^{\top}\Bar{\mathbf{b}}) +2\eta\mathbf{a}^*\mathbf{\Bar{b}}^{*^\top}\Bar{\mathbf{b}}\rVert} \\
        & = \delta_{t} C
    \end{align}

  Obviously $C \ge 0$. We drop the superscript $t$ when it is clear from context. Note that we have
  \begin{align}
      C^2 &= \frac{|1- 2\eta \mathbf{\Bar{b}}^{\top}\Bar{\mathbf{b}}|^2}{\lVert \mathbf{a}(1- 2\eta \mathbf{\Bar{b}}^{\top}\Bar{\mathbf{b}}) +2\eta\mathbf{a}^*\mathbf{\Bar{b}}^{*^\top}\Bar{\mathbf{b}}\rVert^2} \\
      & = \frac{|1- 2\eta \mathbf{\Bar{b}}^{\top}\Bar{\mathbf{b}} |^2}{(1- 2\eta \mathbf{\Bar{b}}^{\top}\Bar{\mathbf{b}})^2 \mathbf{a}^\top\mathbf{a} + 4\eta^2 (\Bar{\mathbf{b}}^{*^\top}\Bar{\mathbf{b}})^2 + 4\eta (1- 2\eta \mathbf{\Bar{b}}^{\top}\Bar{\mathbf{b}}) \mathbf{a}^\top \mathbf{a}^* \Bar{\mathbf{b}}^{*^\top}\Bar{\mathbf{b}} } \\
      & = \frac{|1- 2\eta \mathbf{\Bar{b}}^{\top}\Bar{\mathbf{b}}|^2}{(1- 2\eta \mathbf{\Bar{b}}^{\top}\Bar{\mathbf{b}})^2  + 4\eta^2 (\Bar{\mathbf{b}}^{*^\top}\Bar{\mathbf{b}})^2  + 4\eta (1- 2\eta \mathbf{\Bar{b}}^{\top}\Bar{\mathbf{b}})\mathbf{a}^\top \mathbf{a}^* \Bar{\mathbf{b}}^{*^\top}\Bar{\mathbf{b}} }  \label{ineq-C}
  \end{align}
  Recall that $\Bar{\mathbf{b}} = \Bar{\mathbf{b}}^* \mathbf{a}^{*^\top}\mathbf{a} = \Bar{\mathbf{b}}^* \cos \theta (\mathbf{a}^*, \mathbf{a})$, \eqref{ineq-C} becomes:
  \begin{align}
      C^2 &= \frac{|1- 2\eta \mathbf{\Bar{b}}^{\top}\Bar{\mathbf{b}}|^2}{(1- 2\eta \mathbf{\Bar{b}}^{\top}\Bar{\mathbf{b}})^2  + 4\eta^2 (\Bar{\mathbf{b}}^{*^\top}\Bar{\mathbf{b}})^2  + 4\eta (1- 2\eta \mathbf{\Bar{b}}^{\top}\Bar{\mathbf{b}})\mathbf{a}^\top \mathbf{a}^* \Bar{\mathbf{b}}^{*^\top}\Bar{\mathbf{b}} }  \\
      & = \frac{|1- 2\eta \mathbf{\Bar{b}}^{\top}\Bar{\mathbf{b}}|^2}{1  +  4 \eta^2 \mathbf{\Bar{b}}^{\top}\Bar{\mathbf{b}} (\mathbf{\Bar{b}}^{*^\top}\Bar{\mathbf{b}}^*  - \mathbf{\Bar{b}}^{\top}\Bar{\mathbf{b}})} \\
      & \leq (1-2\eta \mathbf{\Bar{b}}^{\top}\Bar{\mathbf{b}})^2 \label{ineq-C-2}\\
      &= (1-2\eta \|\Bar{\mathbf{b}}\|^2)^2 \label{ineq-C-3}
  \end{align}
  where \eqref{ineq-C-2} holds because $
\mathbf{\Bar{b}}^{*^\top}\Bar{\mathbf{b}}^*  - \mathbf{\Bar{b}}^{\top}\Bar{\mathbf{b}} =(1-\cos^2 \theta (\mathbf{a}^*, \mathbf{a}))\mathbf{\Bar{b}}^{*^\top}\Bar{\mathbf{b}}^*  \geq 0$. Equation \eqref{ineq-C-3} implies $C\le 1-2\eta \|\Bar{\mathbf{b}}\|^2$  if $2\eta \|\Bar{\mathbf{b}}\|^2\le 1$, which can be ensured by choosing a proper step size $\eta \leq \frac{1}{2L_{max}^2} \leq \frac{1}{2\|\Bar{\mathbf{b}}\|^2}$. Now by the assumption that $\delta_t \leq \delta_{t-1} \leq \dots \leq \delta_0$,
\begin{align}
    C&\le 1-2\eta \|\Bar{\mathbf{b}}\|^2 \\
    &= 1-2\eta \cos^2 \theta (\mathbf{a}^*, \mathbf{a}) \|\Bar{\mathbf{b}}^*\|^2 \\
    &= 1-2\eta (1-(\delta^t)^2) \|\Bar{\mathbf{b}}^*\|^2 \\
    &\le 1-2\eta (1-(\delta^0)^2) \|\Bar{\mathbf{b}}^*\|^2 
\end{align}

Summarizing, we obtain $\delta^{t+1}\leq \delta^{t} C \le\delta^{t}  (1 - 2 \eta (1-(\delta^0)^2)\|\Bar{\mathbf{b}}^*\|^2)$.
\end{proof}
\begin{proposition} \label{fa-saturate}(FFA-LoRA lower bound)
Suppose we are in the setting described in Section~\ref{rank-1-heter-clients}. For any set of ground truth parameters ($\mathbf{a}^*,\{\mathbf{b}_i^*\}_{i=1}^N$), initialization $\mathbf{a}^0$, initial angle distance $\delta_0\in (0,1)$, we apply Freezing-A scheme to obtain a shared global model ($\mathbf{a}^0,\mathbf{b}^{FFA}$), where ${\mathbf{b}}^{FFA} = {\mathbf{b}}^*{\mathbf{a}^*}^\top \mathbf{a}^0$. The global loss is 
\begin{equation}
    \frac{1}{N}\sum_{i=1}^N l_i(\mathbf{a}^0,{\mathbf{b}}^{FFA}) = \gamma^2 + \| \mathbf{\Bar{b}}^*\|^2\delta_0^2 \label{fa-lower-bound}
\end{equation}
\end{proposition}
\begin{proof}
    Through single step of minimization on $\mathbf{b}_i$ and corresponding aggregation, the minimum of the global objective is reached by FFA-LoRA. ${\mathbf{b}}^{FFA}$ is obtained through:
    \begin{align}
        \mathbf{b}_i &= \frac{\mathbf{b}_i^* \mathbf{a^*}^\top \mathbf{a}^0}{\mathbf{a}^{0^\top}\mathbf{a}^0} = \mathbf{b}_i^* \mathbf{a^*}^\top \mathbf{a}^0 \\
        {\mathbf{b}}^{FFA} &= \frac{1}{N} \sum_{i=1}^{N} \mathbf{b}_i = \Bar{\mathbf{b}}^*\mathbf{a^*}^\top \mathbf{a}^0
    \end{align}
    Next we compute the global loss with a shared global model $(\mathbf{a}^0, \Bar{\mathbf{b}}^{FFA})$. Note that we use Tr$(.)$ to denote the trace of a matrix.
    \begin{align}
        &\frac{1}{N} \sum_{i=1}^{N} l_i(\mathbf{a}^0, {\mathbf{b}}^{FFA}) \\&= \frac{1}{N} \sum_{i=1}^{N} \lVert \mathbf{a}^* (\mathbf{b}_i^*)^\top - \mathbf{a}^0 ({\mathbf{b}}^{FFA})^\top  \rVert^2 \\
        & = \frac{1}{N} \sum_{i=1}^{N} \lVert \mathbf{a}^* (\mathbf{b}_i^*)^\top -\mathbf{a}^*(\Bar{\mathbf{b}}^*)^\top+ \mathbf{a}^*(\Bar{\mathbf{b}}^* )^\top- \mathbf{a}^0 ({\mathbf{b}}^{FFA})^\top  \rVert^2  \\
        & = \frac{1}{N} \sum_{i=1}^{N} (\lVert \mathbf{a}^* (\mathbf{b}_i^* )^\top-\mathbf{a}^*(\Bar{\mathbf{b}}^*)^\top \rVert^2 + \lVert \mathbf{a}^*(\Bar{\mathbf{b}}^*)^\top - \mathbf{a}^0 ({\mathbf{b}}^{FFA} )^\top \rVert^2 \nonumber \\
        &+ 2\text{Tr}((\mathbf{a}^* (\mathbf{b}_i^*)^\top -\mathbf{a}^*(\Bar{\mathbf{b}}^*)^\top)^\top(\mathbf{a}^*(\Bar{\mathbf{b}}^*)^\top - \mathbf{a}^0( {\mathbf{b}}^{FFA})^\top )) \\
        & = \frac{1}{N} \sum_{i=1}^{N} (\lVert \mathbf{a}^* (\mathbf{b}_i^*)^\top -\mathbf{a}^*(\Bar{\mathbf{b}}^*)^\top \rVert^2 + \lVert \mathbf{a}^*(\Bar{\mathbf{b}}^*)^\top - \mathbf{a}^0 ( {\mathbf{b}}^{FFA})^\top \rVert^2) \nonumber\\
       & + 2\text{Tr}((\mathbf{a}^* \frac{1}{N} \sum_{i=1}^{N} (\mathbf{b}_i^*)^\top -\mathbf{a}^*(\Bar{\mathbf{b}}^*)^\top)^\top(\mathbf{a}^*(\Bar{\mathbf{b}}^*)^\top - \mathbf{a}^0( {\mathbf{b}}^{FFA})^\top ))  \\
        & = \frac{1}{N} \sum_{i=1}^{N} (\lVert  \mathbf{a}(\mathbf{b}_i^*-\Bar{\mathbf{b}}^*)^\top \rVert^2 + \lVert \mathbf{a}^*(\Bar{\mathbf{b}}^*)^\top -  \mathbf{a}^0 {\mathbf{a}^0}^\top \mathbf{a}^* (\Bar{\mathbf{b}}^*)^\top \rVert^2) \label{heter-loss-6}  \\
        & = \frac{1}{N} \sum_{i=1}^{N} \lVert \mathbf{b}_i^*-\Bar{
        \mathbf{b}}^*\rVert^2 + \frac{1}{N} \sum_{i=1}^{N} \lVert(\mathbf{I}_d-\mathbf{a}^0\mathbf{a}^{0^\top})\mathbf{a}^*(\Bar{\mathbf{b}}^*)^\top \rVert^2  \label{heter-loss-7}\\
        & = \frac{1}{N} \sum_{i=1}^{N} \lVert \mathbf{b}_i^*-\Bar{
        \mathbf{b}}^*\rVert^2 + \frac{1}{N} \sum_{i=1}^{N} \| (\mathbf{I}_d-\mathbf{a}^0\mathbf{a}^{0^\top})\mathbf{a}^*\|^2 \|\Bar{\mathbf{b}}^* \|^2  \label{heter-loss-8}\\
        &=\gamma^2 +\|\Bar{\mathbf{b}}^* \|^2 \delta_0^2 \label{geq-1}
    \end{align} 
    where \eqref{heter-loss-6} holds since the last term is 0, \eqref{heter-loss-7} and \eqref{heter-loss-8} hold since $\|\mathbf{u}\mathbf{v}^\top\|=\|\mathbf{u}\|\cdot \|\mathbf{v}^\top \|$ for vector $\mathbf{u}$ and $\mathbf{v}$, \eqref{geq-1} holds because of Definition~\ref{client-var}.
\end{proof}

\paragraph{Proof of Theorem A.14}
\begin{proof}
    In order to achieve $\epsilon$-recovery of $\mathbf{a}^*$, we need
    \begin{align}
        \delta^{0} ({1-c(1-(\delta^{0})^2)})^{T} &\leq \epsilon \\
        ({1-c (1-(\delta^{0})^2)})^{T} &\leq \frac{\epsilon}{\delta^0} \\
        T \log{({1-c (1-(\delta^{0})^2)})} &\leq \log (\frac{\epsilon}{\delta^0})  \\
        \end{align} 
        We proceed such that
        \begin{align}
        T &\geq \frac{\log (\frac{\epsilon}{\delta^0})}{\log{({1-c (1-(\delta^{0})^2)})}} \\
         &> \frac{\log (\frac{\epsilon}{\delta^0})}{-c (1-(\delta^{0})^2)} \label{T-3} \\
         &= \frac{1}{c (1-(\delta^{0})^2)} \log(\frac{\delta^0}{\epsilon}) 
    \end{align} 
    where \eqref{T-3} follows by using $\log(1-x)<-x$ for $|x|<1$.

Now we show the convergence to the global minimizer.
Recall that $\mathbf{b}^{T+1} = \Bar{\mathbf b}^{*}\mathbf{a}^{*^\top}\mathbf{a}^T$ and $\delta^T = \|  (\mathbf I_d - \mathbf{a}^T(\mathbf a^T)^\top )  \mathbf{a}^* \|$, we have
    \begin{align}
        \lVert \mathbf{a}^T({\mathbf{b}}^{T+1})^\top- \mathbf{a}^* \Bar{\mathbf b}^{*^\top}\rVert &= \lVert \mathbf{a}^T (\mathbf{a}^T)^\top \mathbf{a}^* \Bar{\mathbf b}^{*^\top}  - \mathbf{a}^* \Bar{\mathbf b}^{*^\top}\rVert \\
        & = \| (\mathbf{a}^T (\mathbf{a}^T)^\top-\mathbf{I}_d) \mathbf{a}^* \Bar{ \mathbf b}^{*^\top}\|\\
        & = \| (\mathbf{I}_d-\mathbf{a}^T (\mathbf{a}^T)^\top) \mathbf{a}^*\| \cdot \|\Bar{\mathbf b}^{*} \|\\
        &\leq \epsilon \|\Bar{\mathbf b}^{*} \| \\
        & =  \epsilon \| \mathbf{a}^* \Bar{\mathbf b}^{*^\top} \label{eq:conv_target_2_1}\|
    \end{align}
    where \eqref{eq:conv_target_2_1} is due to the fact that $\| \mathbf x \mathbf y^\top\| =\| \mathbf x \|\| \mathbf y\| $ and $\|\mathbf a^*\| = 1$.
    
\end{proof}
Proposition~\ref{fa-saturate} shows that for any $\delta_0 \in (0,1)$, the global objective of FFA-LoRA is given by \eqref{geq-1}, comprising two terms: $\gamma^2$, reflecting the heterogeneity of $\{\mathbf{b}_i^*\}_{i=1}^N$, and $\|\Bar{\mathbf{b}}^* \|^2 \delta_0^2$, due to the angular distance between $\mathbf{a}^0$ and $\mathbf{a}^*$. By Theorem~\ref{convergence}, RoLoRA achieves $\epsilon$-accurate recovery of the global minimizer, with global loss upper bounded by $\gamma^2 + \|\Bar{\mathbf{b}}^* \|^2\epsilon^2$, since RoLoRA reduces the angular distance loss from $\|\Bar{\mathbf{b}}^* \|^2\delta_0^2$ to $\|\Bar{\mathbf{b}}^* \|^2\epsilon^2$. We can make $\epsilon$ arbitrarily small by increasing the iterations.


\newpage
\section{Convergence Analysis of Non-Convex Case} \label{app:convergence-non-convex}
We follow the approach of Li et al.~\cite{Li2020On} to demonstrate the convergence of RoLoRA (Algorithm~\ref{alg:rolora-llm}) in smooth, non-convex landscapes. Assumptions~\ref{assumption1} and~\ref{lipschitz} are standard and commonly employed in the convergence analysis of federated learning. Assumption~\ref{assumption3} is adapted from FedSA-LoRA~\cite{guo2025selective}, which proposes a personalized federated fine-tuning framework that maintains local diversity in $\mathbf{B}$ while aggregating only $\mathbf{A}$ through simultaneous updates of both $\mathbf{A}$ and $\mathbf{B}$. In contrast, our work focuses on a single global model with alternating optimization of $\mathbf{A}$ and $\mathbf{B}$.

In Assumption~\ref{assumption3}, bounding Frobenius norms of $\mathbf{A}$ and $\mathbf{B}$ is a standard weight‑regularity requirement in LoRA fine‑tuning practice, where small rank and scaling factors keep the adapters from exploding. The inner‑product conditions simply posit that each adapter has enough non‑degenerate singular values, requiring that the low-rank updates retain sufficient rank and alignment with local gradient directions. Formally, we require the smallest singular values of $\mathbf{A}_i^{t}$ and $\mathbf{B}_i^{t}$ to be lower bounded by $\sqrt{c_A}$ and $\sqrt{c_B}$.


\begin{assumption}[Bounded Stochastic Gradient] \label{assumption1}
Let a mini‐batch $
\mathbf{x}_{i},
$
be drawn uniformly at random from client $i$’s dataset, meaning $\mathbb{E}_{\mathbf{x}_{i}} [\nabla_{\mathbf{W}_i}\,l_i\bigl(\mathbf{W}_i,\,\mathbf{x}_{i} \bigr) ]= \nabla_{\mathbf{W}_i}\,l_i\bigl(\mathbf{W}_i\bigr)$. We assume that the expected squared norm of any stochastic gradient is uniformly bounded, that is,
\[
\mathbb{E}_{\mathbf{x}_{i}}\!\bigl\|\nabla_{\mathbf{W}}\,l_i\bigl(\mathbf{W},\,\mathbf{x}_{i}\bigr)\bigr\|^2
\;\le\;G^2
\]
where the expectation is over the random draw of the mini-batch $\mathbf{x}_{i}$, and $G>0$ is a constant.
\end{assumption}

\begin{assumption}[Lipschitz smooth] \label{lipschitz}
    Loss functions ${l}_1, \cdots, {l}_N$ are all $L$-smooth. For all weights $\mathbf{W}$ and $\mathbf{U}$:
    \begin{equation}
        l_i(\mathbf{V}) \leq l_i(\mathbf{W}) + \langle \nabla_{\mathbf{W}} l_i(\mathbf{W}), \mathbf{V}-\mathbf{W} \rangle_F + \frac{L}{2} \|\mathbf{V}-\mathbf{W}\|_F^2, \forall i \in [N] \nonumber
    \end{equation}
    
\end{assumption}

\begin{assumption} \label{assumption3}
Let $\mathbf{W}_i=\mathbf{W}_0+\mathbf{B}_i\mathbf{A}_i$ represent the model parameters for the $i$-th client. There exist constants $C_B>0$, $C_A>0$, $c_B > 0$, and $c_A > 0$ such that:
\begin{equation*}
    \begin{aligned}
        \|\mathbf{B}_i\|_F &\leq C_B, \\
        \|\mathbf{A}_i\|_F &\leq C_A, \\
        \langle {\mathbf{A}_i}{\mathbf{A}_i}^\top, \nabla_{\mathbf{W}} {l}_i(\mathbf{W}_i)\nabla_{\mathbf{W}} {l}_i(\mathbf{W}_i)^{\top}\rangle_F &\geq c_A \|\nabla_{\mathbf{W}} {l}_i(\mathbf{W}_i)\|_F^2, \\
        \langle {\mathbf{B}_i}^{\top}\mathbf{B}_i, \nabla_{\mathbf{W}}{l}_i(\mathbf{W}_i)^{\top}\nabla_{\mathbf{W}} {l}_i(\mathbf{W}_i)\rangle_F & \geq c_B \|\nabla_{\mathbf{W}} {l}_i(\mathbf{W}_i)\|_F^2,
    \end{aligned}
\end{equation*}
for all $i\in[N]$.
\end{assumption}

\begin{theorem}[Convergence to the stationary point] \label{convergence-non-convex}
    Let Assumption ~\ref{assumption1}, ~\ref{lipschitz}, and ~\ref{assumption3} hold. Suppose each client runs $2T$ rounds, each consisting of $Q$ local epochs, using a learning rate $\eta \propto O(1/\sqrt{T})$, then we obtain:
    \begin{align}
        \min_{0\leq t\le 2T}  \mathbb{E}[\|\nabla_{\mathbf{W}}l_i(\mathbf{W}^{t})\|_F^2] \leq \frac{ \Delta_i}{2T\eta c_{min}} + \frac{D\eta}{2  c_{min}}
    \end{align}
    where $\Delta_i = \mathbb{E}[l_i(\mathbf{W}^{0})] -l_i^*, c_{min} = \min (c_A ,c_B), $ $D$ is chosen such that $ D \eta^2 \geq D_A+ D_B$, and $D_A = L \eta C_A^2 Q^2 G^2 + \frac{1}{2} \eta G^2 + 2L \eta^2 C_A^2 Q^2 G^2 + \frac{L}{2} \eta^2 C_A^4 G^2 ,
    D_B = L \eta C_B^2 Q^2 G^2 + \frac{1}{2} \eta G^2 + 2L \eta^2 C_B^2 Q^2 G^2 + \frac{L}{2} \eta^2 C_B^4 G^2$.
\end{theorem}

According to Theorem~\ref{convergence-non-convex}, we achieve an $O\left(\frac{1}{\sqrt{T}}\right)$ convergence rate toward a stationary point under smooth, non-convex conditions, matching the convergence rate of FedAVG in the same setting. We follow a similar derivation framework to \cite{guo2025selective}, adopting similar proof techniques, such as applying Assumption~\ref{lipschitz} to the global and local models (Eq.\eqref{eq:lipschitz-1} and Eq.\eqref{eq:lipschitz-2}). However, due to structural differences between our algorithm and that of \cite{guo2025selective}—in particular, our use of alternating optimization rather than simultaneous updates—the derivation diverges in how key gradient bounds are applied and combined. As a result, even though the proof steps are analogous, our final convergence bound depends on $\min(c_A, c_B)$ and features decoupled $\mathbf{A}$- and $\mathbf{B}$-related terms, in contrast to the $c_A + c_B$ dependence along with cross-terms due to the simultaneous updates of $\mathbf{A}$ and $\mathbf{B}$ in \cite{guo2025selective}. Although FedSA-LoRA reports the same $O\left(\frac{1}{\sqrt{T}}\right)$ convergence rate, the two algorithms address fundamentally different FL scenarios: FedSA-LoRA optimizes personalized client models, whereas RoLoRA learns a single shared global model.

\begin{proof}
    Let $\mathbf{W}^{2t} = \mathbf{W}_0 + \mathbf{A}^t\mathbf{B}^t$ be the global model parameters at the $2t$-th communication round. Let $\mathbf{W}_{i,q}^{2t} = \mathbf{W}_0 + \mathbf{A}^t\mathbf{B}_{i,q}^t$ be the local model parameters of client $i$ at the $q$-th local epoch of the $2t$-th communication round, where each client performs a total of $Q$ local epochs. We define the following for convenience:

    \begin{align}
       & \text{The $2t$-th communication round:} \nonumber\\ &\quad \quad \quad \quad \mathbf{W}^{2t} = \mathbf{W}_0 + \mathbf{A}^t\mathbf{B}^t \label{w-2t}\\
        &\text{The $q$-th local epoch of the $2t$-th communication round:} \nonumber \\
        &\quad \quad \quad \quad  \mathbf{W}_{i,q}^{2t} = \mathbf{W}_0 + \mathbf{A}^t\mathbf{B}_{i,q}^t \label{w-2t-q} \\
        &\text{The $2t+1$-th communication round:} \nonumber \\
        &\quad \quad \quad \quad  \mathbf{W}^{2t+1} = \mathbf{W}_0 + \mathbf{A}^t\mathbf{B}^{t+1}    =   \mathbf{W}_0 + \frac{1}{N}\sum_{i=1}^{N}\mathbf{A}^t\mathbf{B}_{i,Q}^{t} \label{w-2t+1} \\
        &\text{The $q$-th local epoch of the $2t+1$-th communication round:} \nonumber \\
         &\quad \quad \quad \quad  \mathbf{W}_{i,q}^{2t+1} = \mathbf{W}_0 + \mathbf{A}_{i,q}^t\mathbf{B}^{t+1} \label{w-2t+1-q} \\
         &\text{The $2t+2$-th communication round:} \nonumber \\
         &\quad \quad \quad \quad  \mathbf{W}^{2t+2} =  \mathbf{W}_0 + \mathbf{A}^{t+1}\mathbf{B}^{t+1} = \mathbf{W}_0 + \frac{1}{N}\sum_{i=1}^{N}\mathbf{A}_{i,Q}^t\mathbf{B}^{t+1}\label{w-2t+2}  
    \end{align}

    According to chain rule, 
    \begin{align}
       \nabla_{\mathbf{B}} l_i(\mathbf{W}) =  \mathbf{A}^\top \nabla_{\mathbf{W}} l_i(\mathbf{W}) \\
       \nabla_{\mathbf{A}} l_i(\mathbf{W}) =  \nabla_{\mathbf{W}} l_i(\mathbf{W})\mathbf{B}^\top
    \end{align}

Now we apply Assumption~\ref{lipschitz} to local update of $\mathbf{B}$, and get
\begin{align}
    l_i(\mathbf{W}_{i,1}^{2t}) \leq l_i(\mathbf{W}^{2t}) + \langle \mathbf{W}_{i,1}^{2t}- \mathbf{W}^{2t}, \nabla_{\mathbf{W}} l_i(\mathbf{W}^{2t})\rangle_F + \frac{L}{2} \| \mathbf{W}_{i,1}^{2t} - \mathbf{W}^{2t} \|_F^2 \label{eq:lipschitz-1}
\end{align}
By \eqref{w-2t} and \eqref{w-2t-q}, 
\begin{align}
    \mathbf{W}_{i,1}^{2t} - \mathbf{W}^{2t} &= \mathbf{A}^t (\mathbf{B}_{i,1}^t-\mathbf{B}^t) \label{eq:gd-fixA}\\
    & = -\eta \mathbf{A}^t {\mathbf{A}^t}^\top \nabla_{\mathbf{W}} l_i(\mathbf{W}^{2t},\mathbf{x}_{i,1}^{2t})
\end{align}
where $\eta$ is the learning rate. Then
\begin{align}
    l_i(\mathbf{W}_{i,1}^{2t}) \leq l_i(\mathbf{W}^{2t}) -\eta \langle \mathbf{A}^t {\mathbf{A}^t}^\top \nabla_{\mathbf{W}} l_i(\mathbf{W}^{2t},\mathbf{x}_{i,1}^{2t}), \nabla_{\mathbf{W}} l_i(\mathbf{W}^{2t})\rangle_F \nonumber\\+ \frac{L}{2} \eta^2 \|  \mathbf{A}^t {\mathbf{A}^t}^\top \nabla_{\mathbf{W}} l_i(\mathbf{W}^{2t},\mathbf{x}_{i,1}^{2t}) \|_F^2 \label{update1}
\end{align}
Taking expectation of \eqref{update1},
\begin{align}
    \mathbb{E}[l_i(\mathbf{W}_{i,1}^{2t})] \leq \mathbb{E}[ l_i(\mathbf{W}^{2t})] - \eta \langle \mathbf{A}^t {\mathbf{A}^t}^\top \nabla_{\mathbf{W}} l_i(\mathbf{W}^{2t}), \nabla_{\mathbf{W}} l_i(\mathbf{W}^{2t})\rangle_F + \frac{L}{2} \eta^2  \mathbb{E}[\|  \mathbf{A}^t {\mathbf{A}^t}^\top \nabla_{\mathbf{W}} l_i(\mathbf{W}^{2t},\mathbf{x}_{i,1}^{2t}) \|_F^2 ] \label{exp-1}
\end{align}
The inner product term is lower bounded such that
\begin{align}
    \langle \mathbf{A}^t {\mathbf{A}^t}^\top \nabla_{\mathbf{W}} l_i(\mathbf{W}^{2t}), \nabla_{\mathbf{W}} l_i(\mathbf{W}^{2t})\rangle_F & = \text{Tr}[{\nabla_{\mathbf{W}} l_i (\mathbf{W}^{2t})}^\top \mathbf{A}^t {\mathbf{A}^t}^\top \nabla_{\mathbf{W}} l_i (\mathbf{W}^{2t})] \\
    & = \text{Tr}[ \mathbf{A}^t {\mathbf{A}^t}^\top \nabla_{\mathbf{W}} l_i (\mathbf{W}^{2t}){\nabla_{\mathbf{W}} l_i (\mathbf{W}^{2t})}^\top] \\
     & =  \langle \mathbf{A}^t {\mathbf{A}^t}^\top , \nabla_{\mathbf{W}} l_i(\mathbf{W}^{2t}){\nabla_{\mathbf{W}} l_i(\mathbf{W}^{2t})}^\top \rangle_F\\
    & \geq  c_A \| \nabla_{\mathbf{W}} l_i(\mathbf{W}^{2t})  \|_F^2 \label{eq:ip-1}
\end{align}
where \eqref{eq:ip-1} follows by Assumption~\ref{assumption3}. Moreover,
\begin{align}
    \frac{L}{2} \eta^2  \mathbb{E}[\|  \mathbf{A}^t {\mathbf{A}^t}^\top \nabla_{\mathbf{W}} l_i(\mathbf{W}^{2t},\mathbf{x}_{i,1}^{2t}) \|_F^2 ] &\leq  \frac{L}{2} \eta^2  \mathbb{E}[\| \mathbf{A}^t\|_F^4 \cdot \|\nabla_{\mathbf{W}} l_i(\mathbf{W}^{2t},\mathbf{x}_{i,1}^{2t})\|_F^2] \label{quadr-1} \\
    &  \leq \frac{L}{2} \eta^2  \mathbb{E}[C_A^4 \cdot \|\nabla_{\mathbf{W}}  l_i(\mathbf{W}^{2t},\mathbf{x}_{i,1}^{2t})\|_F^2] \\
    & =  \frac{L}{2} \eta^2 C_A^4\mathbb{E}[ \|\nabla_{\mathbf{W}} l_i(\mathbf{W}^{2t},\mathbf{x}_{i,1}^{2t})\|_F^2] \\
    & \leq \frac{L}{2} \eta^2 C_A^4 G^2 \label{quadr-2}
\end{align}
where \eqref{quadr-1} follows by Assumption~\ref{assumption3}, \eqref{quadr-2} follows by Assumption~\ref{assumption1}. Combining \eqref{exp-1}, \eqref{eq:ip-1}, and \eqref{quadr-2}, we get
\begin{align}
    \mathbb{E}[l_i(\mathbf{W}_{i,1}^{2t})] \leq \mathbb{E}[ l_i(\mathbf{W}^{2t})] - \eta c_A \| \nabla_{\mathbf{W}} l_i(\mathbf{W}^{2t})  \|_F^2  + \frac{L}{2} \eta^2 C_A^4 G^2 \label{smooth-1-results}
\end{align}

Next we apply Assumption~\ref{assumption3} to aggregation step and take expectation on both sides,
\begin{align}
     \mathbb{E}[l_i(\mathbf{W}^{2t+1})] \leq \mathbb{E}[l_i(\mathbf{W}^{2t}_{i,1})] + \mathbb{E}[\langle \mathbf{W}^{2t+1} - \mathbf{W}^{2t}_{i,1}, \nabla_{\mathbf{W}} l_i(\mathbf{W}^{2t}_{i,1})\rangle_F ]+ \frac{L}{2} \mathbb{E}[ \| \mathbf{W}^{2t+1} - \mathbf{W}^{2t}_{i,1} \|_F^2] \label{eq:lipschitz-2}
\end{align}
where 
\begin{align}
     \mathbf{W}^{2t+1} - \mathbf{W}^{2t}_{i,1} &= \frac{1}{N}\sum_{j=1}^{N} \mathbf{A}^t \mathbf{B}^t_{j,Q}-\mathbf{A}^t \mathbf{B}^t_{i,1} \label{eq:aggre-fixA}\\
     & = \frac{1}{N}\sum_{j=1}^{N} \mathbf{A}^t ( \mathbf{B}^t_{j,Q}-\mathbf{B}^t_{i,1})  \label{smooth-2}
\end{align}
We have
\begin{align}
    \mathbf{B}^t_{j,Q} &= \mathbf{B}^t - \eta \sum_{q=0}^{Q-1} \nabla_{\mathbf{B}} l_j(\mathbf{W}^{2t}_{j,q},\mathbf{x}_{j,q}^{2t}) \\
    \mathbf{B}^t_{i,1} &= \mathbf{B}^t - \eta \nabla_{\mathbf{B}} l_i (\mathbf{W}^{2t},\mathbf{x}_{i,1}^{2t}) \\
    \mathbf{B}^t_{j,Q} - \mathbf{B}^t_{i,1} &= \eta \sum_{q=0}^{Q-1} (\nabla_{\mathbf{B}} l_i (\mathbf{W}^{2t},\mathbf{x}_{i,1}^{2t}) - \nabla_{\mathbf{B}} l_j(\mathbf{W}^{2t}_{j,q},\mathbf{x}_{j,q}^{2t})) \\
    & = \eta {\mathbf{A}^t}^\top \sum_{q=0}^{Q-1}(\nabla_{\mathbf{W}} l_i (\mathbf{W}^{2t},\mathbf{x}_{i,1}^{2t}) - \nabla_{\mathbf{W}} l_j(\mathbf{W}^{2t}_{j,q},\mathbf{x}_{j,q}^{2t}))
\end{align}
Thus,
\begin{align}
     \mathbf{W}^{2t+1} - \mathbf{W}^{2t}_{i,1} &= \frac{\eta}{N} \mathbf{A}^t \sum_{j=1}^{N} \sum_{q=0}^{Q-1} (\nabla_{\mathbf{W}} l_i (\mathbf{W}^{2t},\mathbf{x}_{i,1}^{2t}) - \nabla_{\mathbf{W}} l_j(\mathbf{W}^{2t}_{j,q},\mathbf{x}_{j,q}^{2t}))
\end{align}
Therefore, 
\begin{align}
 \left\|\mathbf{W}^{2t+1} - \mathbf{W}^{2t}_{i,1} \right\|_F^2 
&=  \left\|\frac{\eta}{N}\, \mathbf{A}^t \sum_{j=1}^{N} \sum_{q=0}^{Q-1} \left( \nabla_{\mathbf{W}} l_i (\mathbf{W}^{2t},\mathbf{x}_{i,1}^{2t}) - \nabla_{\mathbf{W}} l_j(\mathbf{W}^{2t}_{j,q},\mathbf{x}_{j,q}^{2t}) \right) \right\|_F^2  \\
&\le \frac{\eta^2}{N^2} \|\mathbf{A}^t\|_F^2 \left\| \sum_{j=1}^{N} \sum_{q=0}^{Q-1} \left( \nabla_{\mathbf{W}} l_i (\mathbf{W}^{2t},\mathbf{x}_{i,1}^{2t}) - \nabla_{\mathbf{W}} l_j(\mathbf{W}^{2t}_{j,q},\mathbf{x}_{j,q}^{2t}) \right) \right\|_F^2 \\
&\le \frac{\eta^2}{N} C_A^2 \sum_{j=1}^{N} \left\| \sum_{q=0}^{Q-1} \left( \nabla_{\mathbf{W}} l_i(\mathbf{W}^{2t},\mathbf{x}_{i,1}^{2t}) - \nabla_{\mathbf{W}}l_j(\mathbf{W}^{2t}_{j,q},\mathbf{x}_{j,q}^{2t}) \right) \right\|_F^2 \\
& \le \frac{\eta^2}{N} C_A^2 Q \sum_{j=1}^{N} \sum_{q=0}^{Q-1} \left\| \nabla_{\mathbf{W}} l_i(\mathbf{W}^{2t},\mathbf{x}_{i,1}^{2t}) - \nabla_{\mathbf{W}} l_j(\mathbf{W}^{2t}_{j,q},\mathbf{x}_{j,q}^{2t}) \right\|_F^2
\end{align}
Taking expectation, 
\begin{align}
    \mathbb{E} \left[\left\|\mathbf{W}^{2t+1} - \mathbf{W}^{2t}_{i,1} \right\|_F^2  \right]
    & \le \frac{\eta^2}{N} C_A^2 Q \sum_{j=1}^{N} \sum_{q=0}^{Q-1} \mathbb{E} \left[ \left\| \nabla_{\mathbf{W}} l_i(\mathbf{W}^{2t},\mathbf{x}_{i,1}^{2t}) - \nabla_{\mathbf{W}} l_j(\mathbf{W}^{2t}_{j,q},\mathbf{x}_{j,q}^{2t})) \right\|_F^2 \right]
\end{align}
For any matrices $\mathbf{A}$ and $\mathbf{B}$ (or vectors), we have
\[
\| \mathbf A - \mathbf B\|_F^{2}
= \|\mathbf A\|_F^{2} + \|\mathbf B\|_F^{2} - 2\langle \mathbf A, \mathbf B\rangle_F
\;\le\; \|\mathbf A\|_F^{2} + \|\mathbf B\|_F^{2} + 2\|\mathbf A\|_F\,\|\mathbf B\|_F
\;\le\; 2\|\mathbf A\|_F^{2} + 2\|\mathbf B\|_F^{2}.
\]
Thus,
\begin{align}
    &\mathbb{E} \left[ \left\| \nabla_{\mathbf{W}} l_i(\mathbf{W}^{2t},\mathbf{x}_{i,1}^{2t}) - \nabla_{\mathbf{W}} l_j(\mathbf{W}^{2t}_{j,q},\mathbf{x}_{j,q}^{2t}) \right\|_F^2 \right] \nonumber\\&\leq 2 \mathbb{E} \left[ \| \nabla_{\mathbf{W}} l_i(\mathbf{W}^{2t},\mathbf{x}_{i,1}^{2t}) \|_F^2 \right] + 2 \mathbb{E} \left[  \| \nabla_{\mathbf{W}} l_j(\mathbf{W}^{2t}_{j,q},\mathbf{x}_{j,q}^{2t}) \|_F^2 \right] \\
    &\leq 4G^2
\end{align}
Leading to
\begin{align}
    \frac{L}{2} \mathbb{E} \left[\left\|\mathbf{W}^{2t+1} - \mathbf{W}^{2t}_{i,1} \right\|_F^2  \right] \leq 2L \eta^2 C_A^2 Q^2 G^2 \label{quadr-smooth-2}
\end{align}
For the inner product term of \eqref{eq:lipschitz-2}, we have
\begin{align}
    \mathbb{E} \left[ \langle \mathbf{W}^{2t+1} - \mathbf{W}^{2t}_{i,1}, \nabla_{\mathbf{W}} l_i(\mathbf{W}^{2t}_{i,1})\rangle_F \right] & \leq \frac{1}{2\eta} \mathbb{E} \left[ \left\|  \mathbf{W}^{2t+1} - \mathbf{W}^{2t}_{i,1} \right\|_F^2 \right] + \frac{1}{2} \eta \mathbb{E} \left[ \left\| \nabla_{\mathbf{W}} l_i(\mathbf{W}^{2t}_{i,1})  \right\|_F^2 \right] \\
    & \le L \eta C_A^2 Q^2 G^2 + \frac{1}{2} \eta G^2 \label{ip-smooth-2}
\end{align}
Combining \eqref{eq:lipschitz-2}, \eqref{quadr-smooth-2}, and \eqref{ip-smooth-2}, we obtain
\begin{align}
    \mathbb{E}[l_i(\mathbf{W}^{2t+1})] \leq \mathbb{E}[l_i(\mathbf{W}^{2t}_{i,1})] + L \eta C_A^2 Q^2 G^2 + \frac{1}{2} \eta G^2  + 2L \eta^2 C_A^2 Q^2 G^2 \label{smooth-2-results}
\end{align}
Combining \eqref{smooth-1-results} and \eqref{smooth-2-results}, we derive
\begin{align}
   & \eta c_A \| \nabla_{\mathbf{W}}l_i (\mathbf{W}^{2t})\|_F^2 \nonumber\\ 
   &  \le \mathbb{E}[l_i(\mathbf{W}^{2t})]-\mathbb{E}[l_i(\mathbf{W}^{2t+1})] + L \eta C_A^2 Q^2 G^2 + \frac{1}{2} \eta G^2 + 2L \eta^2 C_A^2 Q^2 G^2 + \frac{L}{2} \eta^2 C_A^4 G^2
\end{align}
Analogously, applying the same analysis to the $(2t+1)$-th communication round, which fixes $\mathbf{B}$ and updates $\mathbf{A}$, introduces key modifications to steps such as Eq.\eqref{eq:gd-fixA} and Eq.\eqref{eq:aggre-fixA}, which govern the weight updates. These changes propagate through the subsequent steps that depend on the updated weights. As a result, we obtain
\begin{align}
     & \eta c_B \| \nabla_{\mathbf{W}}l_i (\mathbf{W}^{2t+1})\|_F^2 \nonumber\\ 
   &  \le \mathbb{E}[l_i(\mathbf{W}^{2t})]-\mathbb{E}[l_i(\mathbf{W}^{2t+1})] + L \eta C_B^2 Q^2 G^2 + \frac{1}{2} \eta G^2 + 2L \eta^2 C_B^2 Q^2 G^2 + \frac{L}{2} \eta^2 C_B^4 G^2
\end{align}
Add the two inequalities and then sum over $t=0,1,\dots, T-1$, we get
\begin{align}
    \sum_{t=0}^{T-1} \eta (c_A \| \nabla_{\mathbf{W}}l_i (\mathbf{W}^{2t})\|_F^2 + c_B \| \nabla_{\mathbf{W}}l_i (\mathbf{W}^{2t+1})\|_F^2  ) \leq \mathbb{E}[l_i(\mathbf{W}^{0})]-\mathbb{E}[l_i(\mathbf{W}^{2T})] +T (D_A+D_B)
\end{align}
where
\begin{align}
    D_A = L \eta C_A^2 Q^2 G^2 + \frac{1}{2} \eta G^2 + 2L \eta^2 C_A^2 Q^2 G^2 + \frac{L}{2} \eta^2 C_A^4 G^2 \\
    D_B = L \eta C_B^2 Q^2 G^2 + \frac{1}{2} \eta G^2 + 2L \eta^2 C_B^2 Q^2 G^2 + \frac{L}{2} \eta^2 C_B^4 G^2
\end{align}
Assume the per‑client loss is bounded below by $l_i^*$, let 
\begin{align}
    \Delta_i = \mathbb{E}[l_i(\mathbf{W}^{0})] -l_i^*, \quad c_{min} = \min (c_A ,c_B), 
\end{align}
Choosing $D$ such that $D\eta^2\geq D_A+D_B$, then
\begin{align}
    \min_{0\leq t\le 2T}  \mathbb{E}[\|\nabla_{\mathbf{W}}l_i(\mathbf{W}^{t})\|_F^2] \leq \frac{ \Delta_i}{2T\eta c_{min}} + \frac{D\eta}{2  c_{min}}
\end{align}
We choose $\eta \propto O(1/\sqrt{T})$ so that the overall convergence rate with a diminishing step size is $ O(1/\sqrt{T})$ which matches the canonical convergence speed of stochastic gradient methods in non‑convex settings.
\end{proof}

\newpage

\section{Experiments} \label{app:additional-exp}

\subsection{Impact of Non-Linearity on RoLoRA}
Across both linear and non-linear settings, all methods perform similarly, with RoLoRA showing modest improvement in the non-linear case, likely due to its better utilization of the added expressiveness from ReLU.
\begin{figure}[h!]
\begin{center}
\begin{subfigure}
  \centering
\includegraphics[width=0.6\linewidth]{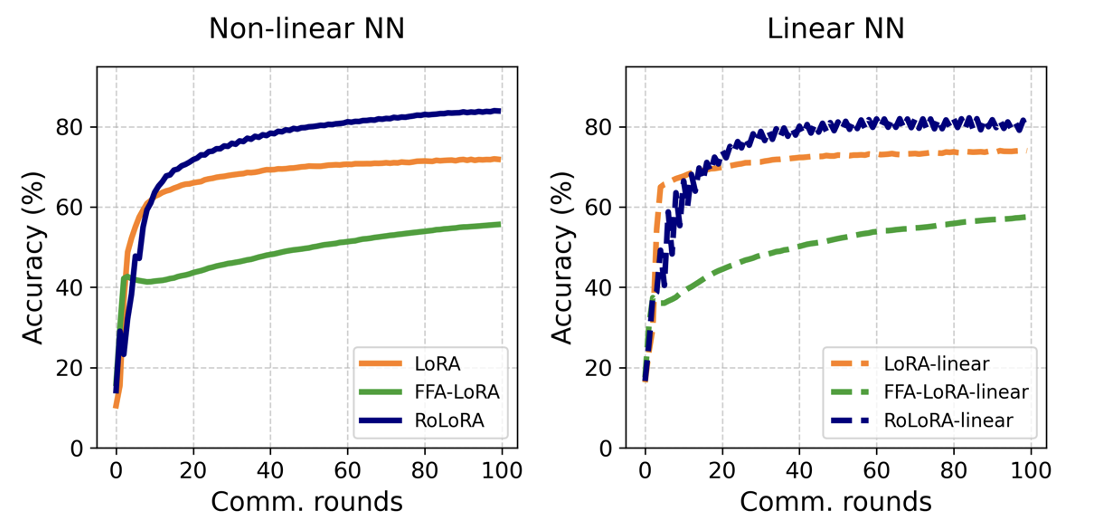}
\end{subfigure}
 \caption{Comparison of RoLoRA, LoRA, and FFA-LoRA on linear and non-linear networks. While overall performance is similar, RoLoRA shows modest gains in the non-linear setting, likely benefiting from ReLU’s added expressiveness.}
    \label{fig:non-linear-add-to-rolora}
\hfill
\end{center}
\end{figure}

\subsection{Hyper-parameters for GLUE task} \label{exp-setup}
\begin{table}[h]
    \centering
{\footnotesize
    \begin{tabular}{cccccc}
    \toprule
         & SST-2 & QNLI & MNLI & QQP & RTE \\
         \midrule
    Total comm. rounds     & 500&500&500&500&200 \\
    Batch Size &64&32&32&32&32\\
    Local Epochs & 20&20&20&20&20 \\
    \bottomrule
    \end{tabular}
     \vspace{4pt}
    \caption{Hyper-parameters configurations. Note the total communication rounds are for the setting with 3 clients. When increasing the number of clients, we decrease the total communication rounds accordingly to maintain a constant sample count used during fine-tuning }
    \label{tab:exp-set}}
\end{table}


We show the hyper-parameter configurations for each dataset in Table~\ref{tab:exp-set}. 

\subsection{Effect of Number of Clients}
\vspace{-0.3cm}
\paragraph{Configurations}
Table~\ref{tab:layer_type_index-client-num} shows the selected layer set attached with LoRA modules for Table~\ref{tab:Clients-num-flex-appendix}. We present Table~\ref{tab:Clients-num-flex-appendix} with the results of FlexLoRA~\cite{bai2024federatedfinetuninglargelanguage} added in Table~\ref{tab:Clients-num-flex-appendix}.
\begin{table}[h!]
\centering
{\footnotesize
    \begin{tabular}{ccccccc}
    \toprule
                            & Layer Attributes & SST-2 & QNLI & MNLI & QQP & RTE \\
                            \midrule
\multirow{2}{*}{$\mathcal{P}_2$} & Type             &     $W_v,W_q$   &  $W_v,W_q$     &   $W_v,W_q$    &   $W_v,W_q$   &   $W_v,W_q$   \\
                    & Index            &   $\{18,\ldots,23\}$    &   $\{15,\ldots,23\}$   &  $\{15,\ldots,23\}$    &   $\{15,\ldots,23\}$  &  $\{16,\ldots,23\}$    \\ \bottomrule
    \end{tabular}}
    \vspace{4pt}
    \caption{The selected layer set attached with LoRA modules for Table~\ref{tab:Clients-num-flex-appendix} and Table~\ref{tab:noniid-llm}}
    \label{tab:layer_type_index-client-num}
\end{table}

\paragraph{Rank-2 Results}
We show the effect of number of clients when using rank-2 LoRA modules in Table~\ref{tab:Clients-num-appendix-rank-2}.
\begin{table}[h!]
\centering
    {\small\begin{tabular}{ccccccc}
    \toprule
        Client num& Methods  & SST-2 & QNLI & MNLI & QQP & RTE \\
        \midrule 
         &LoRA&\textbf{95.64}&\textbf{92.04}&\textbf{85.85}&\textbf{86.16}&\textbf{82.19}\\      
        3 &FFA-LoRA&94.91&90.11&84.06&85.48&80.86\\
         &\cellcolor{ours}RoloRA&\cellcolor{ours}95.60&\cellcolor{ours}{91.62}&\cellcolor{ours}{85.66}&\cellcolor{ours}\textbf{86.16}&\cellcolor{ours}\textbf{82.19}\\
         \midrule 
                 &LoRA&94.27&86.91&81.22&82.07&46.21\\
       20 &FFA-LoRA&93.92&89.58&80.51&82.62&57.76\\
         &\cellcolor{ours}RoloRA&\cellcolor{ours}\textbf{94.84}&\cellcolor{ours}\textbf{90.77}&\cellcolor{ours}\textbf{85.13}&\cellcolor{ours}\textbf{85.10}&\cellcolor{ours}\textbf{81.23}\\
         \midrule 
                 &LoRA&93.23&82.57&58.96&76.96&49.10\\
        50 &FFA-LoRA&92.32&85.15&62.79&77.78&53.07\\
         &\cellcolor{ours}RoloRA&\cellcolor{ours}\textbf{94.61}&\cellcolor{ours}\textbf{89.83}&\cellcolor{ours}\textbf{85.15}&\cellcolor{ours}\textbf{85.55}&\cellcolor{ours}\textbf{72.92} \\
         \bottomrule
    \end{tabular}
    \vspace{4pt}
        \caption{Results with RoBERTa-Large models with varying client numbers (3, 20, 50) using rank-2 LoRA modules in federated setting, maintaining a constant sample count during fine-tuning.
        \label{tab:Clients-num-appendix-rank-2}}
}
\end{table}

\paragraph{Rank-32 Results}
In Table~\ref{tab:rank-32}, we provide additional experiments with rank-32 LoRA adapters in the 20-client and 50-client setting.

\begin{table}
    \centering
    {\footnotesize
    \begin{tabular}{cccc}
    \toprule
    Client num     & Methods & MNLI &QQP \\
    \midrule
         & LoRA & 79.72 \com 0.38 & 83.66 \com 0.02\\
      20   & FFA-LoRA & 80 \com 0.47  & 84.08 \com 0.31 \\
         & \cellcolor{ours}RoLoRA & \cellcolor{ours}85.91 \com 0.63 &\cellcolor{ours} 86.37 \com 0.09 \\
         \midrule
         & LoRA & 70.84 \com 4.63 & 79.75 \com 0.31 \\
      50   & FFA-LoRA & 74.47 \com 1.57 & 80.65 \com 0.31 \\
         & \cellcolor{ours}RoLoRA &\cellcolor{ours} 85.46 \com 0.08 &\cellcolor{ours} 86.15 \com 0.26\\
         \bottomrule
    \end{tabular}}
    \vspace{4pt}
    \caption{Results with RoBERTa-Large models in 20-client and 50-client setting using rank-32 LoRA adapters.}
    \label{tab:rank-32}
\end{table}
\paragraph{10-Client Setting}
In Table~\ref{tab:10-client-rank-4}, we provide results for 10-client setting with rank-4 adapter. The results show that RoLoRA still outperform other methods. In the 10-client setting, RoLoRA's performance gain over other methods falls between the gains observed in the 3-client and 20-client settings.

\begin{table}[h]
    \centering
    {\footnotesize
    \begin{tabular}{cccc}
    \toprule
         & MNLI & QQP & QNLI\\
         \midrule
        LoRA & 81.48 \com 2.19 & 84.1 \com 0.14 & 87.73 \com 0.67\\
        FFA-LoRA & 83.19 \com 0.64 & 84.35 \com 0.06 & 89.88 \com 0.13\\
        \rowcolor{ours} RoLoRA & 84.95 \com 0.8 & 95.25\com 0.39 & 90.3 \com 0.76\\
         \bottomrule
    \end{tabular}
    \vspace{4pt}}
    \caption{Results with RoBERTa-Large model with 10 clients using rank-4 LoRA adapters, running for 150 rounds in total.}
    \label{tab:10-client-rank-4}
\end{table}

\paragraph{FLoRA vs. RoLoRA} Table~\ref{tab:flora_vs_rolora} shows a comparison between FLoRA and RoLoRA. In the 3-client setting, we ran 500 rounds and scaled rounds down proportionally with more clients to keep the total sample budget fixed. RoLoRA consistently outperforms FLoRA across tasks and client counts. While FLoRA eventually converges (e.g., 83.3\% on MNLI after 4000 rounds), it does so much more slowly, highlighting RoLoRA's faster convergence and better scalability.

\begin{table}[h!]
{\footnotesize
\centering
\begin{tabular}{c c c c c}
\toprule
{Client num} & {Method} & {MNLI} & {QQP} & {QNLI} \\
\midrule
3  & FLoRA  & 39.29  & 51.05  & 59.88 \\
   & \cellcolor{ours} RoLoRA & \cellcolor{ours} \textbf{85.70} & \cellcolor{ours}  \textbf{86.14} & \cellcolor{ours} \textbf{91.64} \\
\midrule
20 & FLoRA  & 32.01  & 51.58  & 49.89 \\
   & \cellcolor{ours} RoLoRA &  \cellcolor{ours} \textbf{85.28} & \cellcolor{ours}  \textbf{85.83} & \cellcolor{ours} \textbf{90.35} \\
\midrule
50 & FLoRA  & 31.97  & 50.54  & 38.82 \\
   & \cellcolor{ours} RoLoRA &  \cellcolor{ours} \textbf{82.98} & \cellcolor{ours}  \textbf{85.71} & \cellcolor{ours} \textbf{90.00} \\
\bottomrule
\end{tabular}
\vspace{4pt}
\caption{Results with RoBERTa-Large models with rank-4 LoRA adapter for varying numbers of clients (3, 20, 50), comparing FLoRA with RoLoRA, maintaining a constant sample count during finetuning. In the 3-client setting, while FLoRA eventually converges (e.g., 83.3\% on MNLI after 4000 rounds), the figure shows results for only 500 rounds, within which FLoRA has not yet converged. This highlights RoLoRA's faster convergence and better scalability.}
\label{tab:flora_vs_rolora}}
\end{table}


\paragraph{Finetuning Dynamics within 100 Rounds}
Figure~\ref{fig:convergence-client50-100round} presents the 100-round extension of Figure 3, where RoLoRA consistently converges faster and achieves the highest accuracy.

We want to clarify that Figure~\ref{fig:convergence-speed} focuses on comparing convergence under a fixed sample budget rather than full convergence, and Table ~\ref{tab:Clients-num-flex-appendix} shows that this budget suffices for all methods when using 3 clients. However, as shown in Figure~\ref{fig:convergence-speed}, with 50 clients, only RoLoRA fully converges, underscoring its efficiency in low-resource settings.
\begin{figure}[h!]
\begin{center}
\begin{subfigure}
  \centering
  \includegraphics[width=0.28\linewidth]{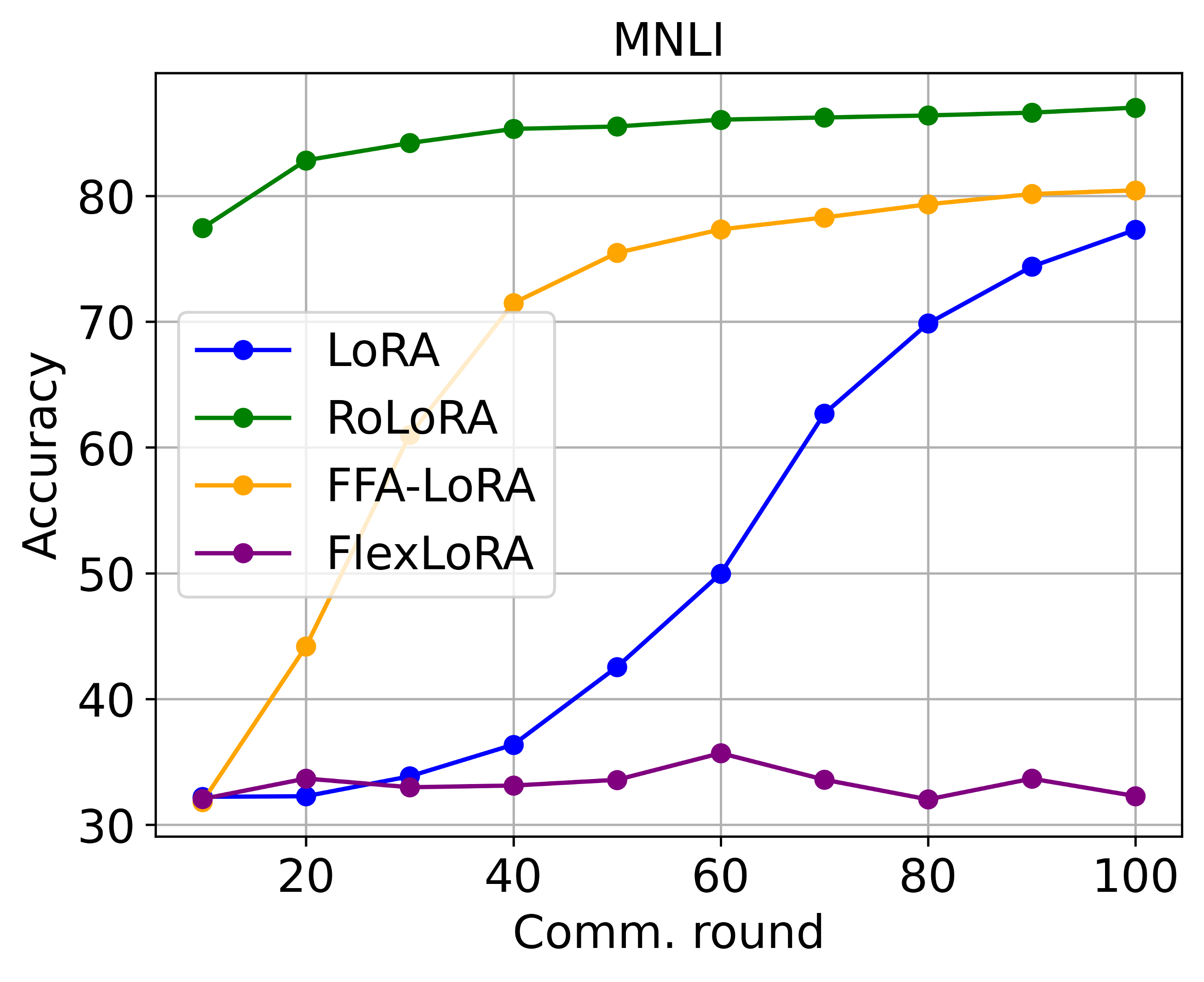}
\end{subfigure}
\begin{subfigure}
  \centering
  \includegraphics[width=0.28\linewidth]{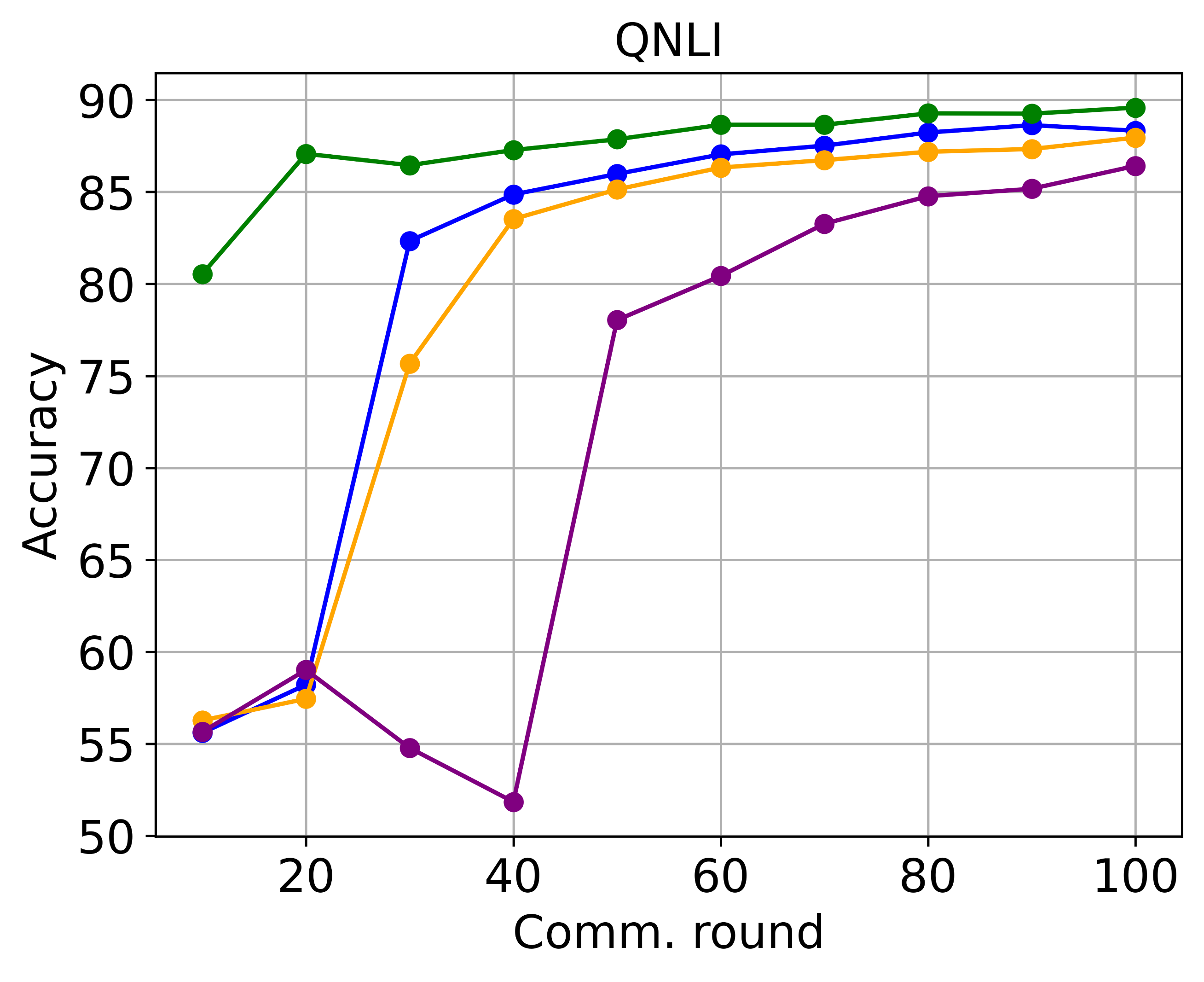}
\end{subfigure}
\begin{subfigure}
  \centering
  \includegraphics[width=0.28\linewidth]{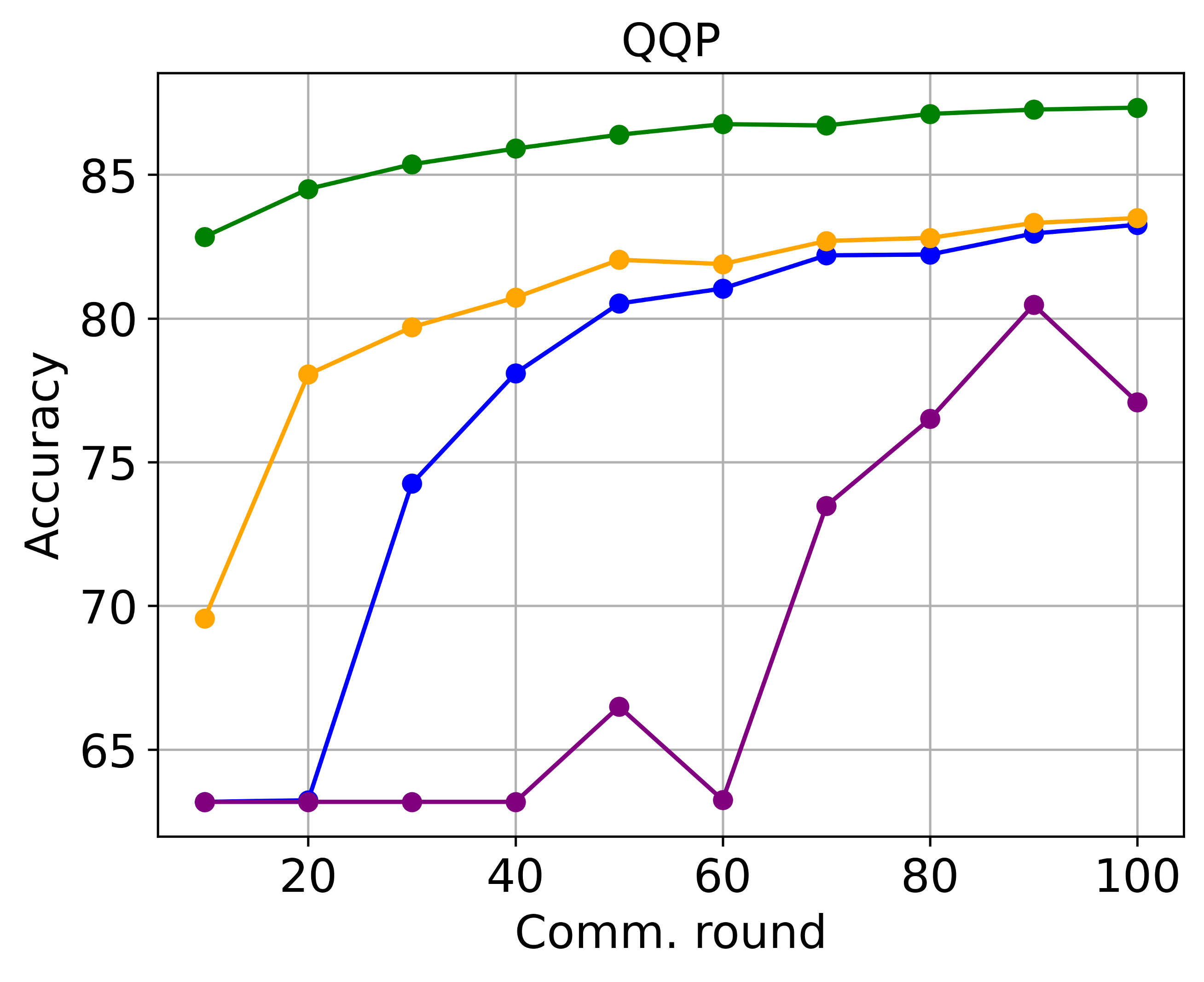}
\end{subfigure}
 \caption{Accuracies over 100 rounds. It involves 50 clients using rank 4.}
    \label{fig:convergence-client50-100round}
\hfill
\end{center}
\end{figure}

 \subsection{Effect of Number of LoRA Parameters} \label{app:num-lora-appendix}
 \vspace{-0.3cm}
 In Table~\ref{tab:layer_type_index}, we include the details about layers attached with LoRA adapters for different budget of finetuning parameters, for each dataset.

 \begin{table}[h!]
\centering
{\footnotesize
    \begin{tabular}{ccccccc}
    \toprule
                            & Layer Attributes & SST-2 & QNLI & MNLI & QQP & RTE \\
                            \midrule
\multirow{2}{*}{$\mathcal{P}_1$} & Type             &   $W_v$    &    $W_v,W_q$    &   $W_v,W_q$     &    $W_v,W_q$   &    $W_v,W_q$   \\
                    & Index            &  $\{21,\ldots,23\}$  &   $\{21,\ldots,23\}$    &   $\{21,\ldots,23\}$    &   $\{21,\ldots,23\}$   &  $\{21,\ldots,23\}$   \\
\multirow{2}{*}{$\mathcal{P}_2$} & Type             &     $W_v,W_q$   &  $W_v,W_q$     &   $W_v,W_q$    &   $W_v,W_q$   &   $W_v,W_q$   \\
                    & Index            &   $\{18,\ldots,23\}$    &   $\{15,\ldots,23\}$   &  $\{15,\ldots,23\}$    &   $\{15,\ldots,23\}$  &  $\{16,\ldots,23\}$   \\
\multirow{2}{*}{$\mathcal{P}_3$} & Type             &  $W_v,W_q$     &   $W_v,W_q$    &   $W_v,W_q$    &  $W_v,W_q$    &  $W_v,W_q$    \\
                    & Index            &    $\{0,\ldots,23\}$  &  $\{12,\ldots,23\}$    &   $\{12,\ldots,23\}$   &   $\{12,\ldots,23\}$  &   $\{12,\ldots,23\}$  \\ \bottomrule
    \end{tabular}}
    \vspace{4pt}
    \caption{The selected layer set attached with LoRA for the setup of Figure~\ref{fig:five_subfigures}}
    \label{tab:layer_type_index}
\end{table}

\begin{figure}[h!]
\begin{center}
\begin{subfigure}
  \centering
  \includegraphics[width=0.161\linewidth]{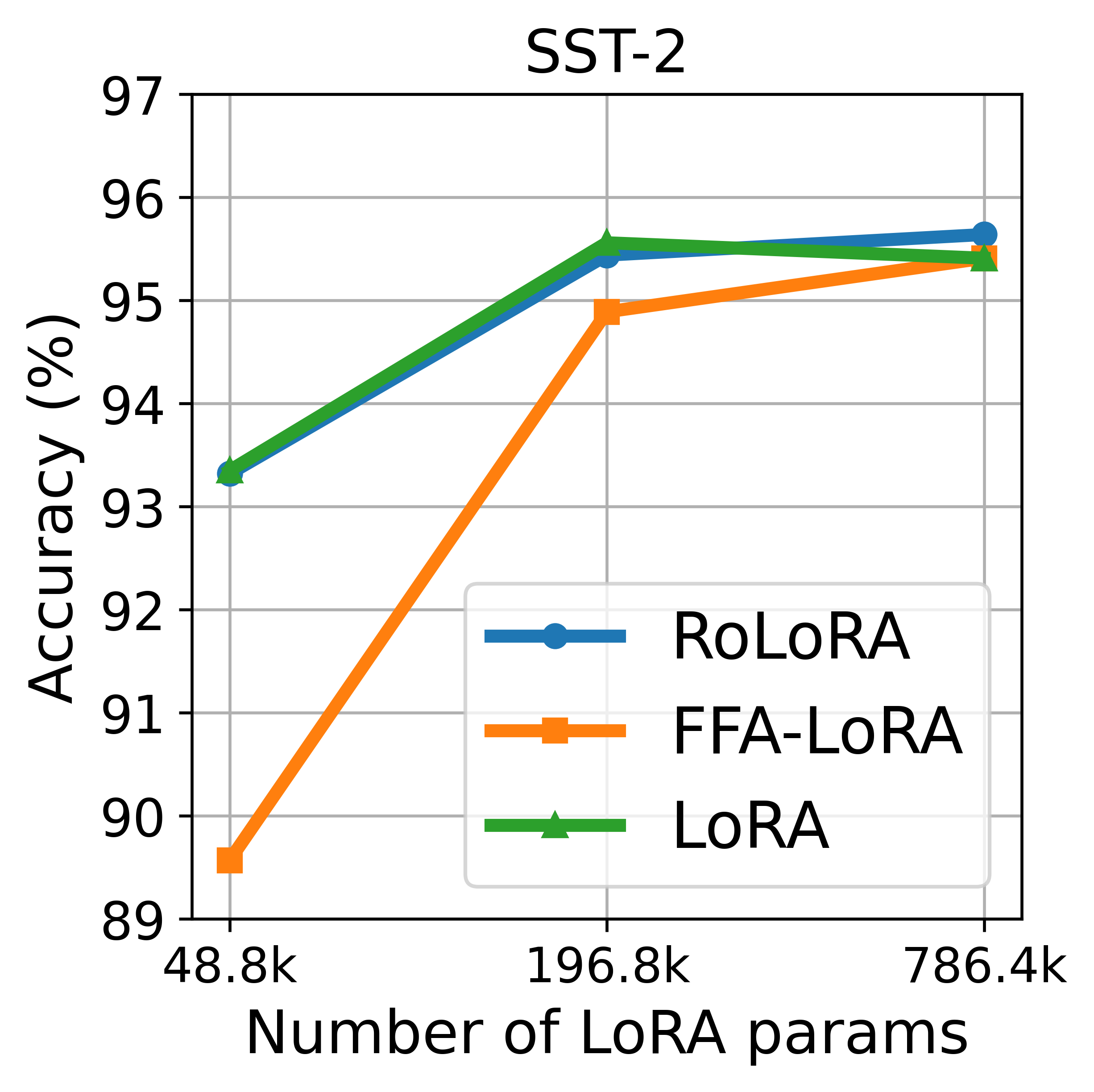}
  \label{fig:sub1-8}
\end{subfigure}
\hfill
\begin{subfigure}
  \centering
  \includegraphics[width=0.154\linewidth]{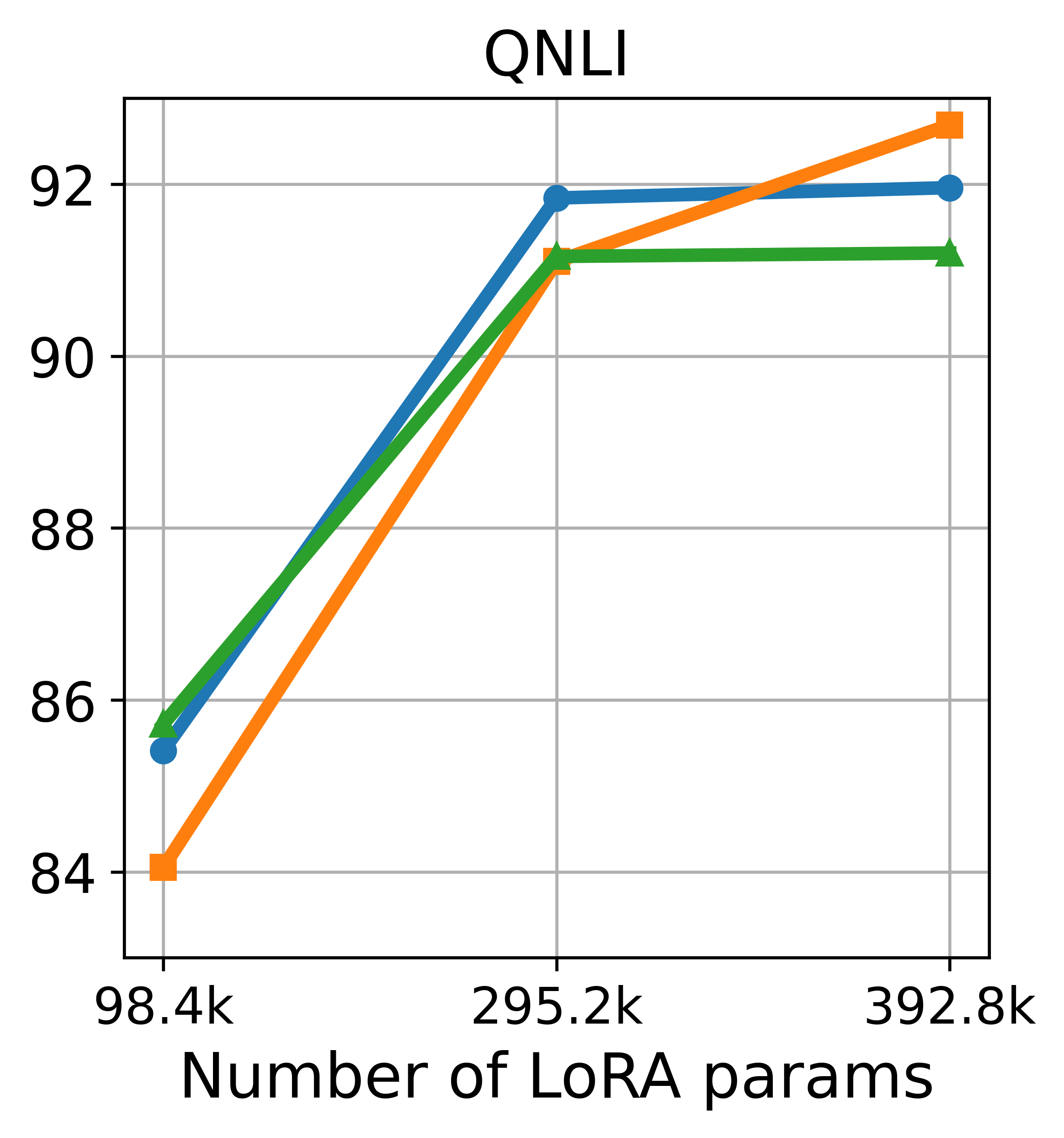}
  \label{fig:sub2-8}
\end{subfigure}
\hfill
\begin{subfigure}
  \centering
  \includegraphics[width=0.154\linewidth]{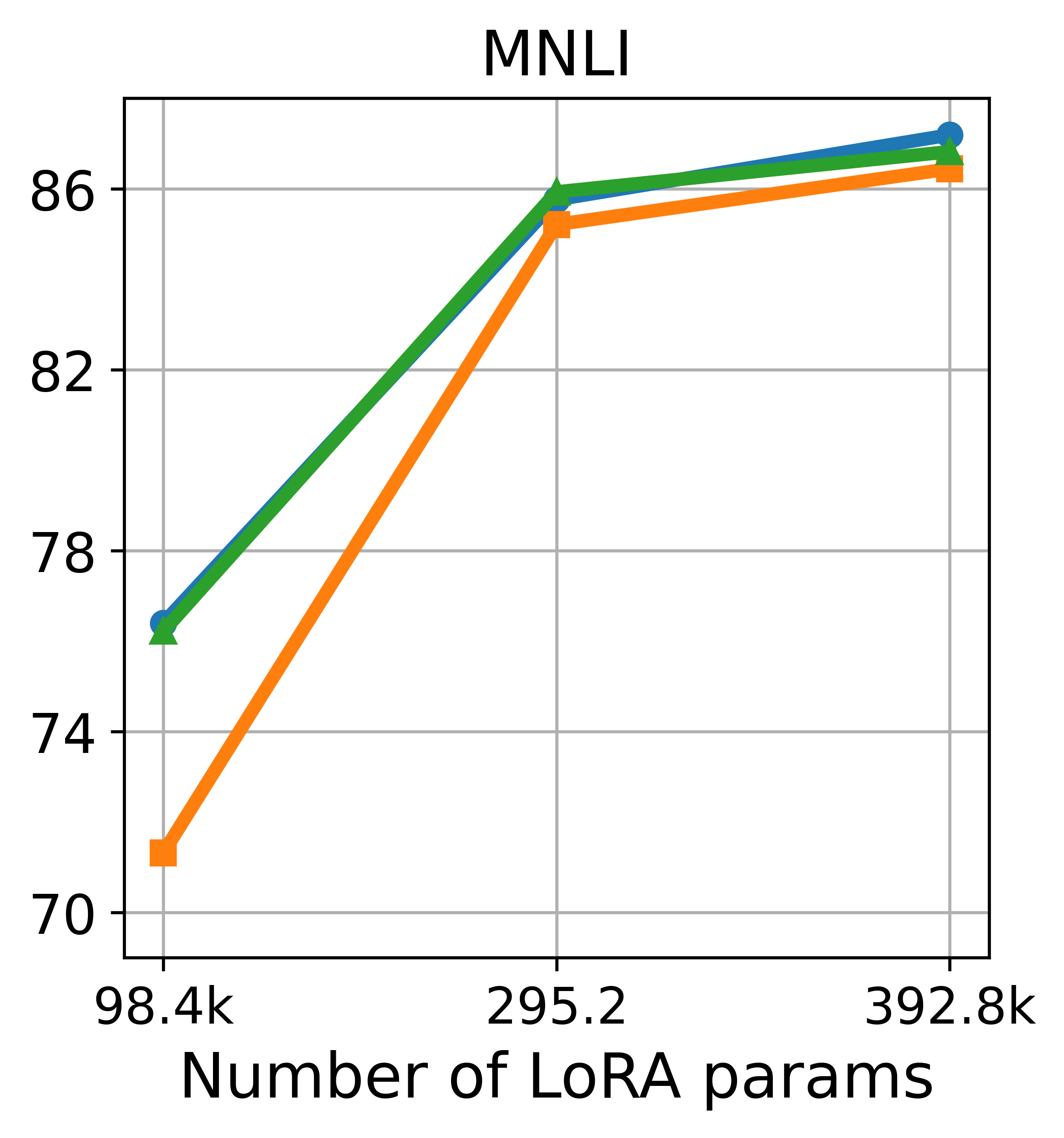}
  \label{fig:sub3-8}
\end{subfigure}
\hfill
\begin{subfigure}
  \centering
  \includegraphics[width=0.154\linewidth]{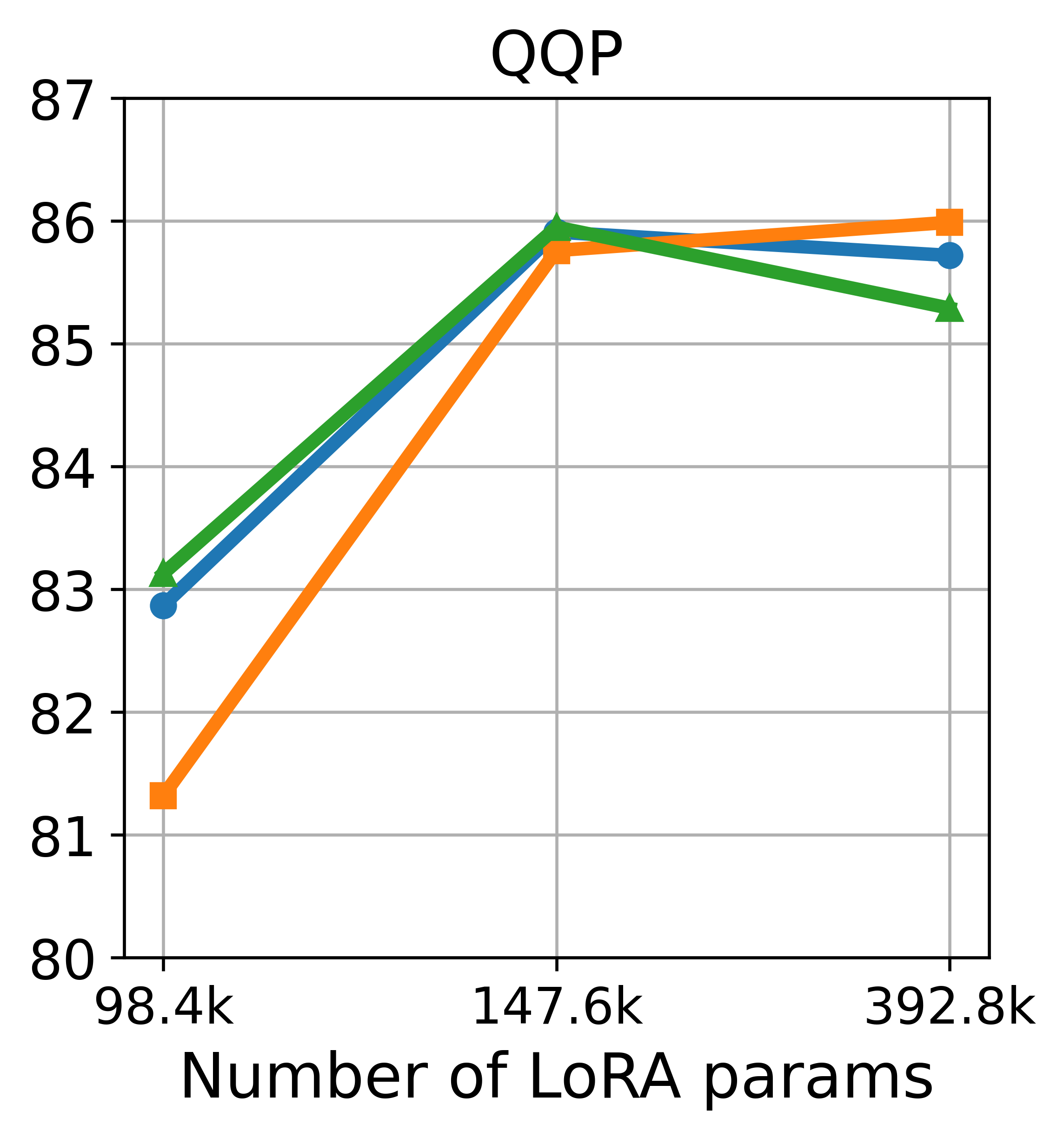}
  \label{fig:sub4-8}
\end{subfigure}
\hfill
\begin{subfigure}
  \centering
  \includegraphics[width=0.154\linewidth]{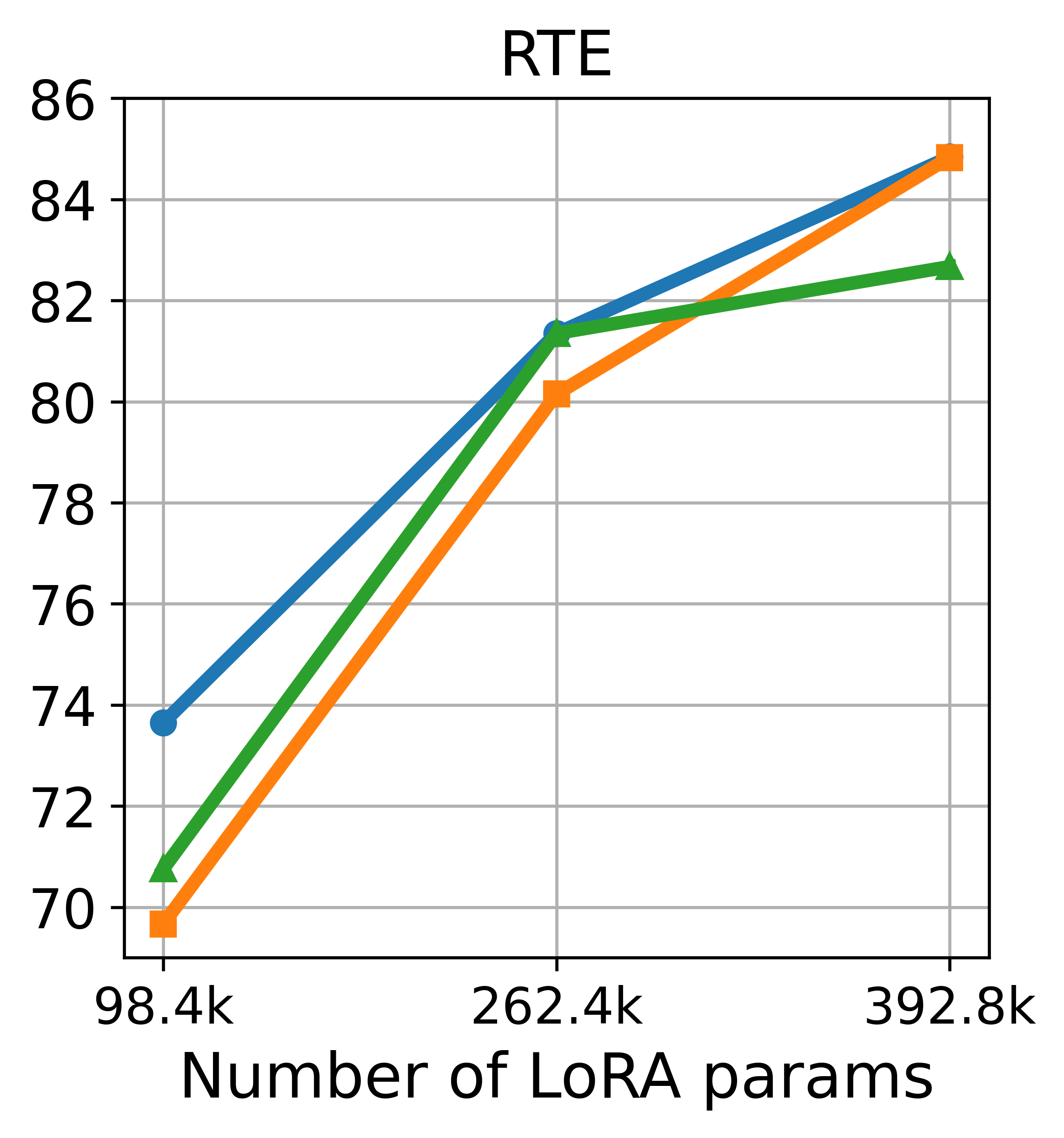}
  \label{fig:sub5-8}
\end{subfigure}
 \caption{Results with RoBERTa-Large models on GLUE of different budget of finetuning parameters. It involves 3 clients using rank 8.}
    \label{fig:five_subfigures-rank-8}
\hfill
\end{center}
\end{figure}
\begin{figure}[h!]
\begin{center}
\begin{subfigure}
  \centering
  \includegraphics[width=0.164\linewidth]{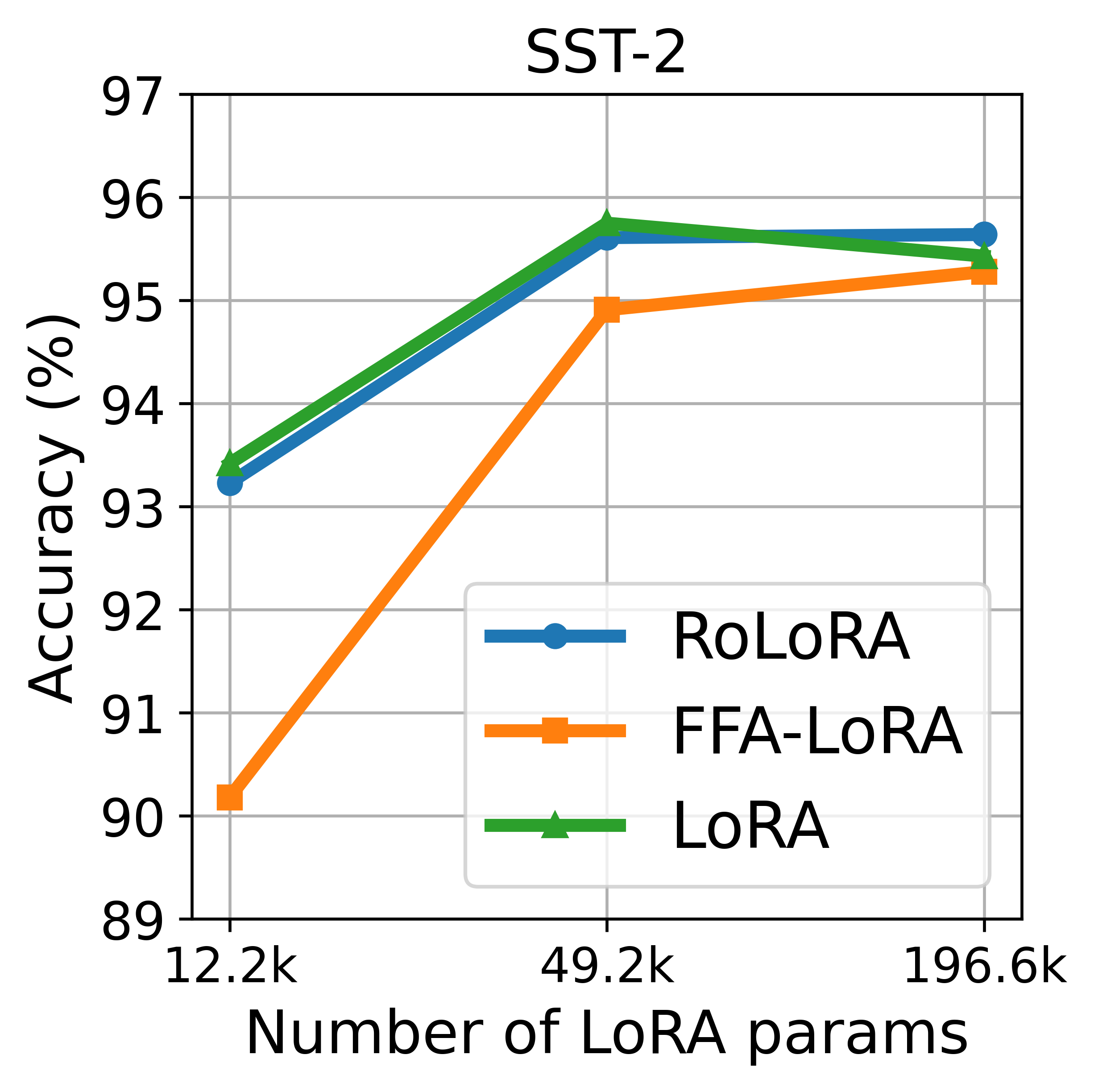}
  \label{fig:sub1-2}
\end{subfigure}
\hfill
\begin{subfigure}
  \centering
  \includegraphics[width=0.154\linewidth]{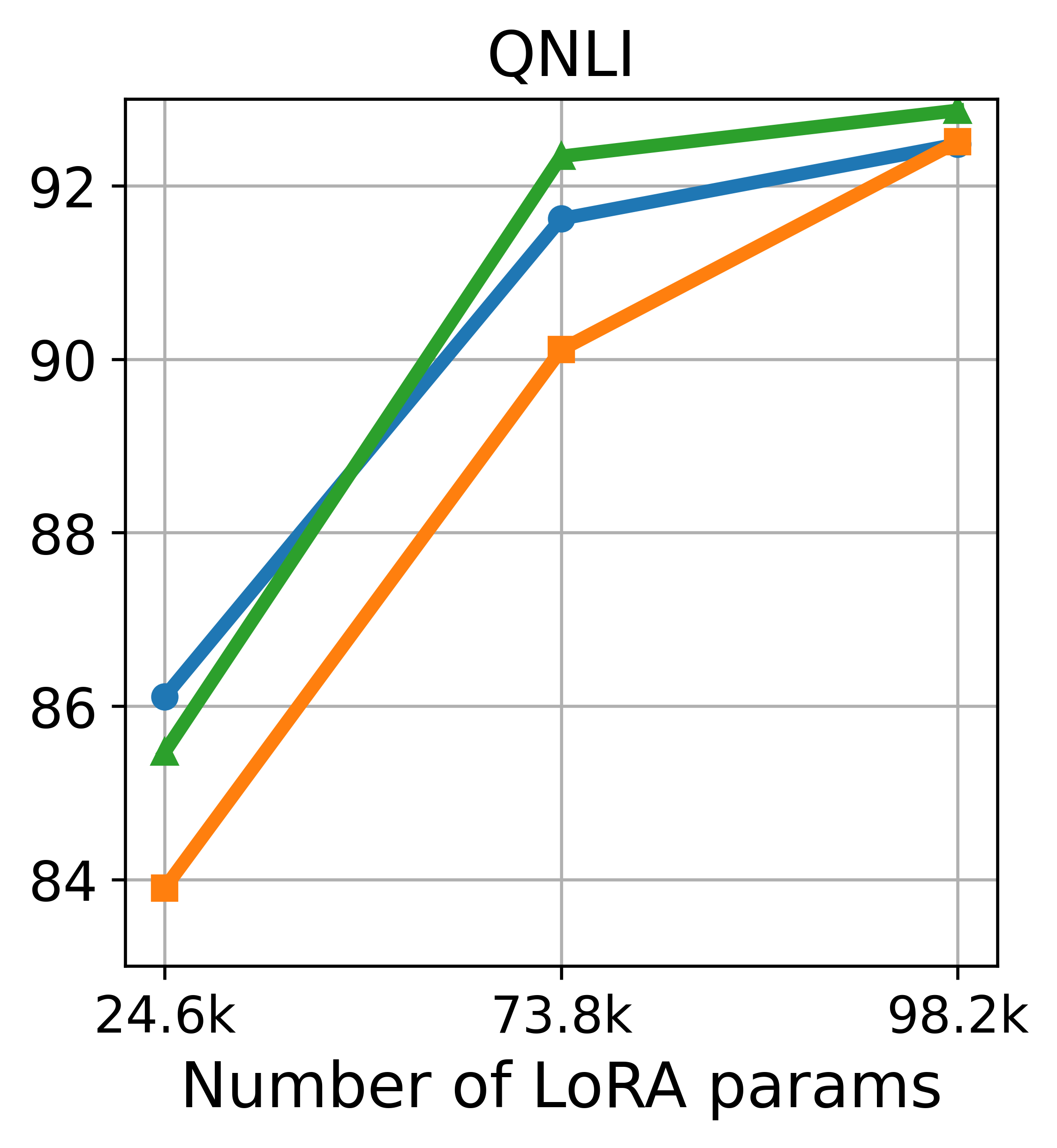}
  \label{fig:sub2-2}
\end{subfigure}
\hfill
\begin{subfigure}
  \centering
  \includegraphics[width=0.154\linewidth]{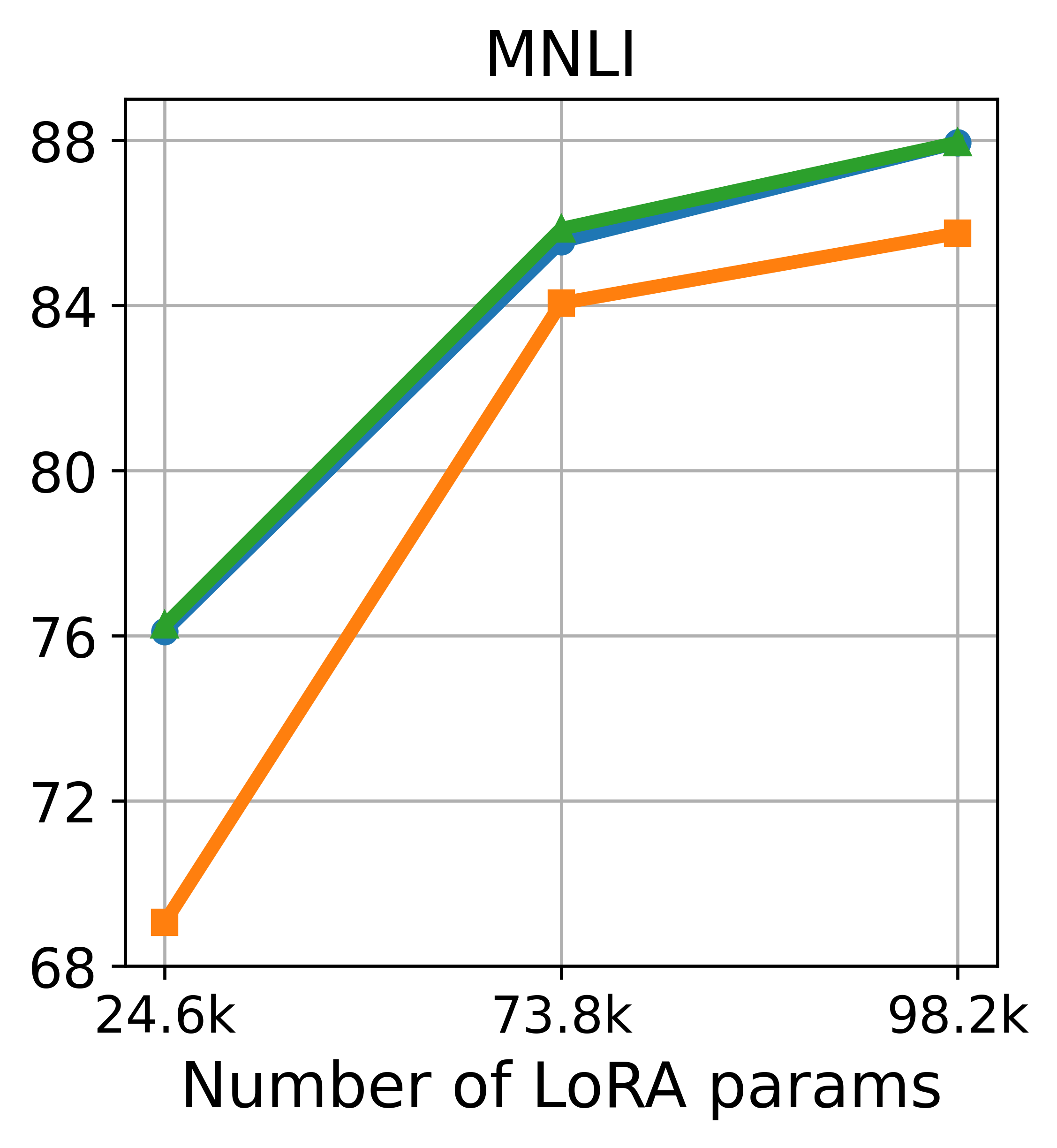}
  \label{fig:sub3-2}
\end{subfigure}
\hfill
\begin{subfigure}
  \centering
  \includegraphics[width=0.154\linewidth]{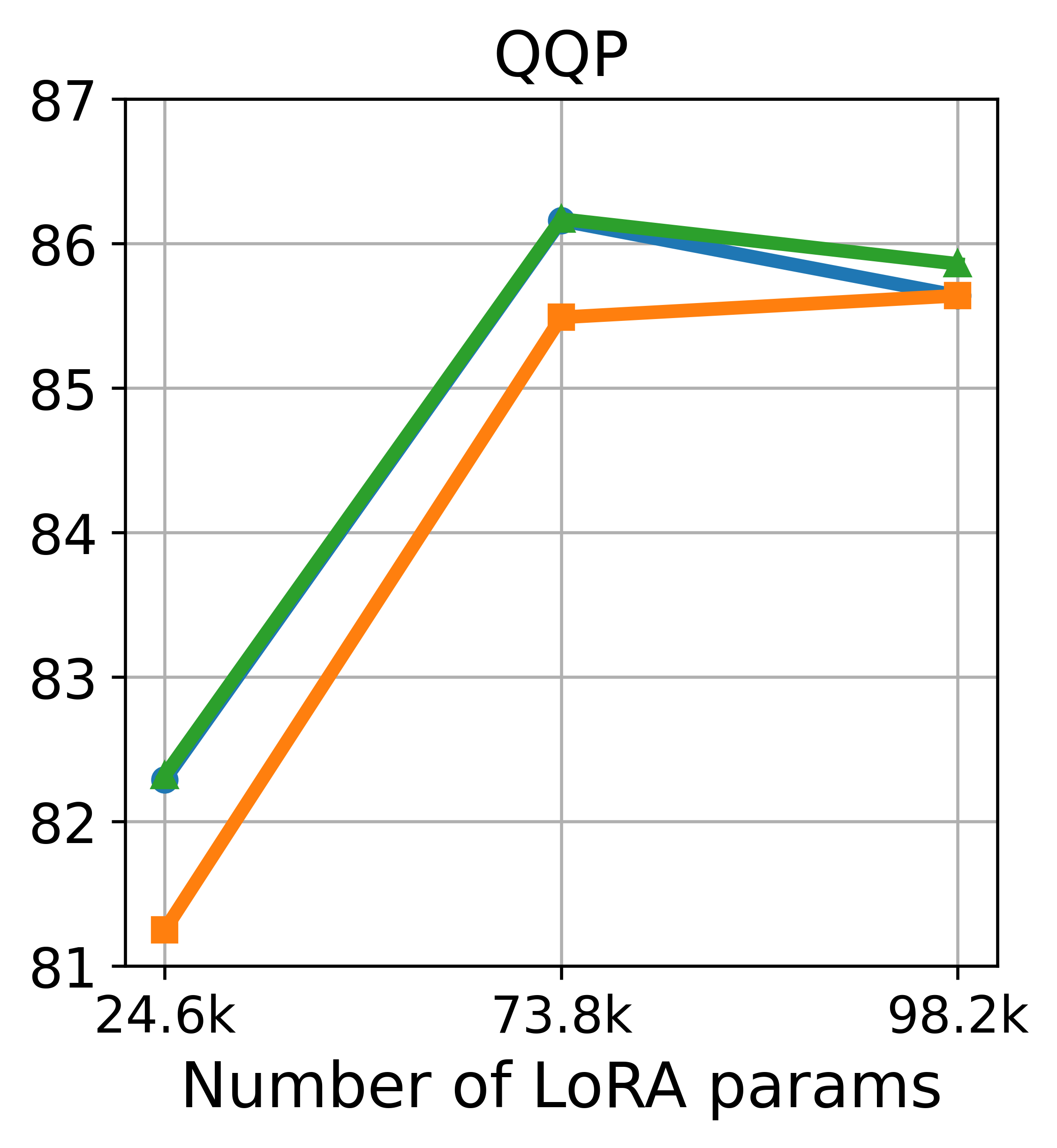}
  \label{fig:sub4-2}
\end{subfigure}
\hfill
\begin{subfigure}
  \centering
  \includegraphics[width=0.154\linewidth]{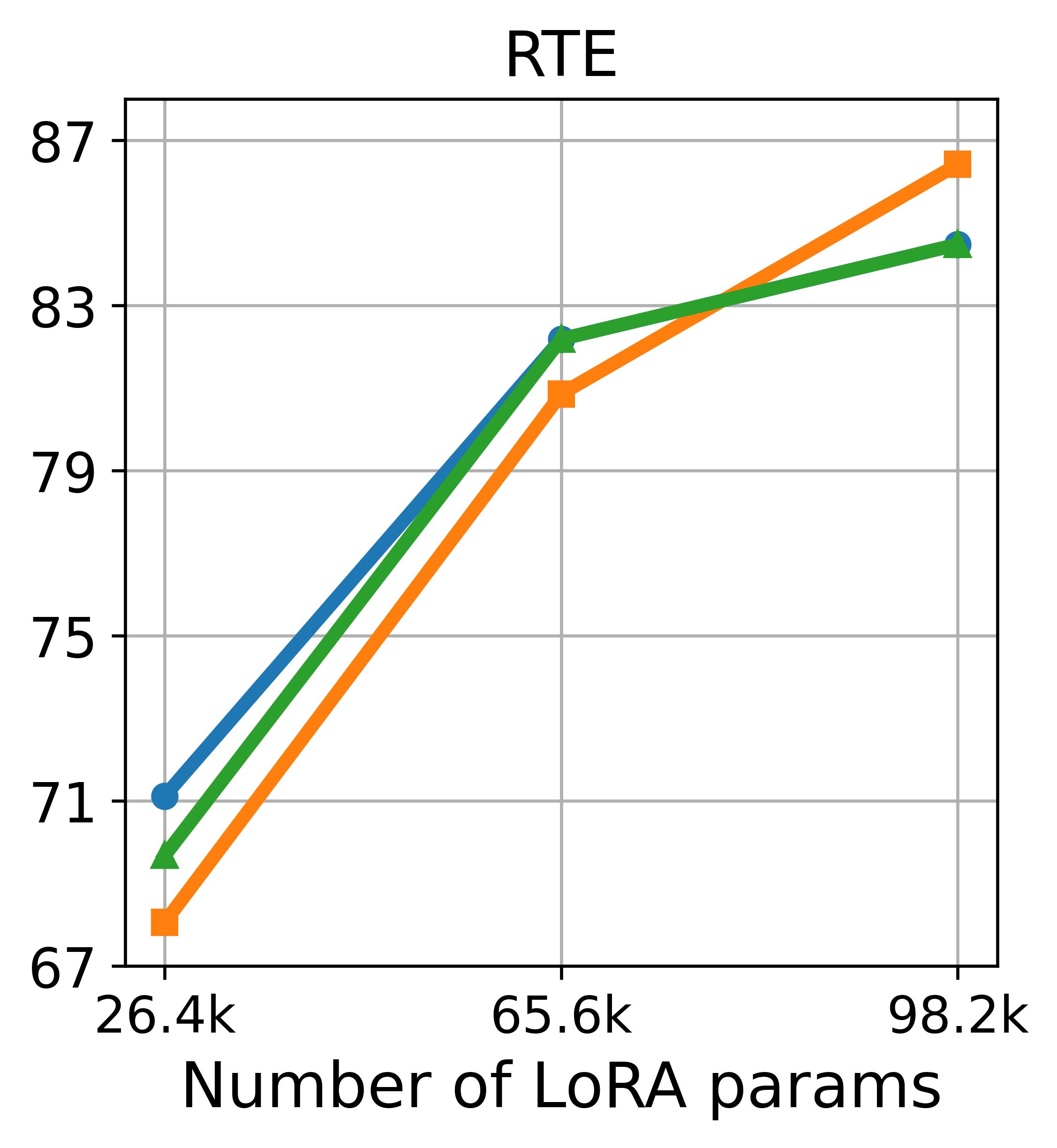}
  \label{fig:sub5-2}
\end{subfigure}
 \caption{Results with RoBERTa-Large models on GLUE of different budget of finetuning parameters. It involves 3 clients using rank 2.}
    \label{fig:five_subfigures-rank-2}
\hfill
\end{center}
\end{figure}
\begin{figure}[h!]
\begin{center}
\begin{subfigure}
  \centering
  \includegraphics[width=0.164\linewidth]{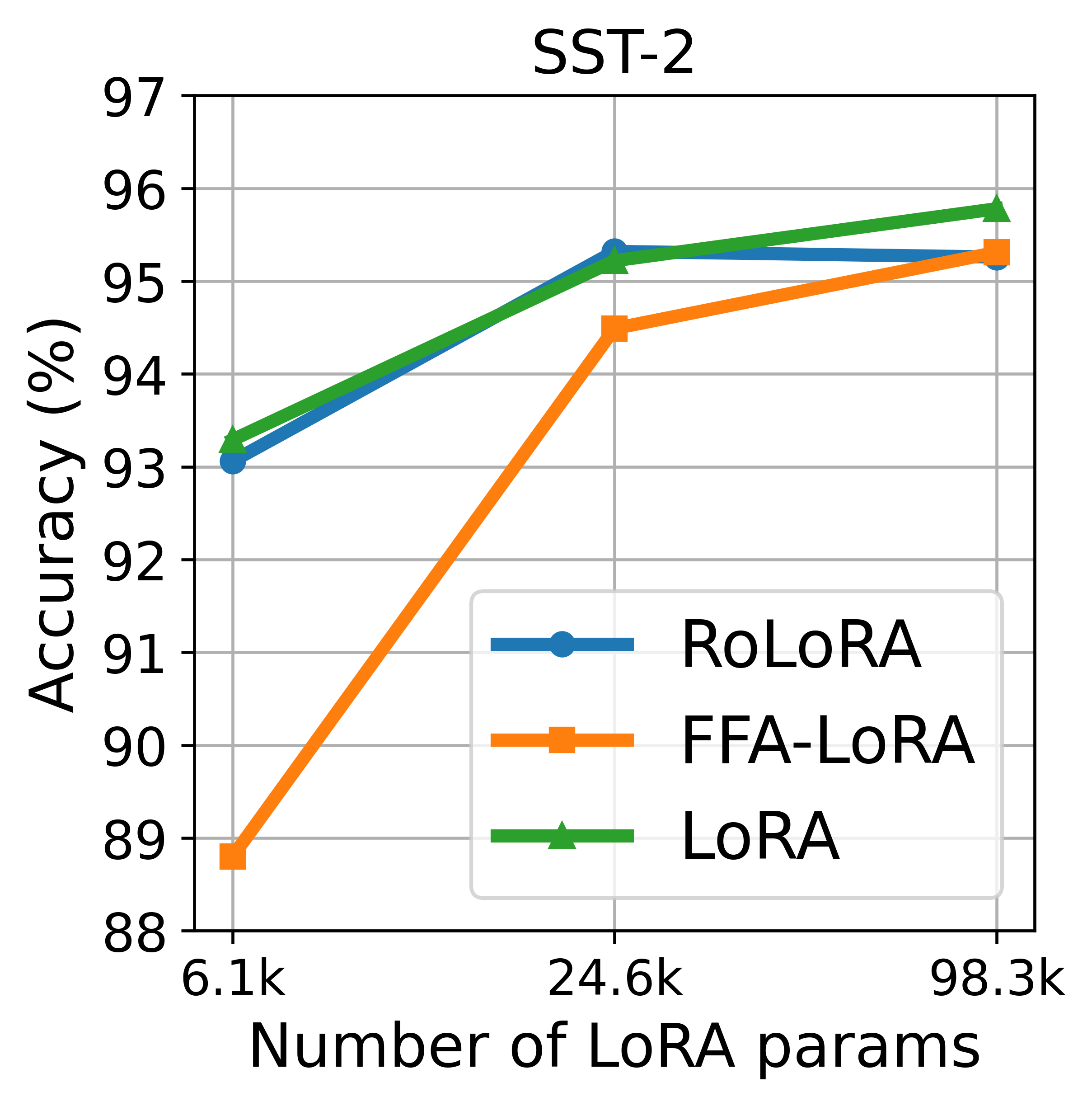}
  \label{fig:sub1-3}
\end{subfigure}
\hfill
\begin{subfigure}
  \centering
  \includegraphics[width=0.154\linewidth]{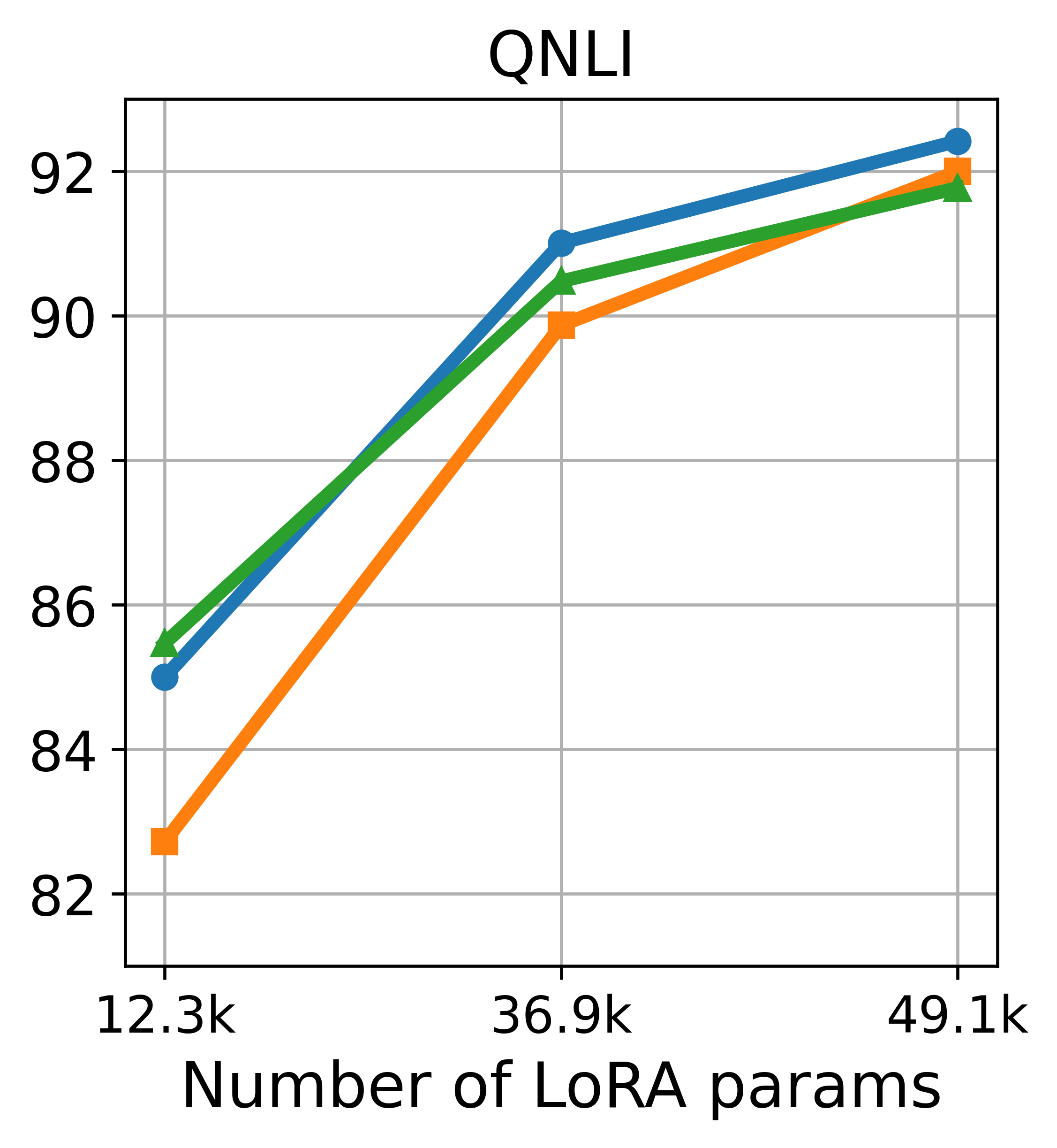}
  \label{fig:sub2-3}
\end{subfigure}
\hfill
\begin{subfigure}
  \centering
  \includegraphics[width=0.154\linewidth]{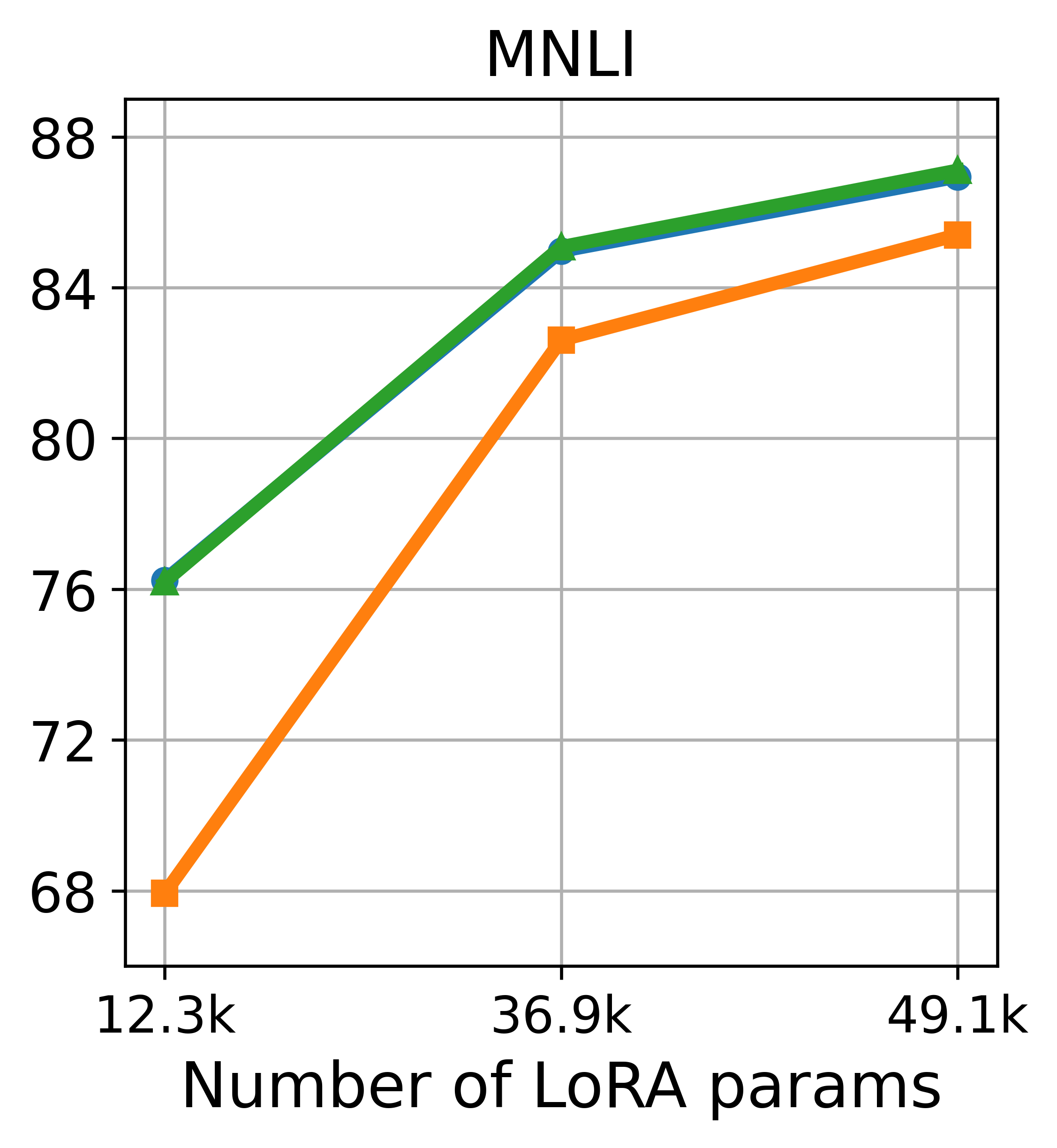}
  \label{fig:sub3-3}
\end{subfigure}
\hfill
\begin{subfigure}
  \centering
  \includegraphics[width=0.154\linewidth]{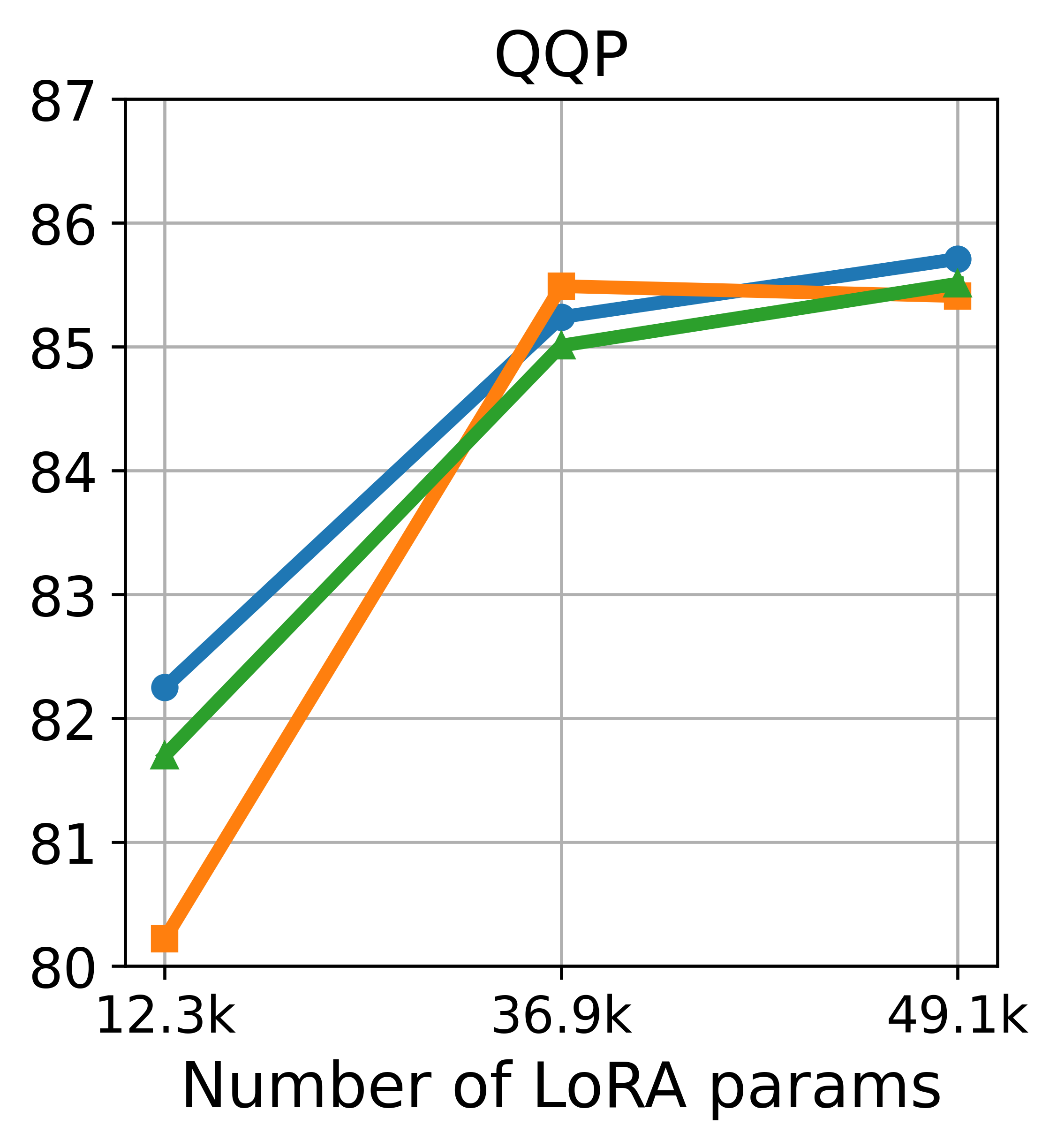}
  \label{fig:sub4-3}
\end{subfigure}
\hfill
\begin{subfigure}
  \centering
  \includegraphics[width=0.154\linewidth]{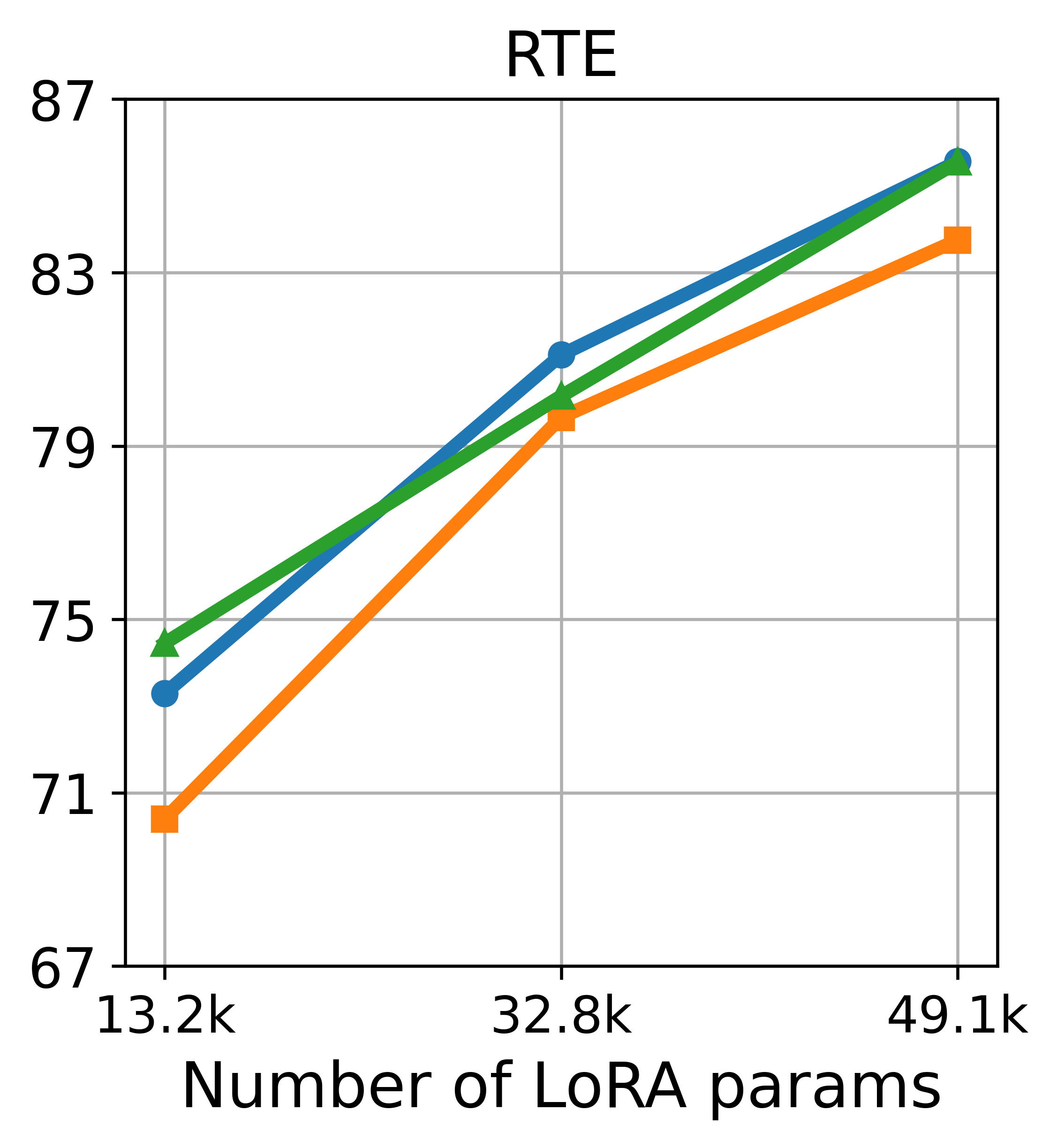}
  \label{fig:sub5-3}
\end{subfigure}
 \caption{Results with RoBERTa-Large models on GLUE of different budget of finetuning parameters. It involves 3 clients using rank 1.}
    \label{fig:five_subfigures-rank-1}
\hfill
\end{center}
\end{figure}
\vspace{-0.3cm}
 \subsection{Effect of Rank on RoLoRA and FFA-LoRA in the Centralized Setting} 
We evaluated FFA-LoRA and RoLoRA on MNLI and QQP using 8 LoRA adapters attached to query and value projection of last 4 layers, and trained for 50000 iterations to ensure full convergence. Increasing ranks can narrow the performance gap between the two schemes. Another related technique to narrow the performance gap between the two schemes is by increasing the number of adapters, as discussed in Section~\ref{exp-LU} (“Effect of Number of Fine-Tuning Parameters”). With sufficient adapters, FFA-LoRA can achieve comparable peak accuracy to RoLoRA. However, in federated settings where resources are constraiend, RoLoRA is more advantageous.
\begin{table}[h]
    \centering
    {\footnotesize
    \begin{tabular}{cccccc}
    \toprule
         &  & rank-1 & rank-2  & rank-32 & rank-64 \\
         \midrule
         & FFA-LoRA & 80.66 & 81.51 & 83.3 & 83.32 \\
        MNLI & RoLoRA & 83.93 & 84.59 & 85.78 & 85.79\\
         & Diff & 3.27 & +3.08 & +2.51 & +2.47\\
         \midrule
         & FFA-LoRA  & 69.61 & 74.01 & 75.53 & 75.51 \\
        QQP & RoLoRA & 77.26 & 77.41 & 78.03 & 78.05 \\
         & Diff & +7.65  &+3.40  & +2.5 & +2.54 \\
         \bottomrule
    \end{tabular}}
    \vspace{4pt}
    \caption{Evaluation on FFA-LoRA and RoLoRA in the centralized setting using only 8 LoRA adapters. Increasing ranks can narrow the performance gap between the two schemes. }
    \label{tab:centralized-rank}
\end{table}

 \subsection{Align the Communication Cost}\label{sec:align-appendices}
In Figure~\ref{fig:align-comm-cost}, we conduct a comparison of three methods under the constraint of identical communication costs under the assumption that the number of clients is small. To align the communication costs across these methods, two approaches are considered. The first approach involves doubling the rank of FFA-LoRA and RoLoRA, with results presented in Appendix~\ref{app:num-lora-appendix}. The second approach requires doubling the number of layers equipped with LoRA modules. In Figure~\ref{fig:align-comm-cost}, the latter strategy is employed. Specifically, for both FFA-LoRA and RoLoRA, we adjust the communication costs by doubling the number of layers equipped with LoRA modules, thereby standardizing the size of the transmitted messages. The configurations are presented in Table~\ref{tab:layer_type_index-align-comm-cost}. Figure~\ref{fig:align-comm-cost} demonstrates that when operating within a constrained communication cost budget, the performance of RoLoRA surpasses that of the other two methods for most of the tasks. 

\begin{figure}[h!]
\begin{center}
\begin{subfigure}
  \centering
  \includegraphics[width=0.2\linewidth]{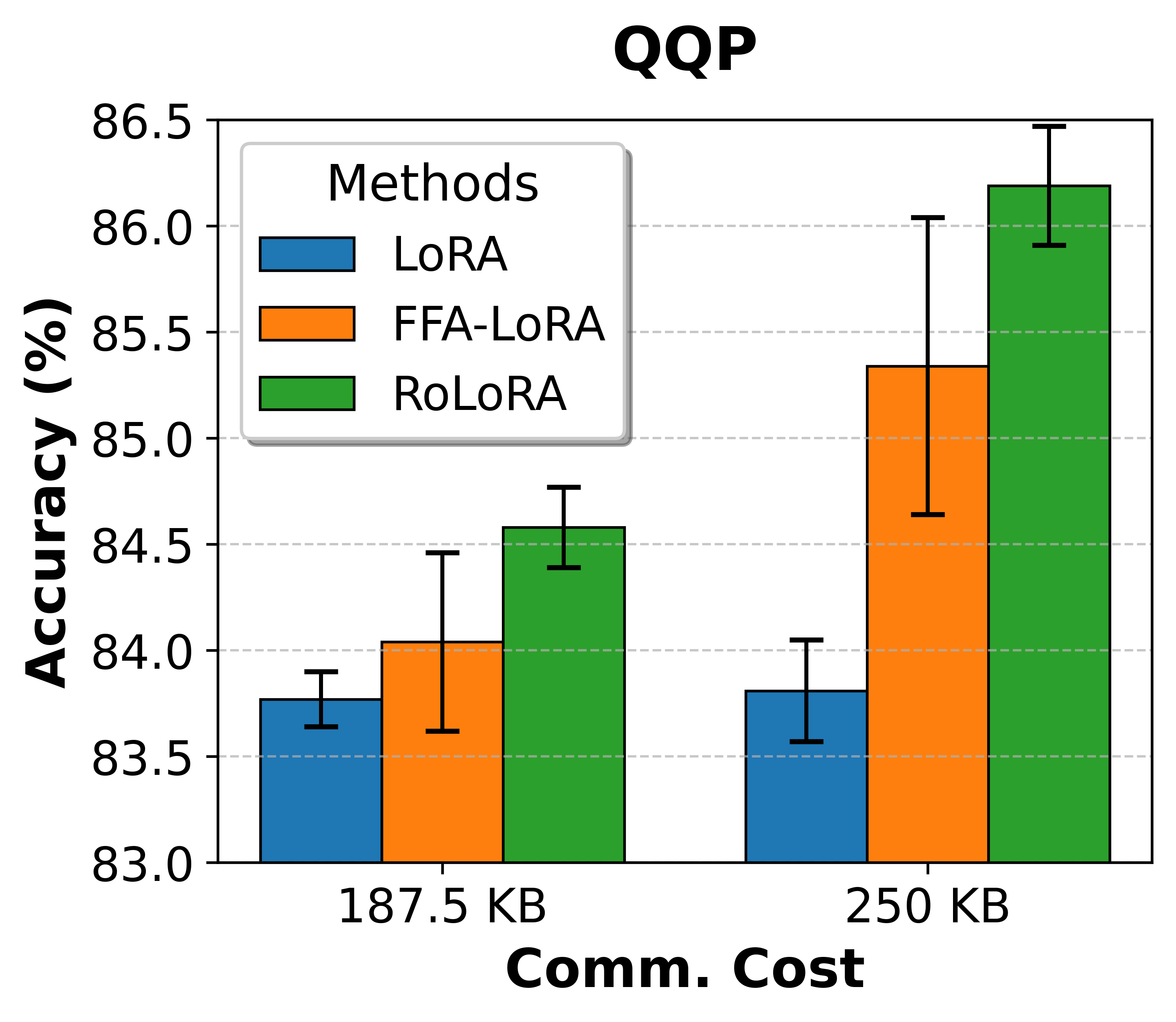}
  \label{fig:sub1-1}
\end{subfigure}
\begin{subfigure}
  \centering
  \includegraphics[width=0.2\linewidth]{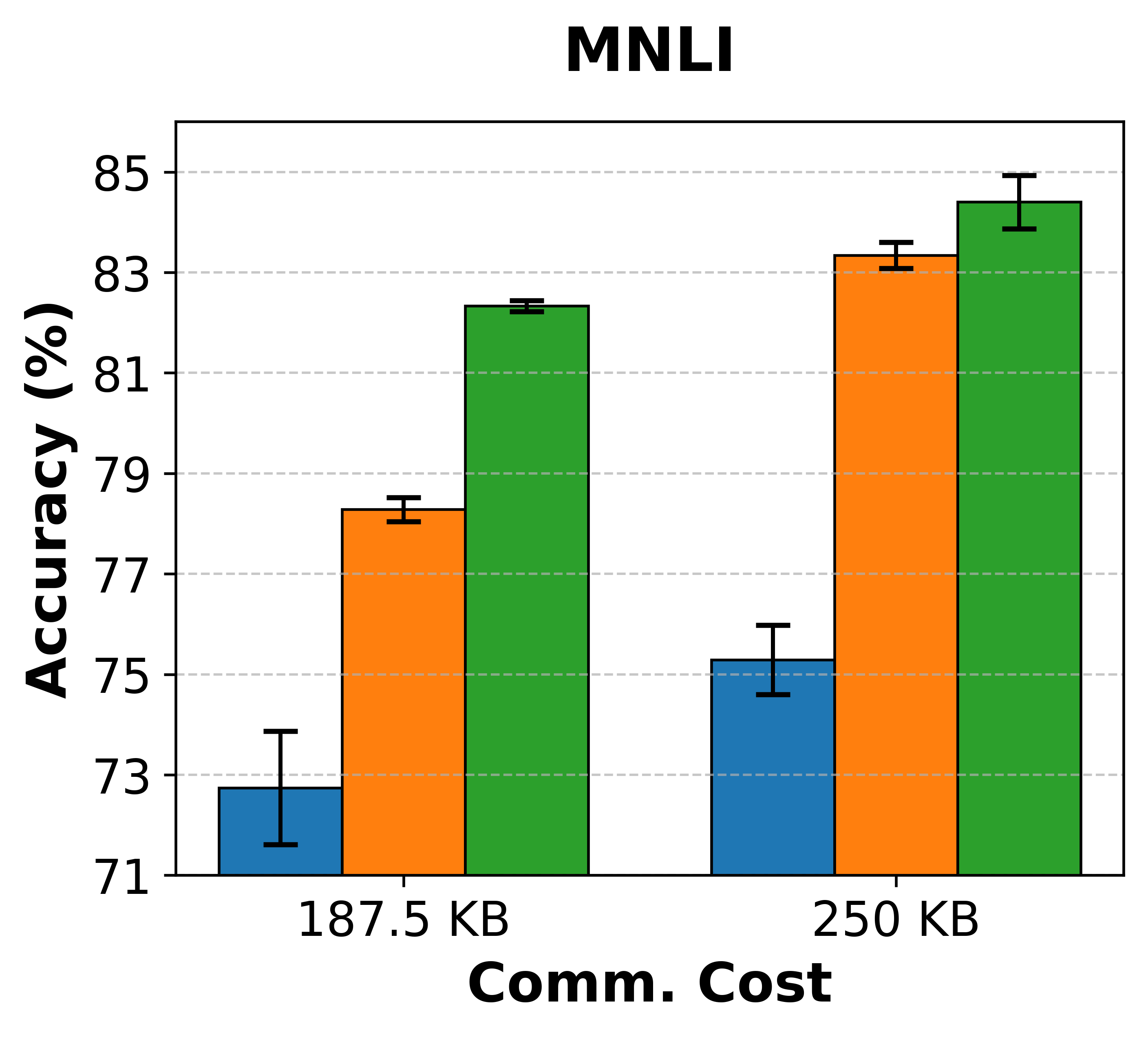}
  \label{fig:sub2-2}
\end{subfigure} 
\begin{subfigure}
  \centering
  \includegraphics[width=0.2\linewidth]{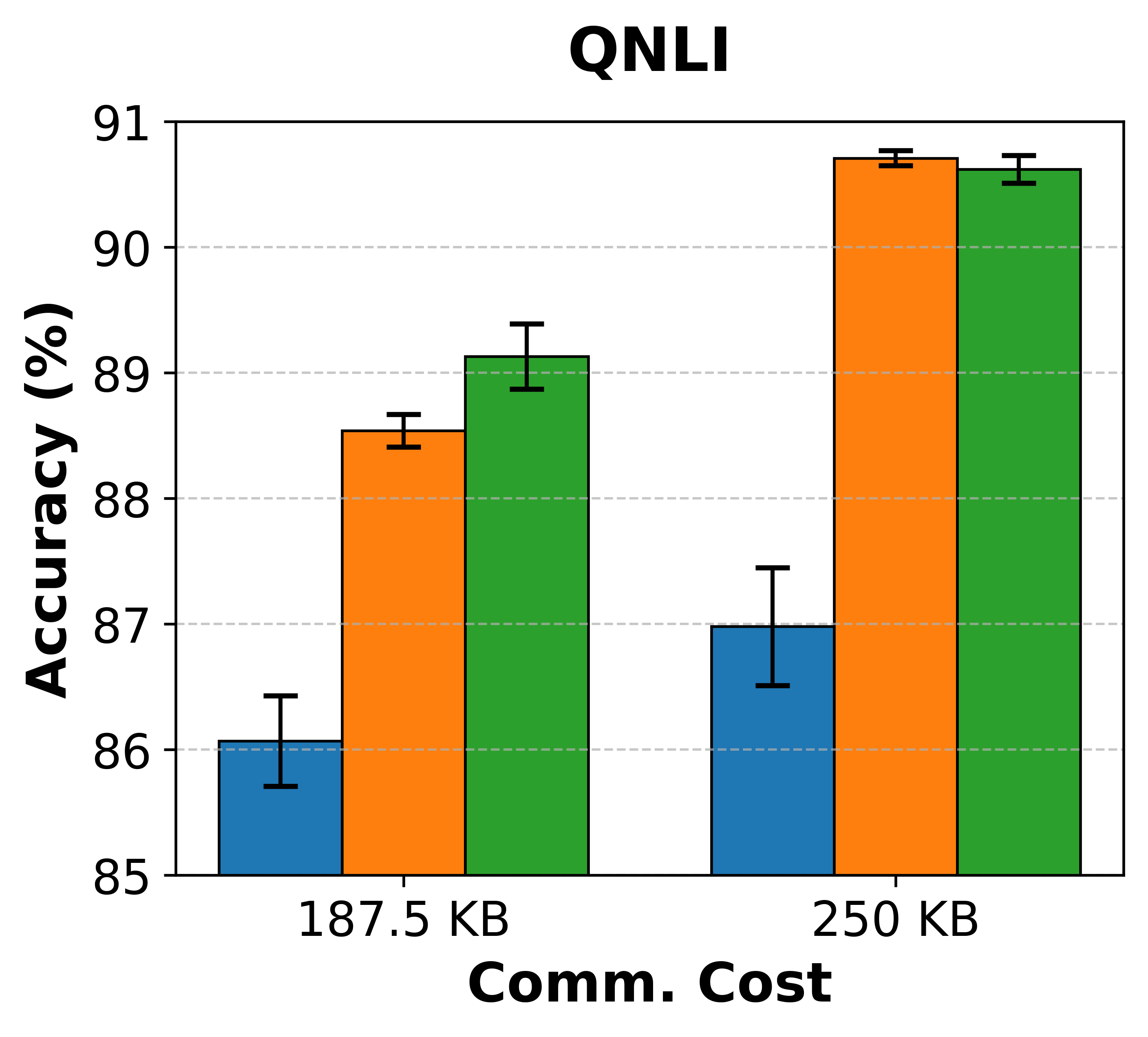}
  \label{fig:sub2-3}
\end{subfigure}
\begin{subfigure}
  \centering
  \includegraphics[width=0.2\linewidth]{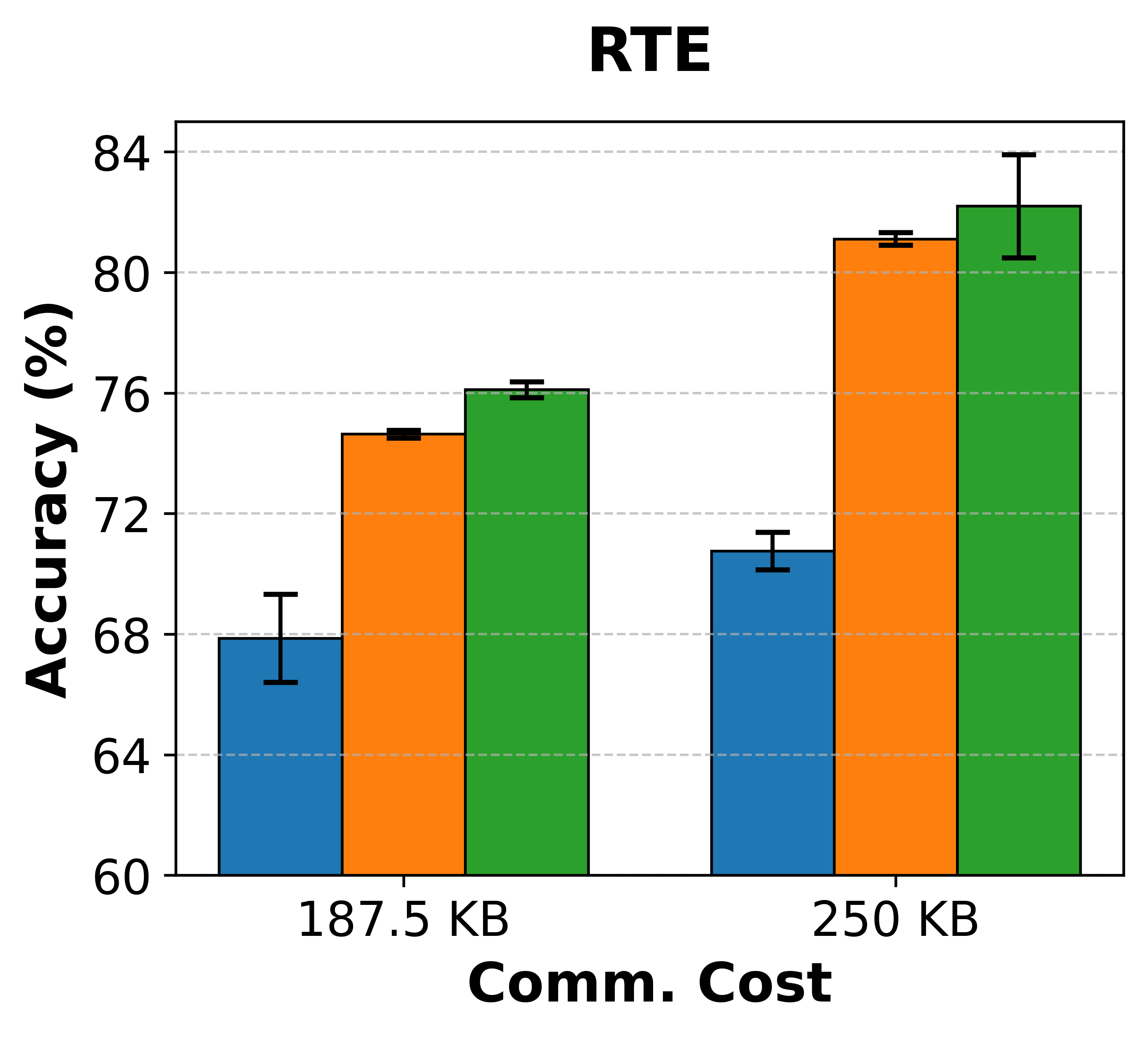}
  \label{fig:sub2-4}
\end{subfigure}
 \caption{RoBERTa-Large accuracies on QQP, MNLI, QNLI, and RTE with specific uplink communication budget. It involves 3 clients using rank 4. Error bars reflect standard deviations.}
 \label{fig:align-comm-cost}
 \end{center}
 \end{figure}
 
 In Table~\ref{tab:layer_type_index-align-comm-cost}, we include the details about layers attached with LoRA adapters.

\begin{table}[h!]
\centering
{\footnotesize
    \begin{tabular}{ccccc}
    \toprule
                            & Communication Cost & LoRA & FFA-LoRA & RoLoRA  \\
                            \midrule
\multirow{2}{*}{187.5 KB} & Type             &     $W_v,W_q$   &  $W_v,W_q$     &   $W_v,W_q$     \\
                    & Index            &   $\{21,\ldots,23\}$    &   $\{18,\ldots,23\}$   &  $\{18,\ldots,23\}$     \\
                    \midrule
                    \multirow{2}{*}{250 KB} & Type             &     $W_v,W_q$   &  $W_v,W_q$     &   $W_v,W_q$     \\
                    & Index            &   $\{20,\ldots,23\}$    &   $\{16,\ldots,23\}$   &  $\{16,\ldots,23\}$       \\ \bottomrule
    \end{tabular}}
    \vspace{4pt}
    \caption{The selected layer set attached with LoRA modules for the setup of Figure~\ref{fig:align-comm-cost}}
    \label{tab:layer_type_index-align-comm-cost}
\end{table}

 \subsection{Commonsense Reasoning Tasks} \label{sec:setup-CRT}
 \vspace{-0.3cm}
We present the configurations for Table~\ref{tab:reasoning} in Table~\ref{tab:commonsense-task}. We show the results under the same setup but using rank-2 LoRA modules in Table~\ref{tab:reasoning-rank2-appendix}.  

\begin{table}[h]
{\scriptsize
    \centering
    \begin{tabular}{ccccc}
    \toprule
        Total comm. rounds & Batch size &  Local Epochs & Layer type attached with LoRA &Layer index attached with LoRA \\
        \midrule
        10 & 1 & 30 & $W_k,W_v, W_q, W_o$ & $\{26,\ldots,31\}$\\
        \bottomrule
    \end{tabular}
    \vspace{4pt}
    \caption{Configurations for Commonsense Reasoning tasks.}
    \label{tab:commonsense-task}}
\end{table}
\begin{table}[h]
{\scriptsize
    \centering
    \begin{tabular}{ccccccccc}
    \toprule
         & BoolQ & PIQA & SIQA & HellaSwag & WinoGrande & ARC-e & ARC-c & OBQA\\
         \midrule
        LoRA & $\text{34.36}_{\pm \text{16.90}}$ & $\text{42.87}_{\pm \text{14.05}}$ & $\text{19.12}_{\pm \text{4.22}}$ & $\text{26.21}_{\pm \text{1.91}}$ & $\text{47.2}_{\pm \text{0.64}}$ & $\text{10.31}_{\pm \text{5.96}}$ & $\text{9.84}_{\pm \text{6.13}}$  &$\text{12.33}_{\pm \text{7.46}}$ \\
        FFA-LoRA & $\text{44.04}_{\pm \text{11.48}}$ & $\text{51.46}_{\pm \text{9.81}}$ & $\text{25.38}_{\pm \text{11.27}}$ & $\text{23.86}_{\pm \text{2.67}}$ & $\text{46.93}_{\pm \text{1.54}}$ & $\text{22.25}_{\pm \text{7.92}}$ & $\text{20.65}_{\pm \text{6.33}}$ & $\text{20.67}_{\pm \text{5.33}}$\\
         \rowcolor{ours}RoLoRA  &  $\text{61.3}_{\pm \text{0.99}}$ & $\text{60.81}_{\pm \text{6.35}}$ & $\text{37.97}_{\pm \text{5.39}}$ & $\text{29.62}_{\pm \text{2.62}}$ & $\text{49.59}_{\pm \text{1.2}}$ & $\text{37.05}_{\pm \text{2.92}}$ & $\text{29.09}_{\pm \text{3.33}}$ & $\text{28.93}_{\pm \text{4.64}}$\\
        \bottomrule
    \end{tabular}
    \vspace{4pt}
    \caption{Results with Llama-2-7B models on commonsense reasoning tasks. It involves 50 clients using rank 2.}
    \label{tab:reasoning-rank2-appendix}}
\end{table}
 \subsection{Language Generation Tasks} \label{sec:setup-LGT}
 \vspace{-0.3cm}
\paragraph{Model, Datasets and Metrics.}
We evaluate the performance of three federated finetuning methods with the model Llama-2-7B \cite{touvron2023llama2openfoundation}, on two datasets: CodeAlpaca \cite{codealpaca} for coding tasks, and Alpaca \cite{alpaca} for general instruction-following tasks. Using HumanEval \cite{chen2021codex} as the metric for CodeAlpaca, we assess the model’s ability to generate accurate code solutions. For Alpaca, we employ MMLU (Massive Multitask Language Understanding) \cite{hendrycks2021measuringmassivemultitasklanguage} to evaluate the model's performance across diverse domains. This provides an assessment of Llama-2-7B's coding proficiency, and general language capabilities when finetuning in the federated setting. We show the setup in Table~\ref{tab:my_label}.
\paragraph{Results}
Table ~\ref{llama-results} presents the evaluation results of the Llama-2-7B model using three methods, across two tasks: HumanEval, and MMLU. The metrics reported include the average and standard deviation of performance over five seeds, with 50 clients involved. The results show that RoLoRA achieves the highest scores across most metrics, demonstrating slightly improved performance compared to LoRA and FFA-LoRA. The improvements are more evident in certain subcategories of the MMLU dataset.

\begin{table}[h!]
    \centering
    {\scriptsize
    \begin{tabular}{ccccc}
    \toprule
        Total comm. rounds & Batch size &  Local Epochs & Layer type attached with LoRA &Layer index attached with LoRA \\
        \midrule
        100 & 1 & 30 & $W_k,W_v, W_q, W_o$ & $\{24,\ldots,31\}$\\
        \bottomrule
    \end{tabular}
    \vspace{4pt}
    \caption{Configurations for language generation tasks.}
    \label{tab:my_label}}
\end{table}
\begin{table}[h!]
\centering
{\footnotesize
\begin{tabular}{lccc}
\toprule
                           & LoRA & FFA-LoRA & \cellcolor{ours} RoLoRA \\ \midrule
HumanEval                  &   $\text{12.96}_{\pm \text{0.37}}$    &     $\text{13.29}_{\pm \text{0.21}}$     &   
\cellcolor{ours}$\text{13.45}_{\pm \text{0.28}}$ \\ \midrule
MMLU                       &   $\text{45.81}_{\pm \text{0.03}}$   &   $\text{45.80}_{\pm \text{0.02}}$        &   \cellcolor{ours}  $\text{45.93}_{\pm \text{0.01}}$     \\
\multicolumn{1}{r}{human}  &  $\text{43.26}_{\pm \text{0.04}}$    &         $\text{43.24}_{\pm \text{0.01}}$   & \cellcolor{ours} \cellcolor{ours} $\text{43.46}_{\pm \text{0.02}}$     \\
\multicolumn{1}{r}{other}  &  $\text{52.67}_{\pm \text{0.06}}$    &   $\text{52.72}_{\pm \text{0.05}}$       &   \cellcolor{ours}    $\text{52.83}_{\pm \text{0.04}}$   \\
\multicolumn{1}{r}{social} &    $\text{51.73}_{\pm \text{0.04}}$    &           $\text{51.64}_{\pm \text{0.05}}$ &    \cellcolor{ours}    $\text{51.81}_{\pm \text{0.04}}$ \\
\multicolumn{1}{r}{stem}   &   $\text{37.10}_{\pm \text{0.03}}$   &  $\text{37.12}_{\pm \text{0.01}}$        &    \cellcolor{ours}  $\text{37.05}_{\pm \text{0.02}}$   \\  
\bottomrule
\end{tabular}}
\vspace{4pt} 
\caption{Results with Llama-2-7B model on HumanEval, and MMLU. We report the average and std. over five seeds. The number of clients is 50. The metric used across all tasks is accuracy, where higher values indicate better performance.}
\label{llama-results}
\end{table}

\subsection{Communication and Time Cost Comparison}

Table~\ref{tab:comm-time-cost-comp} compares the communication cost and time cost in a 50-client setting on MNLI task. RoLoRA and FFA-LoRA have the lowest communication and time costs. 
\begin{table}[htbp]
\centering
{\footnotesize
\begin{tabular}{lcc}
\toprule
Method    & Comm. cost & Time cost \\
\midrule
LoRA      & 1500.8 KB   & 0.0415 sec \\
FFA-LoRA  & 750.4 KB    & 0.0163 sec \\
FlexLoRA  & 1500.8 KB   & 2.252 sec \\
FLoRA     & 193603.2 KB & 2.2179 sec \\
\rowcolor{ours}RoLoRA    & 750.4 KB    & 0.0182 sec \\
\bottomrule
\end{tabular}}
\vspace{4pt}
\captionsetup{justification=justified}
\caption{Communication and Time Cost Comparison}
\caption*{\textbf{Comm. cost}: Total message size sent and received by each client per communication round in the MNLI experiment shown in Table~\ref{tab:Clients-num-flex-appendix}. \\
\textbf{Time cost}: Mean server aggregation time per communication round for the MNLI experiment with 50 clients in Table~\ref{tab:Clients-num-flex-appendix}, averaged over 30 runs.}
\label{tab:comm-time-cost-comp}
\end{table}

\subsection{ \quad Privacy-preserving FL}
Beyond robustness, practical FL often requires privacy-preserving mechanisms, most notably cryptographic approaches such as secure aggregation \cite{bonawitz2016practical} or homomorphic encryption \cite{hosseini2025secure, 10206955,10619282}, and statistical protections like differential privacy (DP) \cite{el2022differential, wei2019federated}. While RoLoRA was not explicitly designed with privacy mechanisms such as differential privacy, we recognize that its ability to mitigate inexactness through alternating optimization and aggregation may help mitigate a key challenge for DP-aware federated learning, where inexact model updates can be particularly problematic. In Table~\ref{tab:PPFL}, we itegrated NbAFL \cite{wei2019federated} with $\epsilon=10$ and $\delta=1e-6 $ in a 3-client setting. RoLoRA outperform others across two tasks.
\begin{table}[htbp]
    \centering
{\footnotesize
    \begin{tabular}{ccc}
    \toprule
         & MNLI & QQP\\
         \midrule
       LoRA  & 72.79 \com 5.23& 57.97 \com 8.23\\
        FFA-LoRA & 79.65 \com 0.56 & 78.16 \com 0.6\\
        \rowcolor{ours}RoLoRA & 81.08 \com 0.81 & 81.64 \com 0.35\\
         \bottomrule
    \end{tabular}}
    \vspace{4pt}
    \caption{Evaluation on differential privacy using $\epsilon=10$ and $\delta=1e-6 $ in a 3-client setting.}
    \label{tab:PPFL}
\end{table}

\clearpage
\newpage

\section*{NeurIPS Paper Checklist}

The checklist is designed to encourage best practices for responsible machine learning research, addressing issues of reproducibility, transparency, research ethics, and societal impact. Do not remove the checklist: {\bf The papers not including the checklist will be desk rejected.} The checklist should follow the references and follow the (optional) supplemental material.  The checklist does NOT count towards the page
limit. 

Please read the checklist guidelines carefully for information on how to answer these questions. For each question in the checklist:
\begin{itemize}
    \item You should answer \answerYes{}, \answerNo{}, or \answerNA{}.
    \item \answerNA{} means either that the question is Not Applicable for that particular paper or the relevant information is Not Available.
    \item Please provide a short (1–2 sentence) justification right after your answer (even for NA). 
\end{itemize}

{\bf The checklist answers are an integral part of your paper submission.} They are visible to the reviewers, area chairs, senior area chairs, and ethics reviewers. You will be asked to also include it (after eventual revisions) with the final version of your paper, and its final version will be published with the paper.

The reviewers of your paper will be asked to use the checklist as one of the factors in their evaluation. While "\answerYes{}" is generally preferable to "\answerNo{}", it is perfectly acceptable to answer "\answerNo{}" provided a proper justification is given (e.g., "error bars are not reported because it would be too computationally expensive" or "we were unable to find the license for the dataset we used"). In general, answering "\answerNo{}" or "\answerNA{}" is not grounds for rejection. While the questions are phrased in a binary way, we acknowledge that the true answer is often more nuanced, so please just use your best judgment and write a justification to elaborate. All supporting evidence can appear either in the main paper or the supplemental material, provided in appendix. If you answer \answerYes{} to a question, in the justification please point to the section(s) where related material for the question can be found.

IMPORTANT, please:
\begin{itemize}
    \item {\bf Delete this instruction block, but keep the section heading ``NeurIPS Paper Checklist"},
    \item  {\bf Keep the checklist subsection headings, questions/answers and guidelines below.}
    \item {\bf Do not modify the questions and only use the provided macros for your answers}.
\end{itemize}


\begin{enumerate}

\item {\bf Claims}
    \item[] Question: Do the main claims made in the abstract and introduction accurately reflect the paper's contributions and scope?
    \item[] Answer: \answerYes{}
    \item[] Justification: Please see Section~\ref{rolora-framework}, Section~\ref{sec:analysis} and Section~\ref{exp}.

\item {\bf Limitations}
    \item[] Question: Does the paper discuss the limitations of the work performed by the authors?
    \item[] Answer: \answerYes{}
    \item[] Justification: See Appendix~\ref{app:discussion}.    

\item {\bf Theory assumptions and proofs}
    \item[] Question: For each theoretical result, does the paper provide the full set of assumptions and a complete (and correct) proof?
    \item[] Answer: \answerYes{} 
    \item[] Justification: Please see Appendix~\ref{app:theo-linear} and Appendix~\ref{app:convergence-non-convex}.

    \item {\bf Experimental result reproducibility}
    \item[] Question: Does the paper fully disclose all the information needed to reproduce the main experimental results of the paper to the extent that it affects the main claims and/or conclusions of the paper (regardless of whether the code and data are provided or not)?
    \item[] Answer: \answerYes{} 
    \item[] Justification: See Appendix~\ref{app:additional-exp}.

\item {\bf Open access to data and code}
    \item[] Question: Does the paper provide open access to the data and code, with sufficient instructions to faithfully reproduce the main experimental results, as described in supplemental material?
    \item[] Answer: \answerYes{} 
    \item[] Justification: The datasets are all open-source. The code is uploaded.

\item {\bf Experimental setting/details}
    \item[] Question: Does the paper specify all the training and test details (e.g., data splits, hyperparameters, how they were chosen, type of optimizer, etc.) necessary to understand the results?
    \item[] Answer: \answerYes{} 
    \item[] Justification: See Section~\ref{exp} and Appendix~\ref{app:additional-exp}

\item {\bf Experiment statistical significance}
    \item[] Question: Does the paper report error bars suitably and correctly defined or other appropriate information about the statistical significance of the experiments?
    \item[] Answer: \answerYes{}
    \item[] Justification: See standard deviation reported in Section~\ref{exp} and Appendix~\ref{app:additional-exp}

\item {\bf Experiments compute resources}
    \item[] Question: For each experiment, does the paper provide sufficient information on the computer resources (type of compute workers, memory, time of execution) needed to reproduce the experiments?
    \item[] Answer: \answerYes{} 
    \item[] Justification: Please see Section~\ref{exp}
    
\item {\bf Code of ethics}
    \item[] Question: Does the research conducted in the paper conform, in every respect, with the NeurIPS Code of Ethics \url{https://neurips.cc/public/EthicsGuidelines}?
    \item[] Answer: \answerYes{} 
    \item[] Justification: We followed the NeurIPS Code of Ethics.

\item {\bf Broader impacts}
    \item[] Question: Does the paper discuss both potential positive societal impacts and negative societal impacts of the work performed?
    \item[] Answer: \answerYes{} 
    \item[] Justification: Our paper discusses the broader societal implications of deploying federated learning methods enhanced with LoRA adapters. On the positive side, we emphasize that our approach can significantly improve the efficiency and scalability of federated learning, thereby facilitating model training on decentralized data. We did not identify any direct negative societal impacts, as our contribution is primarily methodological and does not involve deployment or data collection. 

    
\item {\bf Safeguards}
    \item[] Question: Does the paper describe safeguards that have been put in place for responsible release of data or models that have a high risk for misuse (e.g., pretrained language models, image generators, or scraped datasets)?
    \item[] Answer: \answerNA{} 
    \item[] Justification: This paper proposes a framework for federated finetuning, and does not pose high risks for misuse.

\item {\bf Licenses for existing assets}
    \item[] Question: Are the creators or original owners of assets (e.g., code, data, models), used in the paper, properly credited and are the license and terms of use explicitly mentioned and properly respected?
    \item[] Answer:\answerYes{}
    \item[] Justification: This paper is licensed under CC-BY-NC-SA 4.0. All other codes, datasets, and
references are properly cited.

\item {\bf New assets}
    \item[] Question: Are new assets introduced in the paper well documented and is the documentation provided alongside the assets?
    \item[] Answer: \answerNA{}
    \item[] Justification: We did not have any released new assets.

\item {\bf Crowdsourcing and research with human subjects}
    \item[] Question: For crowdsourcing experiments and research with human subjects, does the paper include the full text of instructions given to participants and screenshots, if applicable, as well as details about compensation (if any)? 
    \item[] Answer: \answerNA{} 
    \item[] Justification: We did not involve crowdsourcing nor research with human subjects.

\item {\bf Institutional review board (IRB) approvals or equivalent for research with human subjects}
    \item[] Question: Does the paper describe potential risks incurred by study participants, whether such risks were disclosed to the subjects, and whether Institutional Review Board (IRB) approvals (or an equivalent approval/review based on the requirements of your country or institution) were obtained?
    \item[] Answer: \answerNA{} 
    \item[] Justification: We did not involve crowdsourcing nor research with human subjects.

\item {\bf Declaration of LLM usage}
    \item[] Question: Does the paper describe the usage of LLMs if it is an important, original, or non-standard component of the core methods in this research? Note that if the LLM is used only for writing, editing, or formatting purposes and does not impact the core methodology, scientific rigorousness, or originality of the research, declaration is not required.
    \item[] Answer: \answerNA{} 
    \item[] Justification: The LLM is used only for writing, editing, or formatting purposes and does not impact the core methodology, scientific rigorousness, or originality of the research, declaration is not required.

\end{enumerate}

\end{document}